\documentclass[runningheads]{llncs}
\usepackage[T1]{fontenc}
\usepackage{graphicx}
\usepackage{booktabs}
\usepackage[misc]{ifsym}
\newcommand{\corr}{(\Letter)}
\usepackage{mycommands}
\newcommand{\RatioCoresetSize}{
	
	\begin{table*}[htbp]
		\centering
		\renewcommand{\arraystretch}{1.25} % vertical cell padding
		
		\resizebox{\textwidth}{!}{
			\begin{tabular}{ c |  c c c | c c c | c c c   }
				\multirow{2}{*}{Dataset} 
				&  \multicolumn{3}{c|}{$m = 50k$} 
				&  \multicolumn{3}{c|}{$m = 200k$} 
				&  \multicolumn{3}{c}{$m = 500k$}  \\
				
				&  $k = 10$ &  $k = 20$ &  $k = 50$ 
				&  $k = 10$ &  $k = 20$ &  $k = 50$ 
				&  $k = 10$ &  $k = 20$ &  $k = 50$  \\
				
				\hline
				
				Twitter & 1.00$\pm{\scriptstyle 0.00}$ & 1.00$\pm{\scriptstyle 0.00}$ & 1.00$\pm{\scriptstyle 0.00}$ & 1.00$\pm{\scriptstyle 0.00}$ & 1.00$\pm{\scriptstyle 0.00}$ & 1.00$\pm{\scriptstyle 0.00}$ & 1.00$\pm{\scriptstyle 0.00}$ & 1.00$\pm{\scriptstyle 0.00}$ & 0.99$\pm{\scriptstyle 0.00}$ \\
				
				IntelLab & 0.97$\pm{\scriptstyle 0.04}$ & 0.98$\pm{\scriptstyle 0.03}$ & 0.99$\pm{\scriptstyle 0.02}$ & 0.95$\pm{\scriptstyle 0.05}$ & 0.97$\pm{\scriptstyle 0.04}$ & 0.98$\pm{\scriptstyle 0.06}$ & 0.95$\pm{\scriptstyle 0.06}$ & 0.97$\pm{\scriptstyle 0.07}$ & 0.98$\pm{\scriptstyle 0.13}$ \\
				
				Taxi & 0.64$\pm{\scriptstyle 0.16}$ & 0.58$\pm{\scriptstyle 0.14}$ & 0.53$\pm{\scriptstyle 0.14}$ & 0.56$\pm{\scriptstyle 0.14}$ & 0.51$\pm{\scriptstyle 0.14}$ & 0.46$\pm{\scriptstyle 0.12}$ & 0.53$\pm{\scriptstyle 0.14}$ & 0.49$\pm{\scriptstyle 0.12}$ & 0.44$\pm{\scriptstyle 0.11}$ \\
				
				NYT (M) & 0.22$\pm{\scriptstyle 0.14}$ & 0.38$\pm{\scriptstyle 0.19}$ & 0.53$\pm{\scriptstyle 0.20}$ & 0.18$\pm{\scriptstyle 0.13}$ & 0.36$\pm{\scriptstyle 0.18}$ & 0.50$\pm{\scriptstyle 0.18}$ & 0.18$\pm{\scriptstyle 0.13}$ & 0.34$\pm{\scriptstyle 0.17}$ & 0.45$\pm{\scriptstyle 0.18}$ \\
				
				NYT (Y) & 0.56$ \pm {\scriptstyle 0.19}$ &  0.54$ \pm {\scriptstyle 0.25}$ &  0.58$ \pm {\scriptstyle 0.26}$ &  0.53$ \pm {\scriptstyle 0.21}$ &  0.52$ \pm {\scriptstyle 0.25}$ &  0.57$ \pm {\scriptstyle 0.26}$ &  0.51$ \pm {\scriptstyle 0.22}$ &  0.52$ \pm {\scriptstyle 0.24}$ &  0.56$ \pm {\scriptstyle 0.25}$  \\
			\end{tabular}
		}
		
		\caption{Number of unique points in coresets produced by \algname\ over number of unique points in coresets produced by \uniformsampling. For each dataset (row), results are averaged over all the snapshots, and over 10 independent repetitions of coreset construction. Ratio standard deviations are estimated via error propagation.}
		\label{tab:ratio_size}
	\end{table*}
	
}

\newcommand\CompactRuntime[1]{%
	\begin{figure*}[htbp]
		\centering
		\includegraphics[width=\textwidth]{figures/runtime/runtime_compact_barplot_m#1.pdf}
		\caption{Times for producing coresets across different number $k$ of centers, fixing coreset size $m = #1 k$. For each dataset (subplot), the sum of runtimes over all the snapshots is reported, with mean and standard deviation across 10 independent repetitions of coreset construction. y-axis is reported in log scale.}
		\label{fig:runtime_m\detokenize{#1}}
	\end{figure*}%
}

\newcommand\CompactLogRatioCompressionOptimization[1]{%
	\begin{figure}[!htbp]
		\centering
		\includegraphics[width=\textwidth]{figures/copt/compact_copt_log_ratios_m#1.pdf}
		\caption{$k$-means cost of coresets-based algorithms for coreset size $m = #1 k$. 		
			We plot the \emph{ratio} between each algorithm's achieved cost and the cost achieved on the whole dataset, on a \emph{log scale}. Black dashed line: reference cost of $k$-means on the whole dataset (always $1$).
		Results are averaged over 10 independent repetitions of coreset construction, reporting means and 95\% confidence intervals.}
		\label{fig:log_ratio_compression_optimization_m\detokenize{#1}}
	\end{figure}
}

\newcommand\EstimatedLogRatioDistortion[1]{%
	\begin{figure}[htbp]
		\centering
		\includegraphics[width=\textwidth]{figures/ed/compact_log_ratio_ed_m#1.pdf}
		\caption{
			Estimated distortion on coresets with size $m = #1 k$. 
			We plot the \emph{ratio} between each algorithm's estimated distortion and the estimated distortion of \algname, which is fixed to $1$ by construction (baseline), on a \emph{log scale}.
			Results are averaged over 10 independent repetitions of coreset construction, reporting means and 95\% confidence intervals. 
			Estimated distortion of \sstwenty\ algorithm for Twitter dataset is not displayed for the sake of visualization, since it is always greater than $2.73\times$ the estimated distortion of \algname. }
		\label{fig:estimated_log_ratio_distortion_m\detokenize{#1}}
	\end{figure}%
}

\newcommand{\DataVizTaxi}{
	\begin{figure}[htbp]
		\centering
		\begin{subfigure}[htbp]{0.48\columnwidth}
			\centering
			\includegraphics[width=\columnwidth]{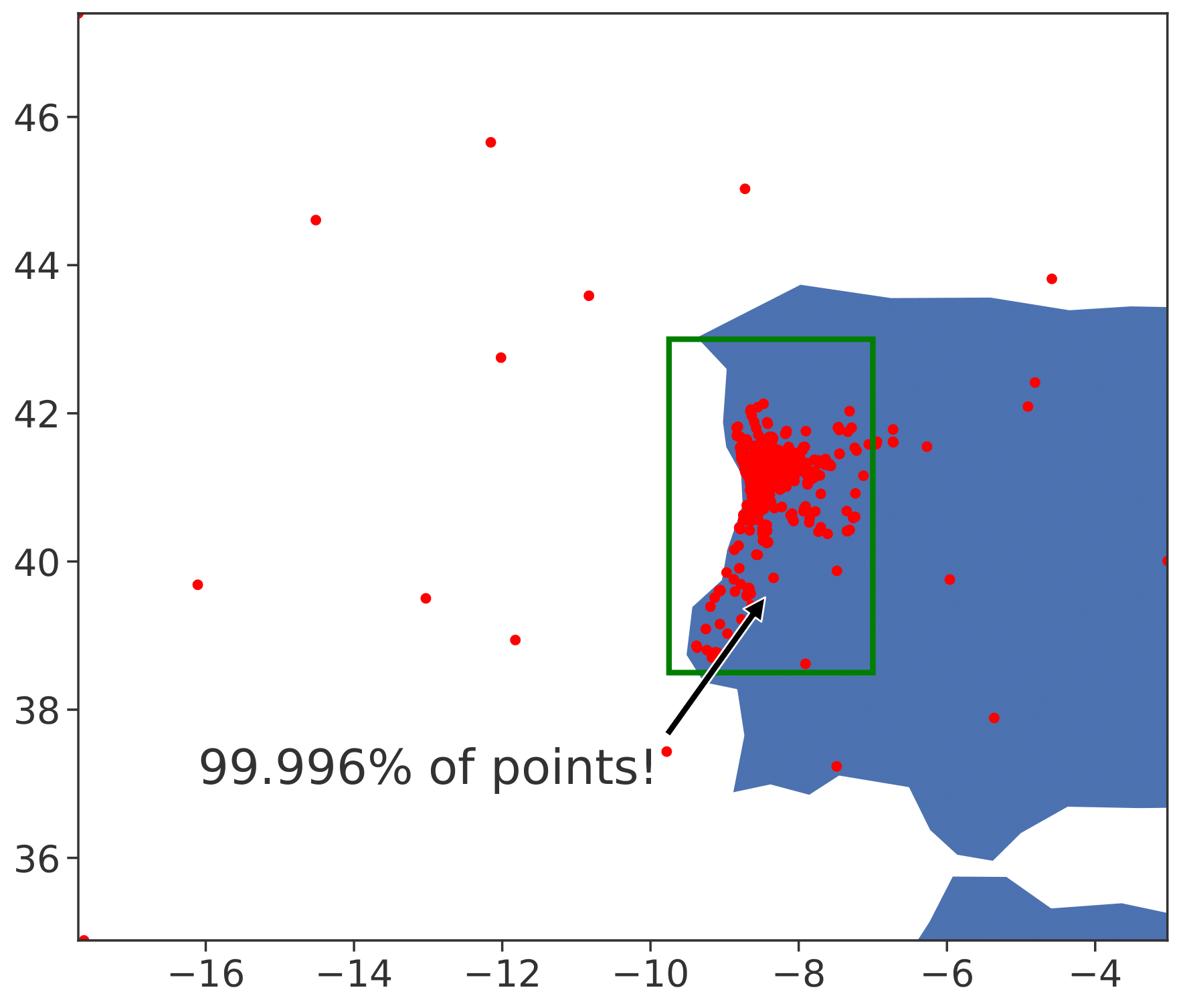}
		\end{subfigure}
		\hfill
		\begin{subfigure}[htbp]{0.48\columnwidth}
			\centering
			\includegraphics[width=\columnwidth]{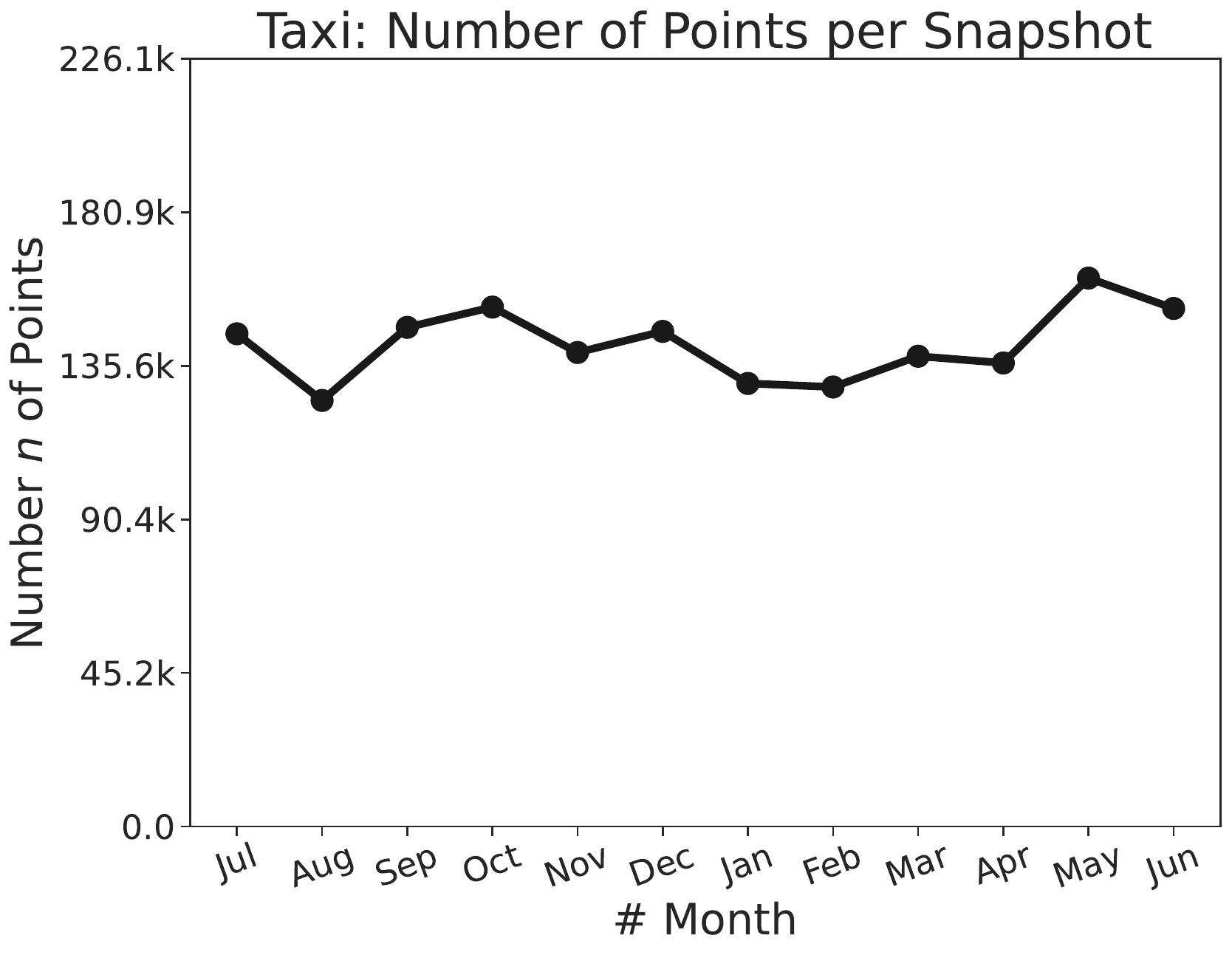}
		\end{subfigure}
		\caption{Visualization of Taxi Dataset (12 snapshots). 
			(Left) Latitude and Longitude of pickup location when excluding the \emph{two most extreme} points. 
			(Right) Number of points for each snapshot, representing the month in which data has been collected. }
		\label{fig:taxi-dataset}
	\end{figure}
	
}

\newcommand{\DataVizTwitter}{
	\begin{figure}[htbp]
		\centering
		\begin{subfigure}[htbp]{0.5\columnwidth}
			\centering
			\includegraphics[width=\columnwidth]{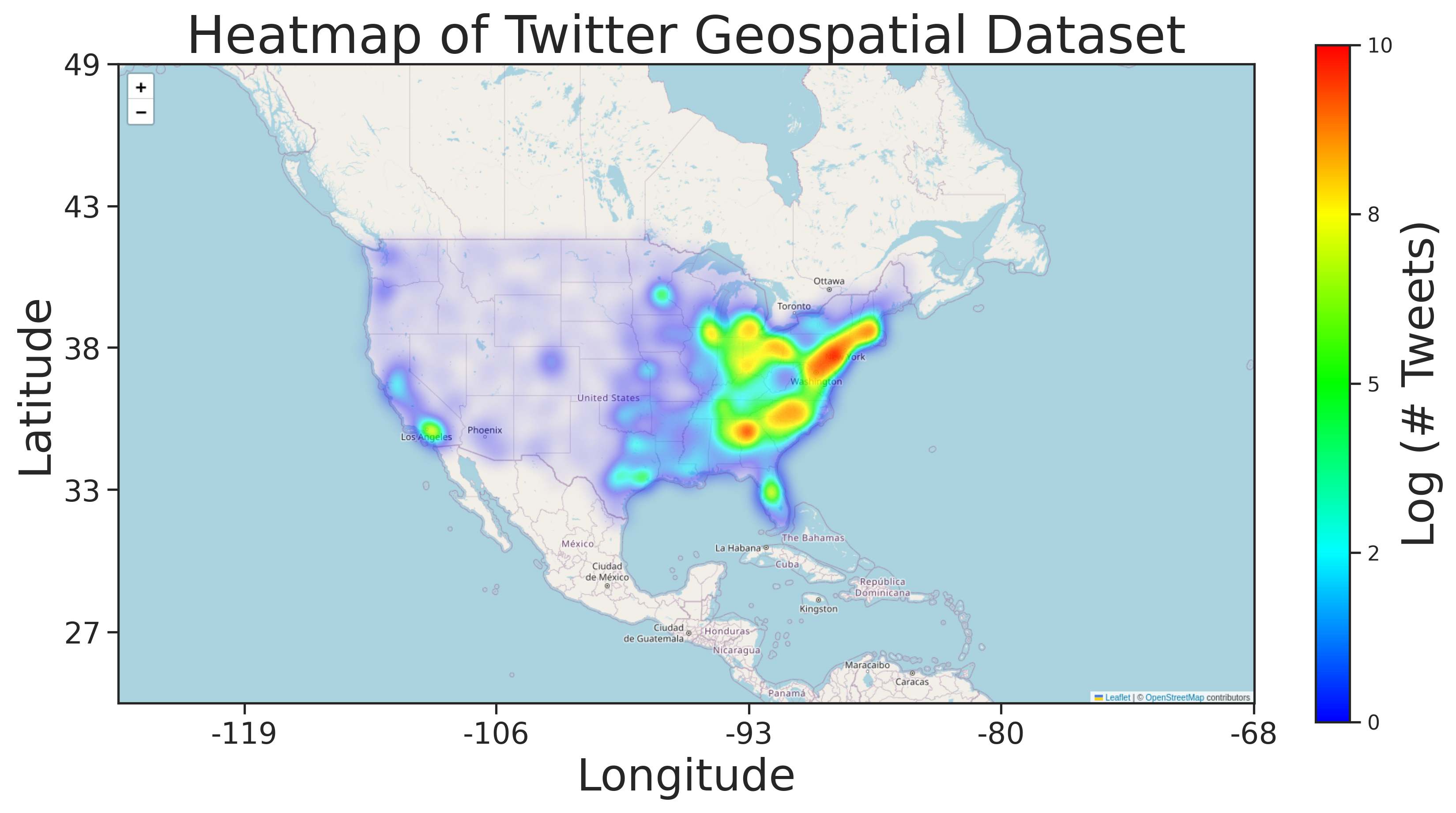}
		\end{subfigure}
		\begin{subfigure}[htbp]{0.45\columnwidth}
			\centering
			\includegraphics[width=\columnwidth]{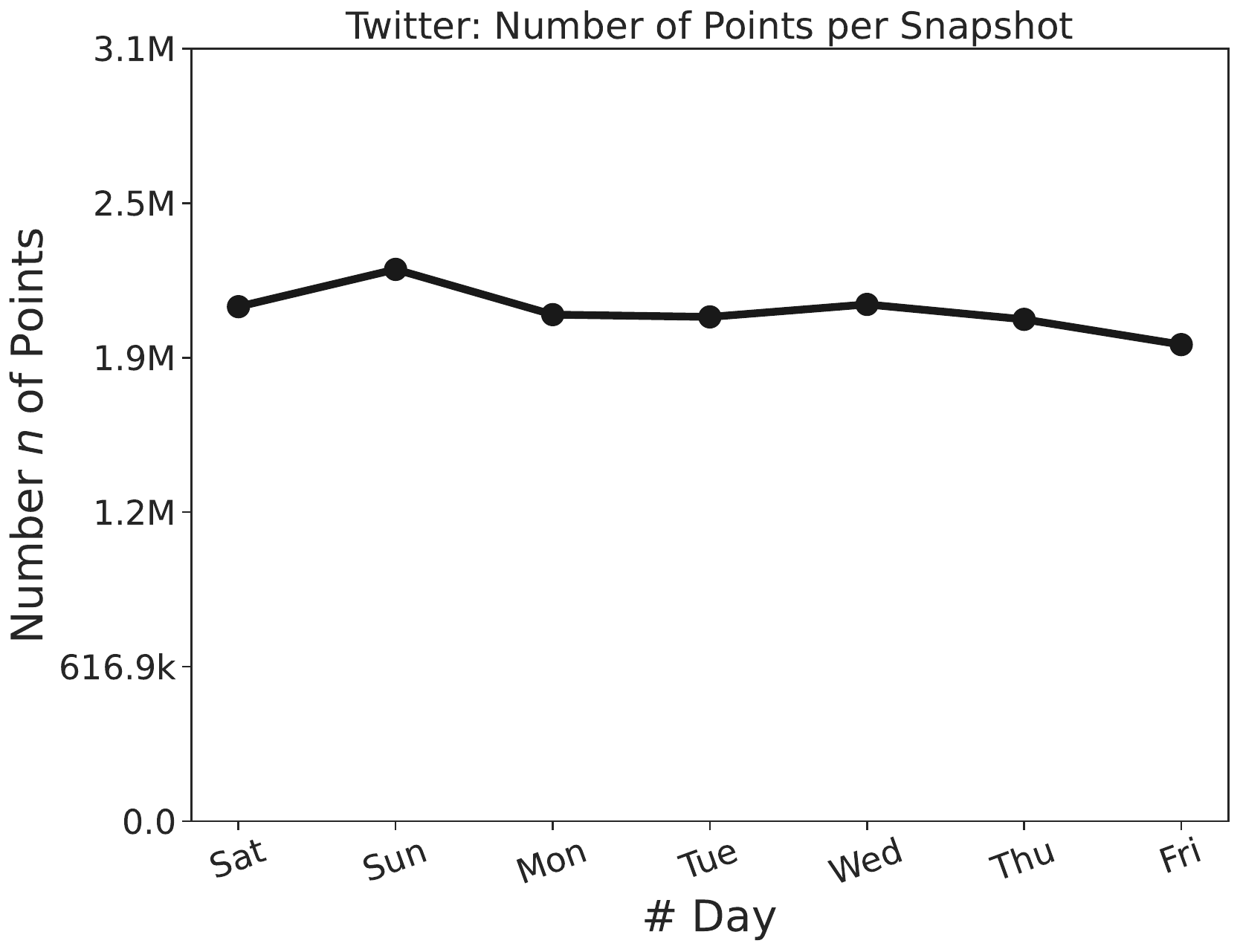}
		\end{subfigure}
		\caption{Visualization of Twitter Dataset (7 snapshots). (Left) Latitude and Longitude of location of Tweets over the US map. Locations have been discretized into bins of $1000\times500$, and logarithm of Tweet counts is shown. (Right) Number of points for each snapshot, representing the day in which data has been collected. }
		\label{fig:twitter-dataset}
	\end{figure}
}

\newcommand{\DataVizNYCTLC}{
	\begin{figure}[htbp]
		\centering
		\begin{subfigure}[b]{0.64 \columnwidth}
			\centering
			\includegraphics[width=\columnwidth]{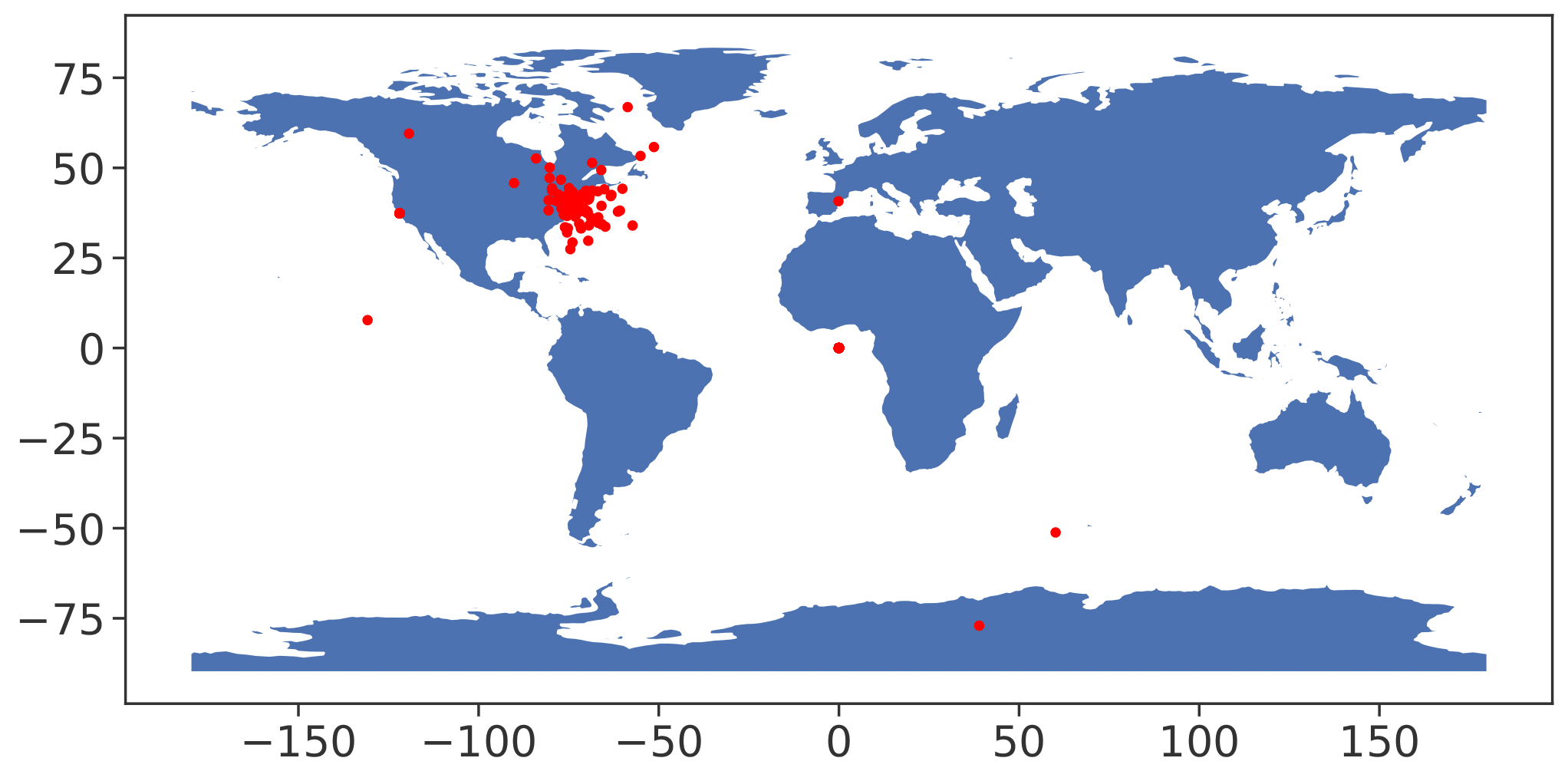}
		\end{subfigure}
		\hfill
		\begin{subfigure}[b]{0.34\columnwidth}
			\centering
			\includegraphics[width=\columnwidth]{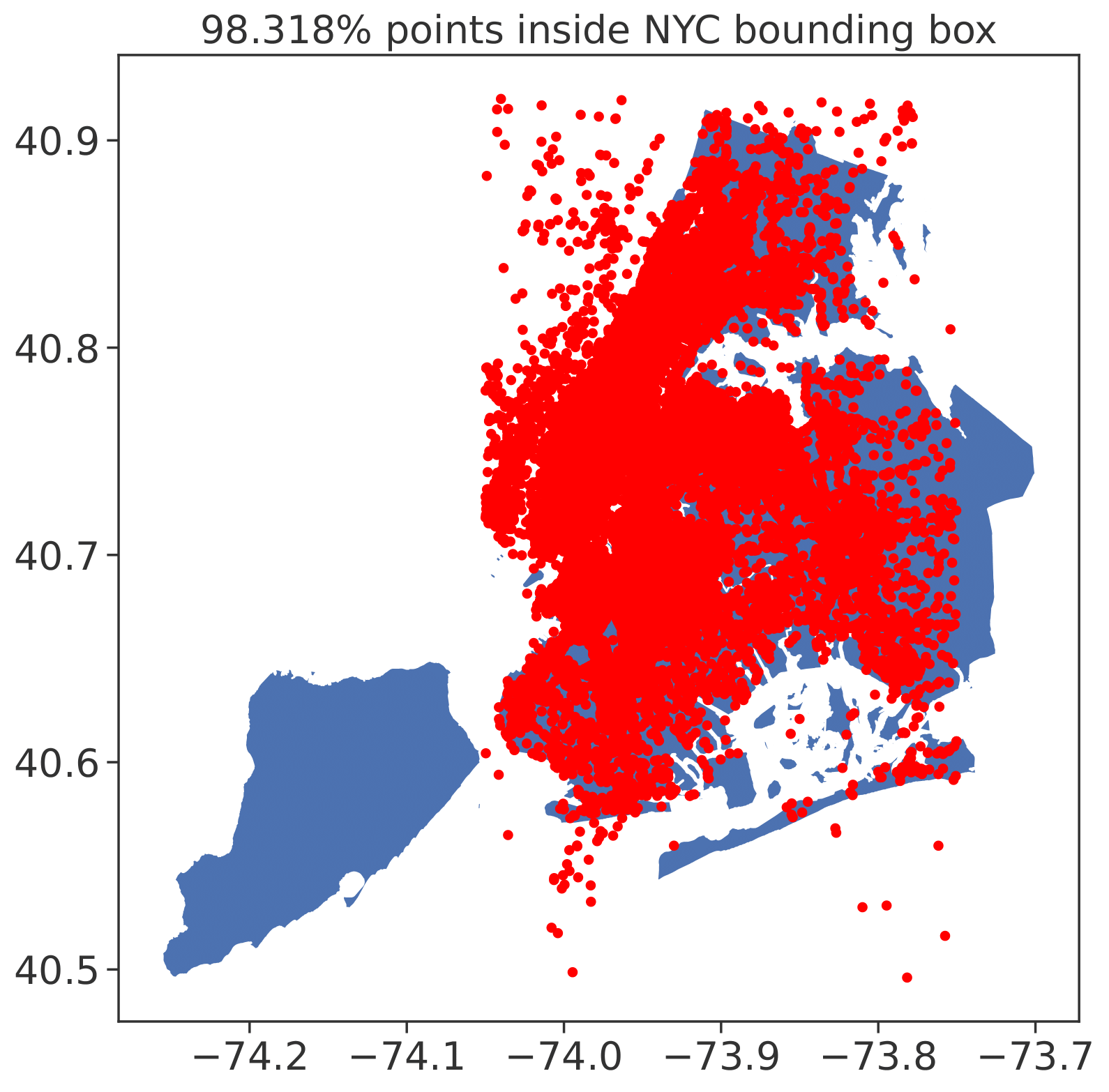}
		\end{subfigure}
		\hfill	
		\caption{Visualization of NYC TLC (M). 
			(Left) Plot of a $10\%$ random sample showing pickup longitude and latitude on world map; 
			(Right) Plot of a $10\%$ random sample showing pickup longitude and latitude on New York City.}
		\label{fig:NYC-Taxi-dataset}
	\end{figure}

}

\begin{document}

\title{Sensitivity Sampling with Predictions for $k$-Means Clustering}

\titlerunning{Sensitivity Sampling with Predictions for $k$-Means Clustering}
% If the full title of your paper is short enough to also fit in the running head, you can omit the abbreviated paper title here. You can check as follows: if you comment out the \titlerunning line, something will appear in the header of all odd-numbered pages of your PDF from page 3 onward. This something is either the full title (in which case all is well), or the error message "Title Suppressed Due to Excessive Length". If this error message appears, you're going to want to provide an abbreviated title within the \titlerunning command, because if you won't do it, Springer will do it for you.

%N.B.: Author information (both in the \author{} and \authorrunning{} command) should only be present in the Camera-Ready Version of your paper. The version that you initially submit for review, ought to be double-blind. So, when initially submitting your paper, use:
%\author{Author information scrubbed for double-blind reviewing}
%\author{Cristian Boldrin\inst{1} \and
%	Fabio Vandin.\inst{1} \corr}
\author{Cristian Boldrin \and Fabio Vandin \corr}
% You may leave out the orcidID information, if you want to.
% Use \corr to indicate the corresponding author. Note the spacing around the \corr command. Only one author can be the corresponding author.

%N.B.: comment out the \authorrunning{} command for the double-blind version of your paper submitted for review. Later, if your paper is accepted, use the command for the Camera-Ready Version.
\authorrunning{C. Boldrin and F. Vandin}
% If there are two authors, write 'A.L. Benjamin and C.C. Broadus Jr.'
% If there are more than two authors, '[...] et al.' is used.

% -- toc info
\toctitle{Sensitivity Sampling with Predictions for $k$-Means Clustering}
\tocauthor{Cristian~Boldrin,~Fabio~Vandin}

\institute{University of Padova, Padova, Italy\\
	\email{boldrincri@dei.unipd.it}, \email{fabio.vandin@unipd.it}}

\maketitle              % typeset the header of the contribution

\begin{abstract}
We study the problem of $k$-means clustering on large datasets. 
The state-of-the-art for the problem is given by coresets-based approach\-es, which build small weighted summaries of the input and derive approximate solutions with rigorous quality guarantees from them.
One of the most popular and advanced approaches to derive coresets for $k$-means is sensitivity sampling. However, sensitivity sampling requires to compute the importance of each input point with respect to the whole dataset over all possible choices of  centers. Since the exact computation of such quantities is unfeasible, current approaches work by approximating the sensitivity values. Nevertheless, the runtime of such approaches is still impractical for large datasets.

In this work, we propose to reduce the runtime of sensitivity-based approaches for $k$-means by leveraging predictions to approximate the importance of input points. 
We first formally prove that current theoretical results on coresets construction via sensitivity sampling hold for coarser approximations of sensitivities compared to the one required by existing approaches. This implies that even fairly noisy predictors can be leveraged for sensitivity-sampling approaches. We then propose a natural predictor, which applies to the common scenario where clustering is performed (over time) on a sequence of datasets from the same problem. We prove that when the datasets in the sequence come from the same (unknown) distribution, centers resulting in a low error on one dataset can be used as predictions for sensitivity sampling in subsequent datasets, with guarantees on their quality. We perform an extensive experimental evaluation showing that our approach significantly improves, in terms of clustering cost vs runtime, over uniform sampling and state-of-the-art sensitivity sampling approaches when applied to sequences of datasets.

\end{abstract}

\section{Introduction}
\label{sec:introduction}
Clustering is a fundamental task for data mining and machine learning, with applications in several domains~\cite{hennig2015handbook}. One of the most popular variants is $k$-means clustering: given a set  $P$ of $n$ points in $d$-dimensional Euclidean space, the goal is to find a set $S$ of $k$ points of the Euclidean space, called \emph{centers}, which partitions $P$ into clusters minimizing the $k$-means cost, defined as the sum of squared distances between points in $P$ and their nearest center in $S$. 
Finding the optimal set of centers is known to be NP-hard even for $k = 2$~\cite{dasgupta2008hardness}, and even finding an approximation to the optimal objective value within a factor of 1.07 is known to be NP-hard~\cite{cohen2019inapproximability}.

The most popular heuristic for solving the $k$-means problem in Euclidean spaces is Lloyd's algorithm~\cite{lloyd1982least}, which uses iterative improvements to find a locally optimal $k$-means clustering.
Each iteration of Lloyd's algorithm requires $\bigO{nkd}$ time, and in the worst case the number of iterations required to converge is $2^{\Omega\left( \sqrt{n} \right)}$~\cite{arthur2006slow}, leading to superpolynomial worst-case time complexity. Moreover, the quality of the final solution depends heavily on the centers' initialization.

% -- brief introduction of coresets, with main properties (motivation)
As modern datasets continue to grow, a natural approach to scale the computation is to derive compressed representations of the input (e.g., by sampling).
For clustering, a common approach is to use \emph{coresets}, where the input data $P$ is replaced by a small set of weighted points, which provably approximates the $k$-means cost function with respect to \emph{any} set of centers. 
Running optimization algorithms (e.g., Lloyd's algorithm) on coresets then produces results close to the solution obtained from running the same algorithms on the original dataset, allowing for \emph{efficiency} in terms of runtime and memory consumption.
%Another key property of coresets is their \emph{composability}: if we let the input $P$ to be partitioned into disjoints blocks $P_i$, and $\Omega_i$ is a coreset for block $P_i$, then the weighted union $\bigcup_i \Omega_i$ is itself a coreset for the whole dataset $P$. Thus, composability of coresets allows for merge-and-reduce techniques, which can be leveraged in streaming and distributed applications~\cite{har2004coresets,bentley1980decomposable}. 

From now on, we restrict our focus to coresets in Euclidean $k$-means clustering. 
More formally, given a set $S$ of $k$ centers and a point $p$, let $\cost(p, S) = \min_{c \in S} \dist(p, c)^2$ be the squared (Euclidean) distance from point $p$ to its nearest center in $S$, and let $\cost(P, S) = \sum_{p \in P} \cost(p, S)$ be the total cost of dataset $P$ induced by $S$.
Arguably, one of the most popular algorithms for constructing coresets is \emph{sensitivity sampling}. In essence, this algorithm samples each point with probability proportional to its \emph{sensitivity}, which is defined, for a point $p \in P$, as $\sup_{S \colon |S| = k} \frac{\cost(p,S)}{\cost(P, S)}$. In the latter, the supremum is taken over the infinitely many sets of $k$ centers in $\mathbb{R}^d$.
Intuitively, the sensitivity is a measure that captures the importance of a point with respect to the whole dataset, over all the possible choices of centers. 
Eventually, the coreset is constructed by assigning to each of the picked points a weight inversely proportional to its sampling probability. 

%	Having a closer look to the definition of sensitivity, one can see the stemming of a chicken-egg problem: the optimal set $S$ of centers (i.e., the one which minimizes the $k$-means objective) needs to be known for computing the sensitivities (for taking the supremum over all the choices of centers). Then, sensitivities themselves are needed for producing the coreset in first place, with which one would like to compute estimates of the optimal set $S$ of centers. 
%	Hence, computing sensitivities is NP-Hard as well. 
%Computing sensitivities is tightly related to solving $k$-means clustering problem: typically, the optimal set $S^*$ of centers (i.e., the one which minimizes the $k$-means objective) is required to bound the sensitivities. This leads to a chicken-egg problem, given that finding $S^*$ is the main motivation for constructing the coreset in the first place. 
While sensitivity sampling leads to small and effective coresets, computing sensitivities is extremely challenging, as it requires considering clustering costs over infinitely many choices of centers in $\mathbb{R}^d$.
Fortunately, sensitivity-based algorithms for constructing coresets can work with approximations of the sensitivities, instead of their exact values. Such approximations can be obtained, for example, by considering the ``importance'' of points with respect to a $\Bicriteria{\bigO{1}}{\bigO{1}}$ \emph{bi-criteria approximation}~\cite{feldman2011unified}.
Generally speaking, a $\left( \alpha, \beta \right)$ bi-criteria approximation for the $k$-means problem on an input set $P$ of points, whose optimal cost is $\OPT_k(P)$, is a set $A$ of centers that:
\begin{myitemize}
	\item gives an $\alpha$ approximation of $\OPT_k(P)$, i.e., $\cost(P, A) \leq \alpha \cdot \OPT_k(P)$, and
	\item uses at most $\beta \cdot k$ centers, i.e., $|A| \leq \beta \cdot k$, with $\beta \geq 1$.
\end{myitemize}
A $\Bicriteria{\bigO{1}}{\bigO{1}}$ bi-criteria approximation is obtained, for example, by running \texttt{$k$-means++}~\cite{arthur2006k} using $2k$ centers; such procedure runs in $\bigTheta{nkd}$, and guarantees~\cite{makarychev2020improved} that, in expectation, the value of $\alpha$ is less than $8.68$ for any $k \ge 2$.  
% -- brief excursus for time complexity which motivates the use of the predictions
For standard sensitivity-based algorithms, time complexity is dominated by the computation of the initial bi-criteria assignment (that is $\bigTheta{nkd}$ if using $\texttt{k-means++}$), after which sampling and re-weighting can be done in time sublinear in the size of the input. Thus, the total runtime of the algorithm is often unfeasible, especially when dealing with hundreds of millions of points in high-dimensions.%, with certain applications required to utilize a number $k$ of centers in the thousands~\cite{hu2017fast,bachem2016fast}. 

% -- AWP
Mitzenmacher and Vassilvitskii~\cite{mitzenmacher2022algorithms} recently formalized the \emph{algorithms with predictions} framework, where combinatorial algorithms are empowered by relevant but noisy predictions (e.g., from a machine learning model). The availability of a predictor (as part of the input) has led to improved algorithms for several classical problems. %, and several algorithms for well-studied problems have been proposed~\cite{hsu2019learning,khalil2017learning,mitzenmacher2018model}.
In our setting, we consider predictions provided as a set of centers on which computing the sampling probabilities, replacing the initial bi-criteria approximation step in sensitivity sampling algorithm, which is the primary runtime bottleneck.
We show that this substantially reduces the empirical runtime compared to recomputing a bi-criteria approximation from scratch (e.g., via \texttt{k-means++} with $2k$ centers). 
%Within this framework, sensitivity-based algorithms producing coresets can replace the computation of the initial bi-criteria approximation (the primary runtime bottleneck) by assuming to access a \emph{predictor} which provides a set of centers, on which computing induced costs. In this way, the empirical runtime can be drastically alleviated, with respect to the standard algorithm which runs \texttt{kmeans++} with $2k$ centers from scratch.

The use of predictions is particularly well-suited to dynamic real-world applications in which datasets constitute a \emph{sequence of snapshots}. 
For instance, sensor-generated datasets are inherently produced as a time-ordered stream and aggregated into snapshots over periods such as days or months. 
In these cases, historical information can be reused to derive predictors for boosting performance on subsequent inputs. 
To this extent, one natural question arises. 

\emph{Question 1. If a dataset is represented as a sequence of snapshots (assuming the same data-generating distribution), can we derive a \emph{provably good} predictor for sensitivities from a \emph{previous snapshot}, in order to speed up the computation of coresets for subsequent snapshots?} 

\textit{Our Contributions.} We positively answer Question 1. Following the statistical clustering framework introduced by Ben-David~\cite{ben2007framework}, we prove that centers yielding a bi-criteria approximation on a past snapshot provide (asymptotically) a bi-criteria approximation on subsequent snapshots \emph{with high probability}, provided that snapshots are \emph{large enough} and drawn i.i.d. from the same underlying data-generating distribution.

Moreover, we show that the requirement of $\Bicriteria{\bigO{1}}{\bigO{1}}$ bi-criteria approximation can be relaxed and, under some mild assumptions, theoretically small coresets are still achievable, matching the state-of-the-art. All in all, we provide formal guarantees to produce small coresets, while improving the empirical runtime leveraging predictions from past data, which is a natural approach in scenarios where the input is a sequence of snapshots.  

Leveraging such results, we propose an algorithm, called \algname, which uses predictions for efficiently producing small and accurate coresets using sensitivity sampling, in the context of Euclidean $k$-means clustering. 
Note that, even with the (predicted) centers given as an input, one still needs to compute the induced costs for deriving the sampling probabilities of points, resulting again in $\bigO{nkd}$ runtime. 
%However, such operation (i.e., computation of distances) can be embarrassingly parallelized, in contrast to the version which computes the initial approximation by running \texttt{kmeans++} with $2k$ centers from scratch, where the sampling distribution has to be recomputed for each newly picked center.
However, we show that, in practice, computing distances from the (predicted) given set of centers is \emph{much faster} than computing bi-criteria approximations from scratch with $\texttt{kmeans++}$.

More in detail, our contributions are:
\begin{myitemize}
	\item We develop algorithm \algname, the first approach which combines coresets and predictions in order to speed up the runtime for sensitivity-based $k$-means clustering. % Most importantly, our algorithm maintains the overall structure of the standard algorithm for sensitivity sampling.
	\item We analyze \algname\ in the context of dynamic applications which involve sequences of snapshots, proving that the predictions (i.e., sets of centers) can be computed from historical data.
	\item We prove that \algname\ matches the state-of-the-art in terms of coreset size, and allows for a relaxed initial approximation under some mild assumptions. In our model, this corresponds to account for inaccurate predictions, while still being able to guarantee small coresets. 
	\item We conduct a detailed experimental evaluation (see Section~\ref{sec:experiments}) on sequences of snapshots, in which our proposed algorithm shows significant speedup in runtime, while still being as accurate as the state-of-the-art sensitivity-sampling algorithm that does not employ predictions.
\end{myitemize}

\section{Preliminaries}
\label{sec:preliminaries}
%For an integer $n > 0$ we use $[n]$ to denote the set $\{ 1, 2, ..., n\}$. 
For two points $p, q \in \mathbb{R}^d$, let $\dist(p, q)= \sqrt{\sum_{i=1}^d (p_i~-~q_i)^2}$ be their Euclidean distance.
Given a fixed set $S$ of $k$ centers in $\mathbb{R}^d$, we indicate with $\cost(p, S) \coloneqq \min_{c \in S} \dist(p, c)^2$ the \emph{cost} of point $p$ for centers $S$, that is the squared distance of $p$  to its nearest center in $S$. For a set $P$ of points, we denote $\cost(P, S) \coloneqq \sum_{p \in P} \cost(p, S)$ to be the $k$-means cost of $P$ with centers $S$. 
Similarly, for a weighted set $\Omega$ of points $\Omega = \{q_1, \dots, q_m\}$ with weights $\{w_{q_1}, \dots, w_{q_m}\}$, we let $\cost_{\Omega}(P, S) \coloneqq \sum_{q_i \in \Omega} w_{q_i} \cost(q_i, S)$ be its weighted $k$-means cost. We denote with $\OPT_k(P) \coloneqq \min_{S \colon |S| = k} \cost(P, S)$ the optimal $k$-means cost of dataset $P$.

We now formally define \emph{coresets}. In particular, we focus on (strong) $\varepsilon$-coreset, which preserve (weighted) costs relative to costs on the original dataset, within a multiplicative factor of $(1 \pm \varepsilon)$, \emph{for any} set of centers. 
Henceforth, we refer to (strong) $\varepsilon$-coresets simply as \emph{coresets}. 
\begin{definition}[Coreset] \label{def:coreset}
	Given an approximation parameter $\varepsilon \in (0, 1)$ and a set $P \subset \mathbb{R}^d$ of points, a subset $\Omega = \{q_1, ..., q_m\} \subset P$ of points with weights $\{w_{q_1}, \dots, w_{q_m}\}$ is a \emph{coreset} for $P$ if, for any set $S$ of $k$ centers, we have:
	\begin{equation} \label{eq:coreset}
		(1 - \varepsilon) \cost(P, S) \leq \cost_\Omega(P, S) \leq (1 + \varepsilon) \cost(P, S).
	\end{equation}
\end{definition}
The \emph{size} of a coreset is the number $m$  of points it contains.
Observe that a coreset $\Omega$ approximately preserves the $k$-means objective of $P$ \emph{for any} set $S$ of $k$ centers \emph{simultaneously}. 
Notice that we constrain the coreset $\Omega$ to be a subset of the dataset $P$: allowing $\Omega$ to contain points outside the dataset yields coresets of smaller size~\cite{braverman2016new,feldman2020turning}, but it is beyond the scope of this work.

\section{Related Work}
\label{sec:related_works}
A lot of research has been done on $k$-means and coresets. We focus on the works that are directly relevant to our setting. In particular, we restrict to coresets that are a subset of the input for Euclidean $k$-means clustering.
For general background on coresets we refer to surveys~\cite{feldman2020core,munteanu2018coresets}.

% -- sensitivity sampling related works
For \emph{sensitivity sampling},~\cite{feldman2011unified} showed that sampling $\bigOtilde{kd\varepsilon^{-4}}$\footnote{The $\bigOtilde{n}$ notation ignores poly-logarithmic factors, i.e., $\bigO{n \cdot \log^c n}$ for constant $c$.} points independently with probability proportional to an upper bound on their sensitivities suffices to produce a coreset (with constant probability). 
Applications of dimensionality reduction subsequently led to coreset size of $\bigOtilde{k^3 \varepsilon^{-4}}$, $\bigOtilde{k^2  \varepsilon^{-4}}$ and $\bigOtilde{k \varepsilon^{-6}}$~\cite{feldman2020turning,braverman2021coresets,huang2020coresets}.
% -- group sampling related works 
\cite{cohen2021new} introduced \emph{group sampling} for processesing the input into groups such that grouped points and clusters have similar costs (with respect to an initial approximation); then, the algorithm draws points in every group proportionate to their costs, which allows for coresets of size $\bigOtilde{k^2  \varepsilon^{-2}}$.
%The technical improvement in the analysis comes from applying a union bound over a nested sequence of increasingly better discretizations (nets), and by carefully trading off variance and net size at each scale, which allows for provable coresets of size $\bigOtilde{k^2  \varepsilon^{-2}}$.
Later,~\cite{cohen2022improved} gave a significantly more sophisticated trade-off, improving the bound to \cohenbound. \cite{huang2024optimal} showed that this bound is tight when the coreset is a subset of the input, showing that any coreset must have size \huanglowerbound. 
%Surprisingly, the bound has been derived by an improved analysis, and not by modifying the structure of the standard sensitivity sampling algorithm. 
%\cristian{\sout{Again, the bound \cohenbound\ for coreset size is optimal when constraining coresets to be part of the input.}}
Hence, the issue of determining a coreset of optimal size for Euclidean $k$-means clustering is considered closed~\cite{cohen2022improved,cohen2022towards,huang2024optimal}.
Recently,~\cite{bansal2024sensitivity} matched the optimal bound from~\cite{cohen2022improved} (achieved using group sampling), within the sensitivity sampling framework. 
% -- recent advancement (SOTA)

In practice, sensitivity sampling has become the simplest and easiest way to implement algorithms for creating coresets, and is consistently preferred to group sampling given that it exhibits lower runtimes and is more accurate on standard datasets (see~\cite{schwiegelshohn2022empirical}), and it is also easier to implement.
Among recent practical implementations,~\cite{draganov2024settling} proposed \fastcoreset, a sensitivity-based algorithm aimed at near linear-time coreset construction using quadtree embeddings, but can only work in practice with unique points in the original dataset. 
% for obtaining coresets in $\bigOtilde{nd}$ time 

Despite its conceptual simplicity and strong empirical performance~\cite{schwiegelshohn2022empirical}, sensitivity sampling requires computing an approximate clustering solution to upper bound sensitivities, which is the dominant cost in practice and limits the scalability of sensitivity sampling approaches.  This is precisely the bottleneck we target: we leverage predictions to avoid recomputing such an approximation from scratch. We show that predictions can be naturally viewed as centers computed on past data when the input is modeled as a sequence of snapshots, while retaining provable guarantees. 

$k$-means clustering has been studied in the \emph{algorithms with predictions} framework~\cite{ergunlearning,nguyenimproved,huang2025new}, assuming informative predictions about an optimal clustering for each input point: the results show that when predictions are accurate, a solution of cost close to the optimal one can be obtained, with the distance between the two costs that depends on the quality of the predictions.
Note that this differs significantly from our approach in two ways. First, we focus on coreset approaches, 
which allow clustering algorithms to scale effectively to large datasets while maintaining provable guarantees on the resulting solution. Second, we consider predictions that are less informative than the ones considered in~\cite{ergunlearning,nguyenimproved,huang2025new} and represented by cluster centers used to approximate sensitivities, which can be efficiently computed from past data.

\section{Sensitivity Sampling with Predictions}
\label{sec:ss_predictor}
In this section, we present our proposed method \algname\ that employs predictions, in the form of cluster centers, to compute an approximation of  sensitivities, which are then used to construct small coresets. 

\algname\ is described in Alg.~\ref{alg:sensitivity-sampling-oracle-coreset}, and it works as follows. Given a dataset $P$, the number $k$ of clusters, the coreset size $m$, and the set $A$ of predicted centers, \algname\ first computes the assignment  (i.e., the corresponding clusters and their costs) of points of $P$ to $A$  (line~\ref{line:assignment}). Then, \algname\ derives a probability $\Prob{p}$ for each point in $P$ using the costs induced by the set $A$ of predicted centers (line~\ref{line:prob-distro-oracle}). 
Finally, it samples $m$ points according to the distribution given by the values $\Prob{p}$'s and computes a weight for each sampled point (lines~\ref{line:sample_first}-\ref{line:sample_last}). The $m$ points and their weight constitute the coreset that is produced in output (line~\ref{line:return}). 
%We assume that the (predicted) set $A$ of centers is such that it yields a $\Bicriteria{\alpha}{\beta}$ bi-criteria approximation. That is we require (i) $\cost(P, A) \leq \alpha \cdot \OPT_k(P)$ and (ii) $k \leq |A| \leq \beta k$, for some (possibly non-constants) $\alpha$ and $\beta$. 

\begin{algorithm}[htbp]
	\caption{\algname $\left( P, k, m, A \right)$}
	\label{alg:sensitivity-sampling-oracle-coreset}
	\LinesNumbered
	%	\DontPrintSemicolon
	\kwInput{dataset $P = \{p_1, ..., p_n \} \subset \mathbb{R}^d$; number $k$ of clusters; coreset size $m$; set $A$ of predicted centers.}
	\kwOutput{coreset $\Omega = \{q_1, ..., q_m\}$ with weights $\{w_{q_1}, ..., w_{q_m}\}$. }
	\tcp{Step 1) Compute assignment induced by predicted centers}
	Let $C_j \subset P$ be the cluster centered at $a_j$.
	For a point $p \in C_j$, let $\Delta_j \coloneqq \cost(C_j, A) / |C_j|$ denote the average cost of its cluster $C_j$\label{line:assignment}\;
	\tcp{Step 2) Compute probability distribution}
	For a point $p \in C_j$, let $\Prob{p} \coloneqq \frac{1}{4} \left( \frac{1}{|A| \cdot |C_j|} + \frac{\cost(p, A)}{|A| \cdot \cost(C_j, A)} + \frac{\cost(p, A)}{\cost(P, A)} + \frac{\Delta_j}{\cost(P, A)} \right) $\label{line:prob-distro-oracle}\;
	\tcp{Step 3) Sample $m$ points from $P$}
	\For{$i \longleftarrow 1 $ to $ m $} {\label{line:sample_first}
		Pick a point $q_i$ from $P$ independently with probability $\Prob{q_i}$\;
		Add $q_i$ to $\Omega$ with weight $w_{q_i} \coloneqq 1 / (m \cdot \Prob{q_i})$\;
	}\label{line:sample_last}
	\Return multiset $\Omega = \{ q_1, ..., q_m\}$ with weights $\{ w_{q_1}, ..., w_{q_m}\}$\label{line:return}\;
	
\end{algorithm}

\emph{Comparison with Bansal et al.~\cite{bansal2024sensitivity}.}
\algname\ is closely related to the algorithm of Bansal et al.~\cite{bansal2024sensitivity}, which is the state-of-the-art for coreset size using sensitivity sampling. Specifically, the algorithm by Bansal et al.~\cite{bansal2024sensitivity} begins with the expensive step of computing an initial approximate $k$-means clustering (from scratch) in order to compute probabilities $\Prob{p}$'s. In contrast,  \algname\ derives the values $\Prob{p}$'s using \emph{predictions} provided in the form of input set $A$ of centers. 

When the predictions provided by $A$ yields a constant-factor approximation, i.e., $\cost(P, A) \le \alpha \cdot \OPT_k(P)$ with $|A| = k$, then we are within the setting analyzed by Bansal et al.~\cite{bansal2024sensitivity} (even if they require the computation of set $A$ in the algorithm, while we assume set $A$ to be part of the input). 
The main result of~\cite{bansal2024sensitivity} establishes the optimal coreset-size bound for Euclidean $k$-means (when the coreset is constrained to be a subset of the input), and is reported in the following theorem in terms of Alg.~\ref{alg:sensitivity-sampling-oracle-coreset}.

\begin{theorem}[Theorem 1 of Bansal et al.~\cite{bansal2024sensitivity}] \label{thm:bansal-bound}
	Let the set $A$ of centers in Alg.~\ref{alg:sensitivity-sampling-oracle-coreset} be an $\bigO{1}$-approximation to $k$-means problem.
	Then, by setting coreset size $m = \bigOtilde{ k \varepsilon^{-2} \cdot \min \left( \sqrt{k}, \varepsilon^{-2} \right)}$, Alg.~\ref{alg:sensitivity-sampling-oracle-coreset} outputs a coreset for Euclidean $k$-means clustering with constant probability.
\end{theorem}

While Theorem~\ref{thm:bansal-bound} provides an important theoretical result, its practical applicability is limited: 
computing an $O(1)$-approximate $k$-means solution is unfeasible in practice, with all the known constant-factor approximation algorithms requiring superlinear time (in the input size), and being difficult to implement (e.g., see~\cite{charikar2025improved,cohen2022constantapx}).

However, it is not difficult to see that the analysis of Bansal et al.~\cite{bansal2024sensitivity} leads to the same bound (asymptotically) on the coreset size even when set $A$ induces an $\Bicriteria{\bigO{1}}{\bigO{1}}$ bi-criteria approximation (see Section~\ref{sec:introduction}).
%To this extent, we report in Alg.~\ref{alg:sensitivity-sampling-coreset-bicriteria} the modified version which computes such relaxed approximation $A$ (line~\ref{line:assignment-bicriteria}), with $|A| \leq \beta \cdot k$, for constant $\beta \geq 1$. 
In practice, an $\Bicriteria{\bigO{1}}{\bigO{1}}$ bi-criteria approximation can be computed in $\bigTheta{nkd}$ time by using \texttt{kmeans++}~\cite{arthur2006k} with $2k$ centers, with the following guarantees~\cite{makarychev2020improved}:
\begin{equation*} \label{eq:assignment-kmpp-guarantee}
	\frac{\Exp{\cost(P, A)}}{\OPT_k(P)} \leq 5 \cdot \min\left( 2 + \frac{1}{2e} + \ln2 , 1 + \frac{k}{e(k - 1)} \right)
\end{equation*}
which implies, for $k \geq 2$, an approximation factor $\alpha \le 8.68$ in expectation.\footnote{Note that $\alpha$-approximation guarantee is given in expectation when using \texttt{kmeans++} with $2k$ centers.
Applying Markov's inequality and union bound yields the same coreset size bound of Thm.~\ref{thm:bansal-bound}, with constant probability.} (The expectation is over all the possible $2k$ centers reported by \texttt{kmeans++}.)

\subsection{Analysis}
\label{sec:analysis_ss_predictor}
In this section, we describe our analytical results, which are our main theoretical contributions. Due to space constraints, all proofs are in the Appendix.
First, we analyze our algorithm \algname\ (Alg.~\ref{alg:sensitivity-sampling-oracle-coreset}) in terms of coreset size.
We then show how to obtain provably good predictions for \algname\ (i.e., the input set $A$) from historical data in applications where the input arrives as a sequence of snapshots.
We then analyze the time complexity of our algorithm \algname. Finally, we describe a coreset-based clustering algorithm, based on \algname, for $k$-means clustering on sequences of datasets.

\emph{Coreset Size.} Theorem~\ref{thm:sensitivity-sampling-oracle-coreset-size} describes an upper bound to the number  $m$ of samples for which \algname\ provides a coreset (i.e., with rigorous guarantees on the cost of clustering with respect to the whole dataset). Note that the coreset size $m$ is a function of the quality of the set $A$ of (predicted) centers given in input to \algname.

\begin{theorem} \label{thm:sensitivity-sampling-oracle-coreset-size}
	Let $A$ be such that $\AssignmentBound$, and $k \le |A| \le \beta  k$, with \ourassumptions. 
	Then, by setting the coreset size $m = \bigOtilde{k \varepsilon^{-2} \cdot \min \left(\sqrt{k}, \varepsilon^{-2} \right) \cdot \beta \cdot \max \left(1, \alpha^2 \right)}  =  \bigOtilde{k \varepsilon^{-2} \cdot \min \left(\sqrt{k}, \varepsilon^{-2} \right)}$, Alg.~\ref{alg:sensitivity-sampling-oracle-coreset} outputs a coreset for Euclidean $k$-means clustering with constant probability.
\end{theorem}

Theorem~\ref{thm:sensitivity-sampling-oracle-coreset-size} shows that \algname\ matches  the state-of-the-art bound from Bansal et al.~\cite{bansal2024sensitivity} (see Theorem~\ref{thm:bansal-bound}) on the coreset size, while extending the range of possible values for $\alpha$ and $\beta$ from $\bigO{1}$ to \ourassumptions. The dependence of coreset size $m$ on $\alpha$ and $\beta$ was already known in the literature (e.g., see~\cite{braverman2016new}); however such dependence is not explicit in the state-of-the-art bound of~\cite{bansal2024sensitivity} (which considers a  probability distribution and analytical tools very different from the ones used in~\cite{braverman2016new}), and, to the best of our knowledge, has not been derived before for the probability distribution we consider.

Note that Theorem~\ref{thm:sensitivity-sampling-oracle-coreset-size} has practical implications for our framework of $k$-means with predictions, by explicitly accounting for possibly inaccurate predictions (i.e., when \ourassumptions\ and $\alpha, \beta = \omega(1)$), while still guaranteeing an asymptotically optimal coreset-size bound.

\emph{Predictions in Sequence of Snapshots.}
In the following, we show that, considering the natural model in which the data appears as a sequence of snapshots from the same distribution, the centers computed on a past snapshot can be used as predictions to derive coresets in subsequent snapshots. 
Formally, we consider a sequence of finite pointsets $P_1, P_2, \dots$ drawn i.i.d. from an (unkwown) distribution $\Distro$ over $\mathbb{R}^d$, where $|P_i| = n_i$.
% disclaimer on range of cost functions
To derive our result, we leverage the statistical clustering framework proposed in~\cite{ben2007framework}. As in~\cite{ben2007framework}, we assume that distances lie in $[0, 1]$, allowing simpler formulas for the convergence bounds. Alternatively, one could assume that distances are bounded by some constant $\Delta > 1$: in this case our results follow with a straightforward rescaling of the bounds by a $\Delta^2$ factor.

Let $\mathcal{A}$ be an algorithm that, on pointset $P$, produces a set $\mathcal{A}(P) \subseteq P$ of centers inducing an $\Bicriteria{\alpha}{\beta}$ bi-criteria approximation.
This corresponds to requiring $\cost(P, \mathcal{A}(P)) \leq \alpha \cdot \OPT_k(P)$, and $|\mathcal{A}(P)| = k' \le \beta k$. Let $\OPT_k(\Distro)$ be the minimum value of  $\Expsub{p \sim \Distro}{\cost(p, S)}$, over the choice of the set $S$ of $k$ centers.\footnote{$\Expsub{p \sim \Distro}{\cost(p, S)}$ is the natural objective function for the problem of statistical clustering (e.g., see~\cite{ben2007framework}).}

% result
The following theorem proves that, for two sufficiently large datasets $P_i, P_j$ from the same underlying (and unknown) distribution $\Distro$, ``good'' centers for $P_i$ are ``good'' centers for $P_j$ as well.
\begin{theorem}
	\label{thm:good_predictions}
	For any $\varepsilon \in (0, 1/2]$, and any distribution $\Distro$ over $\mathbb{R}^d$, let $P_i, P_j \subseteq \mathbb{R}^d$ be two pointsets of size $n_i, n_j$, respectively, drawn i.i.d. from $\Distro$. 
	Let $\mathcal{A}$ be an algorithm returning a set $\mathcal{A}(P_i) \subseteq P_i$ of $k'$ centers that yields an $\Bicriteria{\alpha}{\beta}$ bi-criteria approximation for dataset $P_i$. 
	If $n_i, n_j = \widetilde{O} \left( {\frac{k'}{ \varepsilon^2  \OPT_{k'}(\Distro)}} \right) $
	then, with constant probability (over the choices of $P_i$ and $P_j$), $
	\cost\left(P_j, \mathcal{A}\left(P_i\right)\right) \leq \bigO{\alpha} \cdot \OPT_k\left(P_j\right)$.
\end{theorem}

Note that the  sample size in Thm.~\ref{thm:good_predictions} depends on $1 / OPT_{k'}(\Distro)$, where $OPT_{k'}(\Distro) \in [0, 1]$ as we are considering normalized costs. This is inherent in the definition of the problem, since we target \emph{multiplicative} approximation guarantees for bi-criteria approximations. 
%Also, remember that $\OPT_{k'}(\Distro) \in [0, 1]$, as we are considering normalized costs.
Note that in practice $\OPT_{k'}(\Distro) > 0$, since $\OPT_{k'}(\Distro) = 0$ means that $\Distro$ is, in essence, defined on $k'$ points only. We therefore assume that $\OPT_{k'}(\Distro) \ge  \nu$ for some value $\nu > 0$.\footnote{Note that assumptions on the data-generating distribution $\Distro$ are a common requirement in $k$-means clustering. Typical assumptions are: (i) bounded distance $\dist \in [0, \Delta]$ (\cite{ben2007framework,czumaj2007sublinear}); (ii) optimal cluster sizes of at least $\Omega(n \varepsilon / k)$ (\cite{huang2023power,meyerson2004k}); (iii) well-clusterable datasets satisfying $\OPT_{k-1}(P) / \OPT_k(P) \ge 1 + \beta$ (\cite{bansal2024sensitivity}).}
With such an assumption, we obtain the following result.
\begin{corollary}
	\label{corollary:bicriteria_from_snapshots}
	For every $\varepsilon \in (0, 1/2]$, and every distribution $\Distro$ over $\mathbb{R}^d$ such that $\OPT_{k'}(\Distro) \ge \nu$, for some integer $k$', let $P_i, P_j \subseteq \mathbb{R}^d$ be two pointsets of size 
	$
	n_i, n_j = \bigOtilde{ \frac{k'}{ \varepsilon^2 \; \nu }}
	$ 
	drawn i.i.d. via $\Distro$.
	If a set $A \subseteq P_i$ of $k'$ centers induces an $\Bicriteria{\alpha}{\beta}$ bi-criteria approximation on $P_i$, then, with constant probability (over the choices of $P_i$ and $P_j$), the same set $A$ induces an $\Bicriteria{\bigO{\alpha}}{\beta}$ bi-criteria approximation on $P_j$. 
\end{corollary}

\emph{Clustering Algorithm.}
The result of Corollary~\ref{corollary:bicriteria_from_snapshots} gives the following natural clustering algorithm for datasets represented as a sequence of (large enough) snapshots. First, on the initial snapshot, compute a bi-criteria approximation $A$ (e.g., by using \texttt{$k$-means++} with $2k$ centers). The, for each subsequent snapshot, reuse $A$ as prediction in \algname\ to build a theoretically optimal coreset (in the sense of small size, see Thm.~\ref{thm:sensitivity-sampling-oracle-coreset-size}), and finally run a clustering algorithm on such coreset. 

\smallskip
\emph{Relaxing the Fixed Distribution Assumption.}
As stated above, we assume that the pointsets are drawn i.i.d. from the same underlying generating distribution $\Distro$.  Such an assumption may be too restrictive in some settings, where the distribution is \emph{dynamic}, that is, it changes over time.
A natural way to extend Corollary~\ref{corollary:bicriteria_from_snapshots} to pointsets that come from a 
dynamic distribution is to view the two pointsets $P_i$ and $P_j$ as being drawn from \emph{similar}, but not necessarily identical, (unknown) distributions $\Distro_i$ and $\Distro_j$. More specifically, $\Distro_i$ is the distribution generating the pointset $P_i$ on which the centers are computed, while $\Distro_j$ is the distribution generating the later pointset $P_j$, for which we seek our generalization guarantees.

Informally speaking, in this setting one needs to account for the \emph{drift} of the distribution~\cite{mazzetto2024center}, by measuring how much the clustering cost changes when moving from $\mathcal{D}_i$ to $\mathcal{D}_j$. 
This viewpoint is closely related to learning theory under different domains (e.g., see~\cite{ben2010theory}), where generalization depends also on a suitable discrepancy between the considered distributions.
As a consequence, our theoretical results can be adapted to dynamic distributions, at the price of an additional additive term depending on the drift  in Theorem~\ref{thm:good_predictions}.
For example, Corollary~\ref{corollary:bicriteria_from_snapshots} holds also in the case of dynamic distributions whenever the additive term due to distribution drift is small compared to the target optimum $\OPT_k \left( P_j \right)$. 

\smallskip
\emph{Time Complexity.}
Given the predicted centers, \algname\ computes the induced assignment (clusters and costs), and then samples accordingly. 
Therefore, the total time complexity of \algname\ is $\bigTheta{nkd}$, matching the complexity of the ``standard'' sensitivity sampling (i.e., the algorithm of~\cite{bansal2024sensitivity} that computes the approximation of sensitivities via \texttt{kmeans++} with $2k$ centers).
However, the assignment step in \algname\ is more efficient in practice than the computation of the initial approximation required by the algorithm of~\cite{bansal2024sensitivity}, which is obtained by running \texttt{kmeans++} with $2k$ centers. This allows \algname\ to scale more effectively on larger datasets compared to its counterpart not leveraging predictions, as we show in our experimental evaluation (Section~\ref{sec:experiments}).

Note that, if $\nu\ge \bigOtilde{ \frac{k'}{k} \cdot 1 / \min \left( \sqrt{k}, \varepsilon^{-2} \right)}$ in Corollary~\ref{thm:good_predictions}, then for dataset size $n$ and coreset size $m$ it follows $n \le m$. In this case, we simply return the entire dataset to serve as coreset (as it is already small enough for our purposes).

\section{Experiments}
\label{sec:experiments}
In this section, we present our experimental evaluation. 

\smallskip
\emph{Evaluation Procedure.}
% -- compression + optimization
Verifying that a (weighted) set of points satisfies the guarantees of a coreset (see Def.~\ref{def:coreset}) is computationally difficult: it is co-NP-hard~\cite{schwiegelshohn2022empirical} to verify even for \emph{weak coreset}\footnote{A weak coreset only guarantees that a $(1+\varepsilon)$-approximation computed on the coreset yields a $(1+\varepsilon)$-approximation on the full dataset.} guarantees. Moreover, when coresets are used for $k$-means clustering, a common practice~\cite{ackermann2012streamkm++,fichtenberger2013bico} for comparing different approaches for computing coresets is to compare the cost of the clusterings obtained by running the same clustering algorithm on such coresets. The coreset that yields the minimum cost is considered the best.

Overall, we assess the performance of our algorithm \algname\ according to the following criteria: (i) clustering quality in terms of clustering a coreset; (ii) \emph{estimated distortion}, defined below, which assesses strong-coreset guarantees; (iii) runtime; (iv) coreset size. 
Notice that for all metrics, it yields the lower the better. 

% -- baselines
\smallskip
\emph{Baselines.}
We compare our algorithm \algname\ (Alg.~\ref{alg:sensitivity-sampling-oracle-coreset}) with the following algorithms. 
First, we consider the state-of-the-art sensitivity-sampling approach described in Bansal et al.~\cite{bansal2024sensitivity} where the initial $\Bicriteria{\bigO{1}}{\bigO{1}}$ bi-criteria approximation is obtained by running \texttt{kmeans++} with $2k$ centers. We refer to this algorithm as \ssstandard.
Then, we include the following practical heuristic, which we denote with \sstwenty. For each point $p \in P$, we define its sampling probability proportional to the maximum of $\cost(p, A) / \cost(P, A)$ over $20$ independent runs of \texttt{kmeans++} using \emph{exactly} $k$ centers. Note that such sampling probabilities do not account for sizes and costs of clusters (while \algname\ and \ssstandard\ do). Therefore, there is no guarantee that they  upper bound true sensitivities (as required to have coresets guarantees~\cite{feldman2011unified}), but they serve as an empirical proxy for such quantities.
We also consider \fastcoreset~\cite{draganov2024settling}, a sensitivity-based algorithm with quadree embeddings and \emph{hierarchically separated tree} (hist) metric. Since  \fastcoreset\ cannot handle duplicated points, following authors' implementation, for every point in input to \fastcoreset\ uniform noise in $[-5, 5]\times 10^{-5}$ is added i.i.d. in each dimension. 
Finally, we consider \uniformsampling, which samples $m$ points uniformly at random and assigns weight $n/m$ to each sampled point. Note that \uniformsampling\ does not provide coreset guarantees (Def.~\ref{def:coreset}), but it is the only method running in sublinear time.

% -- datasets
\smallskip
\emph{Datasets.}
We consider commonly used real-world datasets from related works. Table~\ref{tab:datasets} summarizes the datasets; further details and visualizations appear in Section~\ref{sec:datasets-appendix} in the Appendix.
We converted each dataset into a \emph{sequence of snapshots} by aggregating points according to fixed temporal intervals (i.e., days, months or years).

\begin{table}[htbp]
	\centering
	\setlength{\tabcolsep}{7.5pt} % Default value: 6pt
	\resizebox{\textwidth}{!}{
		\begin{tabular}{ c c c c c c }
			Dataset & Total Points  & \# Snapshots &  Aggregation & $n_{max}$ & $d$  \\
			\toprule
			Twitter 			    & $14M$ & 7		    			  & Daily 						          & $2.2M$            & 3		   \\
			IntelLab 				& $2.2M$ & 34	  	  	        & Daily  							       & $101k$				 & 6			\\
			Taxi 					   & $1.7M$ & 12		  	 	    & Monthly						        & $161k$				 & 2			\\
			NYC TLC (M)		     & $47M$	 & 4			      & Monthly	   			              & $12.7M$		     & 16		\\
			NYC TLC (Y)		     & $743M$ & 10				    & Yearly	   			              & $146M$			 & 14		\\
			\bottomrule	
		\end{tabular}
	}
	\caption{Datasets' statistics: total number of points across snapshots;  number of snapshots; aggregation interval; maximum number $n_{max}$ of points over all the snapshots; number $d$ of features.}
	\label{tab:datasets}
\end{table}

% -- predictor
\emph{Predictions.}
To empirically assess our theoretical insight from Corollary~\ref{corollary:bicriteria_from_snapshots}, we derive predicted centers for our algorithm \algname\ as described at the end of Section~\ref{sec:analysis_ss_predictor}:
for each dataset, on the \emph{first} snapshot we run  \texttt{kmeans++} to obtain a set $A$ of $2k$ centers which is an $\Bicriteria{\bigO{1}}{\bigO{1}}$ bi-criteria approximation (in expectation); we then reuse the \emph{same} set $A$ of centers for \emph{all} subsequent snapshots in the sequence. 
%\cristian{According to Corollary~\ref{corollary:bicriteria_from_snapshots}, such approach would guarantee valid bi-criteria approximation for the future snapshots (assuming all datasets have been drawn i.i.d. from the same data-generating distribution).}

% -- setup
\smallskip
\emph{Setup.}
We implemented \algname, \ssstandard, \sstwenty, and \uniformsampling\ in C{\small++}17, compiled with \texttt{gcc} 9.4.0 and \texttt{O3} optimization.
We used authors' implementation for \fastcoreset.
All the code was executed on a 2.20 GHz Intel Xeon CPU with 1 TB of RAM, on Ubuntu $20.04$, using a \emph{single thread}. The source code can be found at \url{https://github.com/VandinLab/PreSenS}. 
Following prior works~\cite{ackermann2012streamkm++,fichtenberger2013bico,schwiegelshohn2022empirical,draganov2024settling} we evaluate the number $k$ of clusters $k \in \{10, 20, 50\}$, and coreset size $m = f \cdot k$, with $f \in \{50, 200, 500\}$. 
Since results are similar across $m$, we report plots for $m = 50k$, and defer to Section~\ref{sec:additional-results-appendix} in the Appendix for all results. 
We do not report results for algorithms \fastcoreset\ and \sstwenty\ on NYC TLC datasets, since their running time and memory usage are out of scale compared to the other methods. 
For instance, on the first snapshot of NYC TLC (M) with $k = 10$ and $m = 50k$ (i.e., the smallest parameters we consider), \fastcoreset\ exceeded 167GB peak memory and $7200s$ runtime, while \sstwenty\ required 6GB and $231s$ runtime (other algorithms take less than 3.1GB and $18s$).

% -- compression + opt  snap
\CompactLogRatioCompressionOptimization{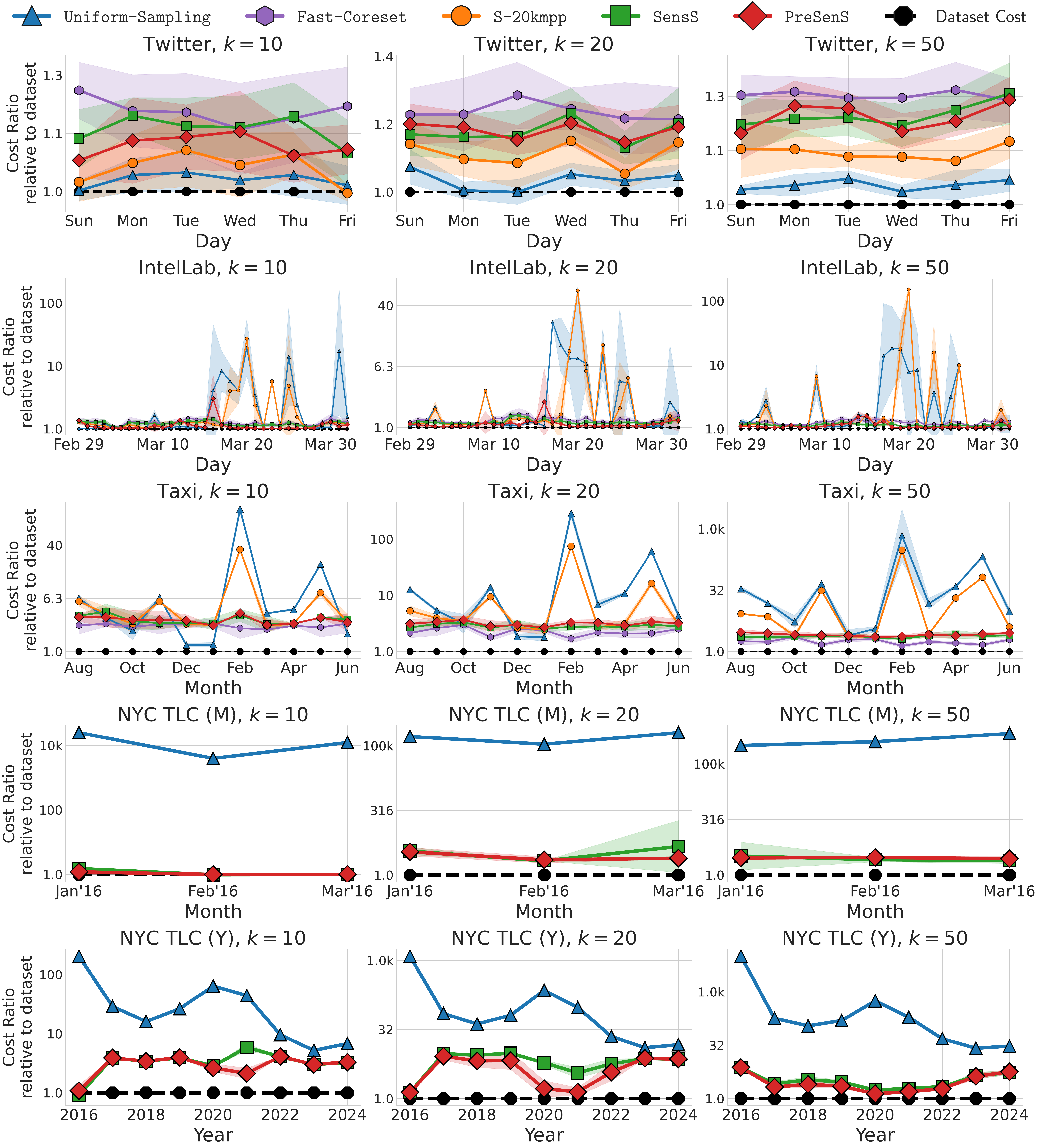}
\smallskip
\emph{Clustering Quality.}
To evaluate the downstream impact of the algorithms to compute coresets, we run \texttt{kmeans++} followed by Lloyd's algorithm until convergence on each coreset, obtaining a set $S$ of centers. 
We then project these centers back onto the full dataset $P$, by running one additional Lloyd's step on $P$, resulting in a set $S_P$ of centroids; we report the final objective value $\cost(P, S_P)$. 
As a reference, we run the same optimization directly on the full dataset and compare the final costs.

Results are reported in Fig.~\ref{fig:log_ratio_compression_optimization_m50}. For readability, we display for each algorithm the ratio between its achieved cost and the cost on the dataset, on a log scale. 
On Twitter, \uniformsampling\ attains the lowest costs, while \sstwenty\ consistently performs the best among sensitivity-based algorithms. 
However, on datasets IntelLab, Taxi, NYC TLC, both \uniformsampling\ and \sstwenty\ show extremely poor performance on several snapshots (on NYC TLC (M), up to $100k$ time worse than \algname), due to uneven distribution of points which are captured only by approximating sensitivities. \fastcoreset\ shows competitive performance on Taxi but slightly worse on IntelLab. 
% Recall that we omit it on NYC TLC due to time and memory constraints. 
Overall, \algname\ matches the performance of \ssstandard, which recomputes a bi-criteria approximation from scratch on each snapshot, suggesting that reusing predicted centers from the first snapshot preserves the quality of coresets to achieve good optimization. 
These observations are consistent with our theoretical insight in Corollary~\ref{corollary:bicriteria_from_snapshots}: a bi-criteria approximation computed on an earlier snapshot can serve as a reliable starting point for producing high-quality coresets on later snapshots. 
Surprisingly, on NYC TLC (Y) our algorithm \algname\ is always comparable or better than \ssstandard: this is due to predicted centers inducing higher-quality approximation of sensitivities with respect to recomputing approximations with \texttt{kmeans++} with 2k centers. 

\smallskip
\sloppy{
\emph{Estimated Distortion on Candidates.} Since (strong) coresets guarantees cannot be efficiently tested, we use the procedure proposed in~\cite{schwiegelshohn2022empirical} to empirically assess the quality of coresets (independently of clustering results). In particular, we empirically estimate the \emph{distortion} of the coreset  $\Omega$ obtained using an algorithm on dataset $P$, where the distortion is
$\sup_{S \subset \mathbb{R}^d \colon |S| = k} \max \left( \frac{\cost(P, S)}{\cost_\Omega(P, S)}, \frac{\cost_\Omega(P, S)}{\cost(P, S)} \right)$. In practice, we use a fixed number of sets $S$ of centers to estimate the distortion, as in~\cite{schwiegelshohn2022empirical}. Concretely, we sample $25$ candidate sets via \texttt{kmeans++}, $25$ candidate sets uniformly at random from the Convex Hull (\texttt{CH}), and $25$ candidate sets uniformly at random in Minimum Enclosing Ball (\texttt{MEB}). 
Additionally, we consider $25$ candidate sets obtained by picking centers uniformly at random.
We apply each of the four sampling methods both to the dataset $P$ and to the coreset $\Omega$, obtaining two sets $\mathcal{C}_P$ and $\mathcal{C}_{\Omega}$. }
Writing $\mathcal{C}_X = \{ S_1, \dots, S_{100}\}$ for $X \in \{P, \Omega \}$, we define the \emph{estimated distortion} as:
\begin{equation*} \label{eq:estimated-distortion}
	\max_{S \in \mathcal{C}_P \cup \mathcal{C}_{\Omega}} \max \left( \frac{ \cost(P, S) }{\cost_{\Omega}(P, S)}, \frac{\cost_{\Omega}(P, S)}{\cost(P, S)} \right).
\end{equation*}

\EstimatedLogRatioDistortion{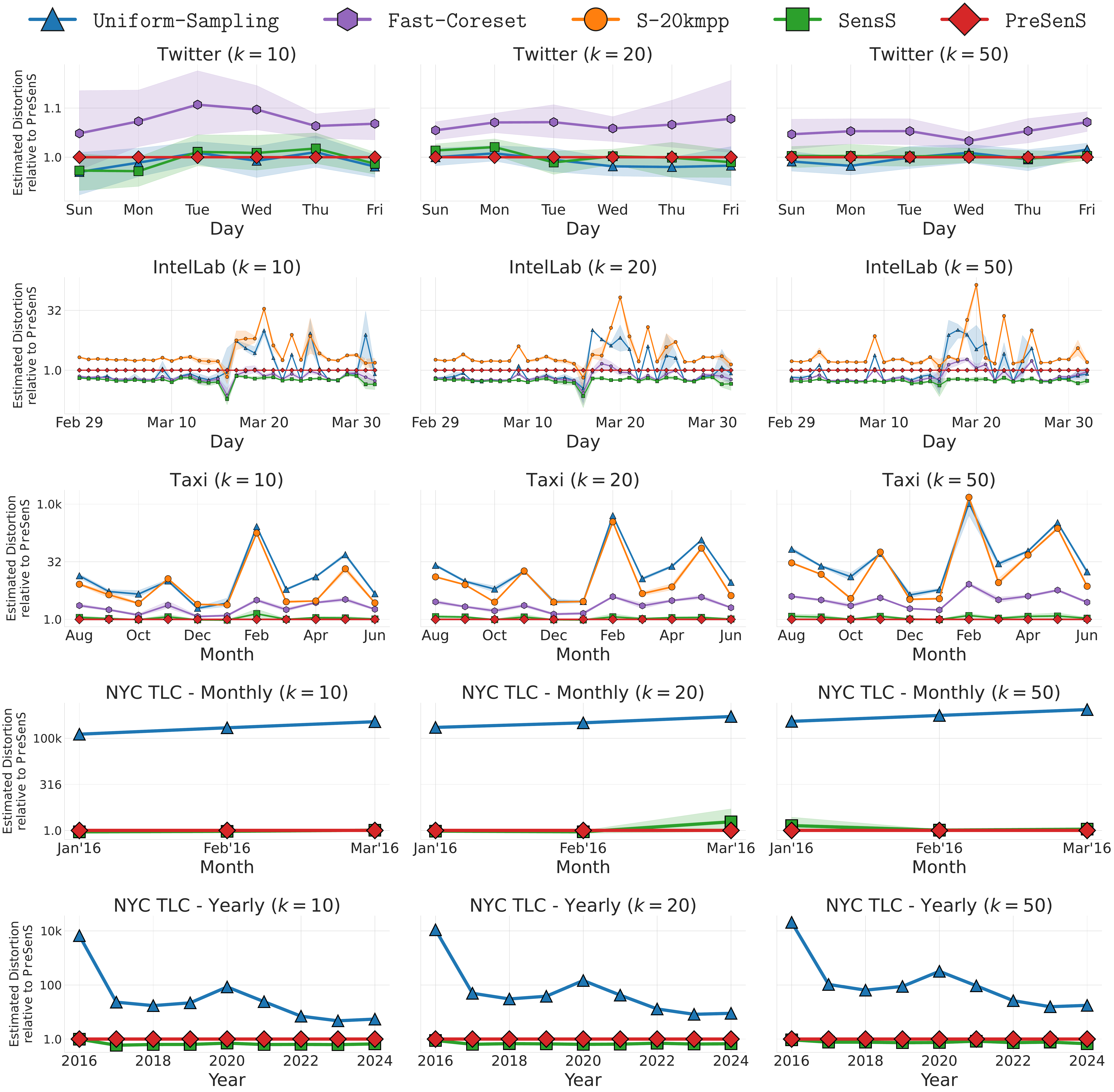}
Fig.~\ref{fig:estimated_log_ratio_distortion_m50} reports the full results.
Again, for readability, we display the ratio between each algorithm's estimated distortion and the estimated distortion of our algorithm \algname, on a log scale.  
On Twitter sequence, estimated distortions are broadly comparable except for \fastcoreset\ and \sstwenty\ (the latter is not displayed for readability, as its distortion is always more than $2.73\times$ that of \algname).

In IntelLab, \fastcoreset\ and \ssstandard\ have better estimated distorsions than \algname, with an estimated distortion ratio of $0.75\times$ and $0.56\times$ compared to the one of \algname, respectively, averaged over all the snapshots and $k$'s. This is due to a drastic shift in the distribution generating the snapshots for IntelLab: let $A$ be the predicted centers (computed via \texttt{kmeans++} with $2k$ centers) from the first snapshot $P_1$. For $k = 10$, the value of $\cost(P_i, A)$ for all $i > 1$ is from $120\times$ and up to $10,7424\times$ the value of $\cost(P_1, A)$. For comparison, such values range from $0.92\times$ and up to $1.06\times$ in Twitter, and from $0.99\times$ and up to $175\times$ in Taxi. Analogous results hold for the other values of $k$ we considered. Note that this matches the theoretical result of our Theorem~\ref{thm:sensitivity-sampling-oracle-coreset-size}: we get ($\varepsilon$-strong) coresets whenever our predicted centers are good enough (in terms of inducing a good bi-criteria approximation).

In Taxi, NYC TLC (M) and NYC TLC (Y), the coresets from our algorithm \algname\  have  distortion comparable to the coresets from \ssstandard; for Taxi, \algname\ is also always better than other methods.
In contrast,  \uniformsampling\ exhibits substantially larger estimated distortion than \ssstandard\ and \algname, similarly to the results of clustering quality (Fig.~\ref{fig:log_ratio_compression_optimization_m50}).

\CompactRuntime{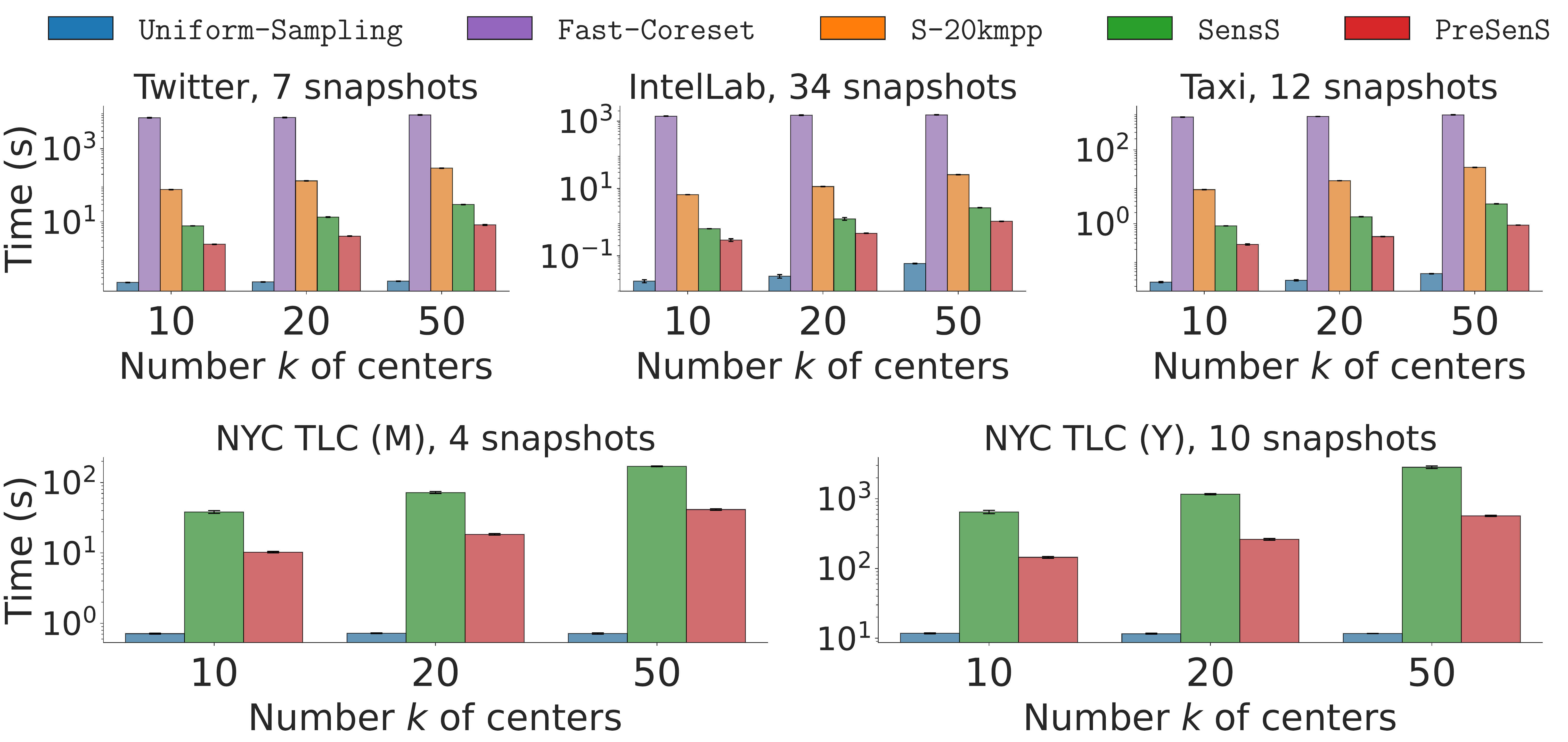}
\emph{Runtimes.}
Fig.~\ref{fig:runtime_m50} displays the sum of runtimes over all the snapshots for each considered sequence. On Twitter, IntelLab and Taxi runtimes for \algname\ are typically well below one second for snapshot.
Across all datasets, using predictions yields a substantial speedup in coreset construction compared to all approaches but \uniformsampling. 
Specifically, comparing our \algname\ with \ssstandard, we observe speedups from $4.1\times$ and up to $5.2\times$ on the largest dataset NYC TLC (Y), and $3.57\times$ on average across all datasets.
%We recall that all experiments have been conducted in a single-thread fashion.  
Notably, \algname\ is not that far from runtimes attained by \uniformsampling\ (the only algorithm with sublinear coreset construction time) on all datasets except the largest ones (NYC TLC), where the gap becomes larger. However, for such datasets  \uniformsampling\ provides extremely poor results (see Fig.~\ref{fig:log_ratio_compression_optimization_m50} and Fig.~\ref{fig:estimated_log_ratio_distortion_m50}). Therefore, our algorithm \algname\ is the fastest approach that provides high-quality coresets across all datasets.

%\vspace{-0.5cm}
\smallskip
\emph{Size of Coresets.}
Finally, we examine the number of \emph{unique} points appearing in the produced coresets.
Since sensitivity-based sampling assigns higher probability to a relatively small set of influential points (the most important points), the resulting multiset (i.e., the coreset) can contain repeated samples; consequently, the number of unique points can be much smaller than $m$, and smaller than what \uniformsampling\ typically yields. Note that a smaller number of unique points implies a lower runtime in practice for processing (e.g., clustering) the set of points.
To quantify this effect, we compute the ratio between the number of unique points in the coresets produced by \algname\ and the number of unique points in the coresets produced by \uniformsampling. 
Table~\ref{tab:ratio_size} shows results for different values of $m$ and $k$.
Interestingly, on datasets Taxi and NYC TLC, \algname\ frequently obtains coresets with fewer than half of the unique points of \uniformsampling. Note that this is not a property of our algorithm itself, but rather of all the sensitivity-based algorithms. 
% table for ratio of coreset size (number of unique points)
\RatioCoresetSize

%\vspace{-1cm}
\section{Conclusion}
In this work, we proposed \algname, an algorithm with predictions for computing sensitivity-based coresets for $k$-means clustering. In particular, we introduced a natural prediction strategy tailored to the common setting where clustering is repeatedly performed on a sequence of (related) datasets. We proved that when these datasets are drawn from the same (unknown) distribution, centers achieving low cost on one dataset can serve as reliable predictors for sensitivity sampling on subsequent datasets, with provable quality guarantees. 
Our experimental results confirm the practical impact of our approach: across a range of dataset sequences, \algname\ consistently achieves a superior trade-off between clustering cost and runtime compared to uniform sampling and state-of-the-art sensitivity-based techniques. Our approach is tailored to Eucliden $k$-means clustering, but the underlying idea of exploiting predictions obtained from clustering previous related datasets is more general: studying how to exploit such predictions and adapting our results to other clustering problems are interesting directions for future research.

\begin{credits}
\subsubsection{\ackname}
Work supported by “National Centre for HPC, Big Data and Quantum Computing” project, CN00000013
(approved under call M42C – Investimento 1.4 – Avvisto “Centri Nazionali” – D.D. n. 3138 of 16.12.2021,
admitted to funding with MUR decree n. 1031 of 06.17.2022).

\subsubsection{\discintname}
The authors have no competing interests to declare that are relevant to the content of this article.

\end{credits}

\bibliographystyle{splncs04}
\bibliography{bibliography}

\newpage
\appendix

% supplementary material / Appendix

\section{Facts and Properties}
\label{sec:facts-appendix}
In the following section, we report some useful properties which we consider in our proofs.

\begin{fact}[Triangle Inequality] \label{fact:triangle-inequalties}
	Let $p_1, p_2, p_3$ be points in $\mathbb{R}^d$. We then have:
	\begin{enumerate}[label = (\roman*)]
		\item $\EuclideanDist{p_1 - p_3}^2 \leq 2 \left( \EuclideanDist{p_1 - p_2}^2 + \EuclideanDist{p_2 - p_3}^2 \right)$; \label{item:tineq-basic}
		% \item For any $\varepsilon > 0, \EuclideanDist{p_1 - p_3}^2 \leq \left( 1 + \varepsilon \right) \EuclideanDist{p_1 - p_2}^2 + \left( 1 + \frac{1}{\varepsilon} \right) \EuclideanDist{p_2 - p_3}^2$;
		\item $\left| \EuclideanDist{p_1 - p_3}^2 - \EuclideanDist{p_1 - p_2}^2 \right| \leq 2 \cdot \EuclideanDist{p_1 - p_2} \cdot \EuclideanDist{p_2 - p_3} + \EuclideanDist{p_2 - p_3}^2$.\label{item:tineq-sqrt}
		%\item $\EuclideanDist{p_1 - p_4}^2 \leq 3 \left( \EuclideanDist{p_1 - p_2}^2 + \EuclideanDist{p_2 - p_3}^2 + \EuclideanDist{p_3 - p_4}^2 \right)$.
	\end{enumerate}
\end{fact}

\begin{theorem}[Berstein's concentration inequality] \label{thm:berstein}
	Let $X_1, ..., X_m$ be independent zero-mean bounded random variables with $|X_i| \leq M$. Let $X = \sum_{i=1}^m X_i$, and $\sigma^2 = \sum_{i=1}^m \Exp{X_i^2}$. Then, for all $t > 0$, we have:
	\begin{equation*}
		\Prob{|X| \geq t} \leq 2 \exp \left( -t^2 / \left( 2\sigma^2 + 2Mt/3 \right)\right)
	\end{equation*}
\end{theorem}

% -- sum of independent gaussians
\begin{fact} \label{fact:sum-of-independent-gaussians}
	Let $a_1, ..., a_m$ be real coefficients and $g_1, ..., g_m$ be independent Gaussians random variables from $\mathcal{N}(0, 1)$. Then, we have $\sum_{i=1}^m a_i g_i \sim \mathcal{N}(0, \sum_{i=1}^m a_i^2).$
\end{fact}

% -- gaussian variance
\begin{fact}[from~\cite{kamath2015bounds}] \label{fact:variance-max-gaussians}
	For $n \geq 2$, let $g_i \sim \mathcal{N}(0, \sigma_i^2)$ be Gaussians random variables such that $\max_i \sigma_i \leq \sigma$, for $i = 1, ..., n$. Then, we have  $\Exp{\max_{i \in [n]} \left| g_i \right| } \leq \sigma \sqrt{2 \log n }.$
\end{fact}

% -- facts of sensitivity-sampling algo
%We state an intrinsic property of the algorithm, useful in the analysis. 
\begin{fact}
	For any point $q \in \Omega$ from cluster $C_j$, for Alg.~\ref{alg:sensitivity-sampling-oracle-coreset}, given the set $A$ with $|A| = k'$, the following holds:
	\begin{equation*}
		w_q \leq 4 \min \left( \frac{k' |C_j|}{m}, \frac{k' \cost(C_j, A)}{m \cost(p, A)}, \frac{\cost(P, A)}{m \cost(p, A)}, \frac{\cost(P, A)}{m \Delta_j} \right).
	\end{equation*} 
	\label{fact:weight-bound}
\end{fact}

%\subsection{Properties of Centroids}
%\label{sec:centroids-properties-appendix}
%
%%\begin{claim} \label{claim:cost-Cj-lower-bound}
%%	Let $S$ be a set of $k$ centers. If for a cluster $C_j$ with center $a_j$ we have $\cost(a_j, S) \geq 32 \Delta_j$, then $\cost(C_j, S) \geq \frac{1}{6} |C_j| \cost(a_j, S).$
%%\end{claim}
%%\begin{proof}
%%	Let $p$ be a point with $\cost(a_j, p) \leq 2 \Delta_j$. By Fact~\ref{fact:triangle-inequalties}~\ref{item:tineq-sqrt}:
%%	\begin{align*}
	%%		& \left| \cost(p, S) - \cost(a_j, S) \right| \\
	%%		& \leq 2 \sqrt{ \cost(a_j, S) \cost(a_j, p)} + \cost(a_j, p) \leq 2 \cost(a_j, S) / 3,
	%%	\end{align*}
%%	since $\cost(a_j, p) \le 2\Delta_j \le \cost(a_j, S) / 16$. Rearranging, we get $\cost(p, S) \geq \cost(a_j, S) / 3.$ Note that the number of points $p$ with $\cost(a_j, p) \leq 2 \Delta_j$ is at least $|C_j| / 2$. Thus, $\cost(C_j, S) \geq |C_j| / 2 \cdot \cost(a_j, S) / 3 \geq \frac{1}{6} |C_j| \cost(a_j, S)$. 
%%\end{proof}

\subsection{Properties of Coresets}
\label{appendinx:properties-of-coreset}
In the following, we state ``good'' properties of coresets achieved by \algname\ (Alg.~\ref{alg:sensitivity-sampling-oracle-coreset}).
We introduce the definition of \emph{rings}, that will be useful later on.

\emph{Ring Partitioning.} For each cluster $C_j$ we define:
\begin{myitemize}
	\item $R_j(0) = \{ p \in C_j : \cost(p, a_j) < 2\Delta_j\}$;
	\item $R_j(\ell) = \{ p \in C_j : \cost(p, a_j) \in \left[2^{\ell}\Delta_j, 2^{\ell + 1}\Delta_j \right)\}$, for $\ell \in [1, \ell_{max}]$;
	\item $R_j(\ell_{max} + 1) = \{ p \in C_j : \cost(p, a_j) \geq 2^{\ell_{max} + 1}\Delta_j\}$.
\end{myitemize}
By setting $\ell_{max} \coloneq \lfloor \log_2 \left( 1 / \varepsilon \right) \rfloor$, the rings $R_j(0), ..., R_j(\ell_{max} + 1)$ partition in a disjoint way all the points in $C_j$, i.e., we have that $\bigcup_{\ell \in [0, \ell_{max} + 1]} R_j(\ell) = C_j$, and for any $\ell \neq \ell'$, $R_j(\ell) \cap R_j(\ell') = \emptyset$.

\vspace{\baselineskip}
%\noindent \textbf{Good Properties of Coreset}.
One can show that, with high probability, the coreset $\Omega$ returned by Alg.~\ref{alg:sensitivity-sampling-oracle-coreset} preserves the number of points in each cluster, as well as its cost. Moreover, it ensures that clusters do not over-sample high-cost points. These key properties are summarized by event $\mathcal{E}$.

\begin{definition}[Event $\mathcal{E}$]
	Event $\mathcal{E}$ occurs if the coreset $\Omega$ returned by Alg.~\ref{alg:sensitivity-sampling-oracle-coreset} satisfies the following three properties:
	\begin{myitemize}
		\item \label{item:p1} {P1: Cluster Size Preservation.} For each cluster $C_j$, $j = \{1, 2, ..., k'\}$:
		$$
		\sum_{q \in \Omega \cap C_j} w_q \in (1 \pm \varepsilon) |C_j|;
		$$
		\item \label{item:p2} {P2: Ring Size Preservation.} For each cluster $C_j$ with $j = \{1, 2, ..., k'\}$, and ring $R_j(\ell)$ with $\ell \in [0, \ell_{max} + 1]$:
		$$
		\sum_{q \in \Omega \cap R_j(\ell)} w_q \leq \frac{|C_j|}{2^{\ell - 1}};
		$$
		\item \label{item:p3} {P3: Cluster Cost Preservation.} For each cluster $C_j$, $j = \{1, 2, ..., k'\}$:
		$$
		\cost_\Omega(C_j, A) \in (1 \pm \varepsilon) \cost(C_j, A).
		$$ 
	\end{myitemize}
	\label{def:eventE}
\end{definition}
Recall that $k \leq k' = |A| \leq \beta k$. Notice that event $\mathcal{E}$ depends only on the sample $\Omega$ (and in particular \emph{does not} place any restriction on the set $S$ of centers). 
We have the following lemma:

\begin{lemma}
	If $m = \bigOmegatilde{ k' \varepsilon^{-2} }$, then event $\mathcal{E}$ holds with probability at least $1 -  \NotEventEInline$.
	\label{lemma:event-E}
\end{lemma}
We show that each of the three properties of event $\mathcal{E}$ holds with high probability. Then, the above lemma follows by a union bound. 

% -- lemma: cluster size preservation
\begin{lemma} \label{lemma:appendix-event-E-cluster-size-preservation}
	If $m = \bigOmegatilde{k' \varepsilon^{-2}}$, then property $P_1$ holds with probability at least $1 -\NotEventEProb$.
\end{lemma}
\begin{proof}
	Let $W_j = \sum_{q \in C_j \cap \Omega} w_q$. For a fixed cluster $C_j$ we show: 
	$$
	\Prob{|W_j - |C_j|| > \varepsilon|C_j|} \leq \varepsilon^4 / (3(k')^4).
	$$
	Lemma follows by a union bound over all $k'$ clusters. 
	For $i \in [m]$, let:
	\[X_i = \begin{cases} 
		w_{q_i} & \text{if $i$-th sample $q_i \in C_j$} \\
		0 & \text{otherwise}.
	\end{cases}
	\]
	Note that $X_i$'s are independent, and $\sum_{i = 1}^m X_i = W_j$. By law of total expectation:
	$$
	\Exp{X_i} = \sum_{q \in C_j} \Prob{q} w_q = |C_j|/m,
	$$
	and thus $\Exp{W_j} = |C_j|$. Next, by Fact~\ref{fact:weight-bound}, we have $w_q \leq 4k' |C_j| / m$ for each $q \in C_j$, and thus $X_i~\in~[0, 4k' |C_j| / m]$. 
	Moreover:
	$$
	\Exp{X_i^2} = \sum_{q \in C_j} \Prob{q} w_q^2 \leq 4k' |C_j|^2 / m^2.
	$$
	Let $Y_i \coloneq X_i - \Exp{X_i}$. Then, $\sum_{i=1}^m Y_i = W_j - |C_j|$. 
	We have $\sum_{i=1}^m \Exp{Y_i^2} =  \sum_{i=1}^m \Var{X_i} \leq  \sum_{i=1}^m \Exp{X_i^2}  \leq 4k' |C_j|^2 / m.$
	By Bernstein's inequality (Thm~\ref{thm:berstein}):
	$$
	\Prob{|W_j - |C_j|| \geq \varepsilon |C_j|} \leq 2 \text{exp} \left( \frac{-\varepsilon^2 |C_j|^2 / 2}{4k' |C_j|^2 / m + 4k' \varepsilon |C_j|^2 / 3m} \right),
	$$
	which is $\leq \varepsilon^4 / (3(k')^4)$ for $m \geq 48 k' \varepsilon^{-2} \log \left( 6 k' \varepsilon^{-1} \right)$.
\end{proof}

\begin{lemma} \label{lemma:appendix-event-E-ring-size-preservation}
	If $m = \bigOmegatilde{ k' \varepsilon^{-2}}$, then property $P_2$ holds with probability at least $1 - \NotEventEProb$.
\end{lemma}
\begin{proof}
	Fix cluster $C_j$ and ring $R_j(\ell)$.
	Let $W_{j, \ell} = \sum_{q \in R_j(\ell) \cap \Omega} w_q$. We want to show:
	$
	\Prob{W_{j, \ell} > |C_j| / 2^{\ell-1}} \leq \varepsilon^4 / (3k^4).
	$
	The lemma follows by a union bound over $k' \cdot \log_2(1 / \varepsilon)$ rings. 
	For $i \in [m]$, let:
	\[X_i = \begin{cases} 
		w_{q_i} & \text{if $i$-th sample $q_i \in R_j(\ell)$} \\
		0 & \text{otherwise}.
	\end{cases}
	\]
	Note that $X_i$'s are independent, and $\sum_{i = 1}^m X_i = W_{j, \ell}$. By law of total expectation:
	$$
	\Exp{X_i} = \sum_{q \in R_j(\ell)} \Prob{q} w_q = |R_j(\ell)|/m,
	$$
	and therefore, $\Exp{W_{j, \ell}} = |R_j(\ell)|$. 
	For any point $p \in R_j(\ell)$, it holds:
	$$
	\frac{\cost(p, A)}{\cost(C_j, A)} \ge \frac{2^{\ell} \Delta_j}{\Delta_j |C_j|} = \frac{2^{\ell}}{|C_j|}.
	$$
	Note that $|R_j(\ell)| \leq |C_j| / 2^\ell$ for each ring, or otherwise the total cost induced by points in that ring would exceed $\cost(C_j, A)$. 
	By Fact~\ref{fact:weight-bound}: 
	$$
	w_q \leq 4k' \cost(C_j, A) / (m \cost(q, A)) \le 4k' |C_j| / (2^{\ell} m )
	$$ 
	for each $q \in R_j(\ell)$, and thus $X_i~\in~[0, 4k' |C_j| / ( 2^{\ell} m )]$. Moreover:
	$$
	\Exp{X_i^2} = \sum_{q \in R_j(\ell)} \Prob{q} w_q^2 \le \frac{4k' |C_j| |R_j(\ell)|}{2^{\ell} m^2} \le \frac{4k' |C_j|^2}{2^{2\ell} m^2}.
	$$
	We have:
	\begin{align*}
		& \Prob{W_{j, \ell} > \frac{|C_j|}{2^{\ell - 1}}} = \Prob{ W_{j, \ell} - |R_j(\ell)| > \frac{|C_j|}{2^{\ell - 1}} - |R_j(\ell)|} \\
		& \le \Prob{ W_{j, \ell} - R_j(\ell) > \frac{|C_j|}{2^\ell} }.
	\end{align*}
	By Bernstein's inequality (Thm~\ref{thm:berstein}):
	$$
	\Prob{W_{j, \ell} - |C_j| > \frac{|C_j|}{2^\ell}} \leq 2 \text{exp} \left( \frac{- |C_j|^2 / 2^{2\ell}}{\frac{4k'|C_j|^2} {2^{2\ell} m} + \frac{4k' |C_j|^2}{2^\ell 3m}} \right),
	$$
	which is $\leq \varepsilon^4 / (3(k')^4)$ for $m \geq 48 k' \varepsilon^{-2} \log \left( 6 k' \varepsilon^{-1} \right)$.
\end{proof}

\begin{lemma} \label{lemma:appendix-event-E-cluster-cost-preservation}
	If $m = \bigOmegatilde{ k' \varepsilon^{-2} }$, then property $P_3$ holds with probability at least $1 - \NotEventEProb$.
\end{lemma}
\begin{proof}
	Fix a cluster $C_j$. For $i \in [m]$, let $X_i$ be $w_q \cost(q_i, A)$ if $q_i \in \Omega \cap C_j$, and $0$ otherwise. Note that $X_i$'s are independent, and $\sum_{i=1}^m X_i = \cost_{\Omega}(C_j, A)$. We have $\Exp{X_i} = \cost(C_j, A) / m$, and $X_i \in [0, 4k' \cost(C_j, A) / m]$ from Fact~\ref{fact:weight-bound}. It holds that $\Exp{X_i^2} \leq 4k' \cost(C_j, A)^2 / m^2$. Let $Y_i = X_i - \Exp{X_i}$. %Then $\sum_{i=1}^m Y_i = \cost_{\Omega}(C_j, A) - \cost(C_j, A)$. 
	By Thm.~\ref{thm:berstein}:
	\begin{align*}
		& \Prob{|\cost_{\Omega}(C_j, A) - \cost(C_j, A)| \geq \varepsilon \cost(C_j, A)} \\
		& \leq 2 \text{exp} \left( \frac{-(\varepsilon \cost(C_j, A))^2 / 2}{4k'\cost(C_j, A)^2 / m + 4k' \varepsilon \cost(C_j, A)^2 / 3m} \right),
	\end{align*}
	which is $\leq \varepsilon^4 / (3(k')^4)$ for $m \geq 48 k' \varepsilon^{-2} \log \left( 6 k' \varepsilon^{-1} \right)$. Main result then follows by a union bound over the $k'$ clusters. 
\end{proof}

We now show how event $\mathcal{E}$ directly affects the total cost of the coreset $\Omega$ (with respect to \emph{any} set $S$ of centers). 
Moreover, we provide an upper bound on the cost of coreset for worst case (when event $\mathcal{E}$ does not hold). 

% -- lemma: lower bound on cost coreset when E holds
\begin{lemma} \label{lemma:appendix-event-E-coreset-cost-lower-bound}
	Consider coreset $\Omega$ from Alg.~\ref{alg:sensitivity-sampling-oracle-coreset}, given set $A$ of (predicted) centers such that $\AssignmentBound$ and $|A| = k' \leq \beta \cdot k$.  If event $\mathcal{E}$ holds, then for any set $S$ of centers, we have $\cost_{\Omega}(P, S) = \bigO{\max(1, \alpha)} \cdot \cost(P, S)$
\end{lemma}

\begin{proof}
	By triangle inequality (Fact~\ref{fact:triangle-inequalties}):
	\begin{align} \label{eq:cost-Omega-triangle-ineq}
		& \cost_{\Omega}(P, S) = \sum_{j = 1}^{k'} \sum_{q \in C_j \cap \Omega} w_q \cost(q, S) \nonumber \\
		& \leq \underbrace{\sum_{j = 1}^{k'} \sum_{q \in C_j \cap \Omega} 2w_q \cost(q, a_j)}_\text{Term 1} + \underbrace{\sum_{j = 1}^{k'} \sum_{q \in C_j \cap \Omega} 2w_q \cost(a_j, S)}_\text{Term 2}
	\end{align} 
	\noindent \textit{Bound on Term 1.}
	We have: 
	\allowdisplaybreaks
	\begin{align} \label{eq:term-1-coreset-lb-event-E}
		& \sum_{j = 1}^{k'} \sum_{q \in C_j \cap \Omega} 2w_q \cost(q, a_j) = 2 \sum_{j = 1}^{k'} \cost_{\Omega}(C_j, A) \nonumber \\
		& \overset{(i)}{\leq} 2 (1 + \varepsilon) \sum_{j = 1}^{k'} \cost(C_j, A) = 2 (1 + \varepsilon) \cost(P, A),
	\end{align}
	where $(i)$ comes from $P_3$ of event $\mathcal{E}$ (Def.~\ref{def:eventE}). 
	
	\vspace{\baselineskip}
	\noindent \textit{Bound on Term 2.} We can write:
	\begin{align} \label{eq:term-2-coreset-lb-event-E}
		& \sum_{j = 1}^{k'} \sum_{q \in C_j \cap \Omega} 2w_q \cost(a_j, S) = 2 \sum_{j = 1}^{k'} \cost(a_j, S) \sum_{q \in C_j \cap \Omega} w_q \nonumber \\ & \overset{(i)}{\leq} 2 \sum_{j = 1}^{k'} \cost(a_j, S) (1 + \varepsilon) |C_j| = 2 (1 + \varepsilon)  \sum_{j = 1}^{k'} \sum_{q \in C_j} \cost(a_j, S) \nonumber \\
		&  \overset{(ii)}{\leq} 2 (1 + \varepsilon)  \sum_{j = 1}^{k'} \sum_{q \in C_j} 2 \left( \cost(q, a_j) + \cost(q, S) \right) \nonumber \\
		& = 4 (1 + \varepsilon) \sum_{j = 1}^{k'} \left( \cost(C_j, a_j) + \cost(C_j, S) \right) = 4 (1 + \varepsilon) \left( \cost(P, A) + \cost(P, S) \right),
	\end{align}
	where $(i)$ comes from property $P_1$ of event $\mathcal{E}$ (Def.~\ref{def:eventE}), and $(ii)$ from triangle inequality. Summing up~\ref{eq:term-1-coreset-lb-event-E} and~\ref{eq:term-2-coreset-lb-event-E}:
	\begin{align*}
		& \cost_{\Omega}(P, S) \leq 2(1 + \varepsilon) (3 \cost(P, A) + 2 \cost(P, S)) = \bigO{\max(1, \alpha)} \cdot \cost(P, S),
	\end{align*}
	since $\AssignmentBound \leq \alpha \cdot \cost(P, S)$, and $\varepsilon \leq 1$.
\end{proof}

% -- lemma: lower bound on cost coreset when E DOES NOT hold
\begin{lemma} \label{lemma:appendix-no-event-E-coreset-cost-lower-bound}
	Consider coreset $\Omega$ from Alg.~\ref{alg:sensitivity-sampling-oracle-coreset}, given set $A$ of (predicted) centers such that $\AssignmentBound$ and $|A| = k' \leq \beta \cdot k$.  Then, in the worst case (when event $\mathcal{E}$ does not hold), for any set $S$ of centers, we have $\cost_{\Omega}(P, S) = \bigO{\max(1, \alpha)} \cdot k'\cdot \cost(P, S)$,
\end{lemma}

\begin{proof}
	The proof is similar to the one of Lemma~\ref{lemma:appendix-event-E-coreset-cost-lower-bound}. Starting from Eq.\ref{eq:cost-Omega-triangle-ineq}: 
	
	\vspace{\baselineskip}
	\noindent \textit{Bound on Term 1.} From Fact~\ref{fact:weight-bound}:
	\begin{align} \label{eq:term-1-coreset-lb-not-event-E}
		& \sum_{j = 1}^{k'} \sum_{q \in C_j \cap \Omega} 2w_q \cost(q, a_j) \overset{}{\leq} \frac{8}{m} \sum_{j = 1}^{k'} \sum_{q \in C_j \cap \Omega} \cost(P, A) \leq 8 \cost(P, A),
	\end{align}
	
	\vspace{\baselineskip}
	\noindent \textit{Bound on Term 2.} 
	\allowdisplaybreaks
	\begin{align} \label{eq:term-2-coreset-lb-not-event-E}
		& \sum_{j = 1}^{k'} \sum_{q \in C_j \cap \Omega} 2w_q \cost(a_j, S) \overset{(i)}{\leq} 4 \sum_{j = 1}^{k'} \sum_{q \in C_j \cap \Omega} w_q \left( \cost(q, a_j) + \cost(q, S) \right) \nonumber \\
		& \overset{(ii)}{\leq} \frac{16k'}{m} \sum_{j = 1}^{k'} \sum_{q \in C_j \cap \Omega} |C_j| ( \cost(q, a_j) + \cost(q, S) ) \nonumber \\
		&=  \frac{16k'}{m} \sum_{j = 1}^{k'} \sum_{q \in C_j \cap \Omega} \sum_{q \in C_j} \left( \cost(q, a_j) + \cost(q, S) \right) \nonumber \\
		& \overset{(iii)}{\leq}  \frac{16k'}{m} \sum_{j = 1}^{k'} m \left( \cost(C_j, A) + \cost(C_j, S) \right) \leq 16k'  \left( \cost(P, A) + \cost(P, S) \right),
	\end{align}
	where $(i)$ comes from triangle ineq., $(ii)$ from Fact~\ref{fact:weight-bound}, and $(iii)$ since $\left| C_j \cap \Omega \right| \leq \left| \Omega \right| = m$.		
	Since $\AssignmentBound \leq \alpha \cdot \cost(P, S)$, summing up~\ref{eq:term-1-coreset-lb-not-event-E} and~\ref{eq:term-2-coreset-lb-not-event-E}:
	\begin{align}
		& \cost_{\Omega}(P, S) = \bigO{\max(1, \alpha)} \cdot k' \cdot \cost(P, S).
	\end{align}
\end{proof}

\section{Analysis: Coreset-Size Bound}
\label{sec:analysis_coreset_size}
The following section is dedicated to the proof of Thm.~\ref{thm:sensitivity-sampling-oracle-coreset-size}, which we restate for completeness.
\begin{manualtheorem}{Theorem~\ref{thm:sensitivity-sampling-oracle-coreset-size}.}
	Let $A$ be such that $\AssignmentBound$, and $k \le |A| \le \beta  k$, with \ourassumptions. 
	Then, by setting the coreset size $m = \bigOtilde{k \varepsilon^{-2} \cdot \min \left(\sqrt{k}, \varepsilon^{-2} \right) \cdot \beta \cdot \max \left(1, \alpha^2 \right)}  =  \bigOtilde{k \varepsilon^{-2} \cdot \min \left(\sqrt{k}, \varepsilon^{-2} \right)}$, Alg.~\ref{alg:sensitivity-sampling-oracle-coreset} outputs a coreset for Euclidean $k$-means clustering with constant probability.
\end{manualtheorem}

The proof is closely related to that of~\cite{bansal2024sensitivity}.
Let $\mathcal{S}$ the set of all possible choices of $k$ centers in $\mathbb{R}^d$.
The goal is to show that Alg.~\ref{alg:sensitivity-sampling-oracle-coreset} produces a (strong) $\varepsilon$-coreset, that is for any set $S \in \mathcal{S}$ of centers it holds (simultaneously):
$$
(1 - \varepsilon) \cost(P, S) \leq \cost_\Omega(P, S) \leq (1 + \varepsilon) \cost(P, S),
$$ 
with constant probability.
Equivalently, we show that the coreset $\Omega$ yields a maximum relative error over all possible placements of centers $S \in \mathcal{S}$ that is bounded by $\varepsilon$:
\begin{equation} \label{eq:eps_coreset_equivalence}
	\Expsub{\Omega}{ \text{sup}_{S \in \mathcal{S}} \left| \frac{\cost(P, S) - \cost_\Omega(P, S)}{\cost(P, S)} \right| } \leq \varepsilon.
\end{equation} 
Then, Thm. 2 follows by standard application of Markov's inequality to Eq.~\ref{eq:eps_coreset_equivalence}.

% As done by~\cite{bansal2024sensitivity}, we proceed by first dividing clusters based on their cost with respect to assignment induced by the (predicted) centers. 
Let $A = \AssignmentCentersBeta$ inducing clusters $\{C_1, ..., C_{k'} \}$. We have, by assumption, $\cost(P, A) \leq \alpha \cdot \OPT_k(P) \leq \alpha \cdot \cost(P, S)$ (for any set $S$ of $k$ centers), and $|A| = k'$ such that $k \leq k' \leq \beta \cdot k$.  Note that $\alpha$ and $\beta$ are not necessarily constants.
To bound the coreset error (Eq.~\ref{eq:eps_coreset_equivalence}), we partition clusters $C_j$'s into groups of similar clusters, and separately bound errors of the coreset on each group. 

\vspace{\baselineskip}
\noindent \textbf{Partitioning Clusters into Far and Close Clusters.}
\label{sec:clusters-partitioning}
For a given $S \in \mathcal{S}$, we first partition $k'$ clusters $C_j$'s (induced by $A$) into \emph{far} and \emph{close} clusters depending on the distance of centroids $a_j$'s from $S$. 
The high-level idea is that, since coresets (approximately) preserve the size of clusters $C_j$'s with high probability (Property 1 of Def.~\ref{def:eventE}), if a cluster $C_j$ is far from $S$, then its centroid $a_j$ is so far away from $S$ that the cost of any point of $C_j$ will essentially be $\cost(a_j, S)$. 
% More precisely, we will show that $\cost(C_j, S) \approx |C_j| \cost(a_j, S)$, and $\cost_{\Omega}(C_j, S) \approx \sum_{q \in C_j \cap \Omega} w_q \cdot \cost(a_j, S)$. 

Let $\Delta_j \coloneq \cost(C_j, A) / |C_j|$ be the average cost of a point in cluster $C_j$, for $j \in [k']$. We say that cluster $C_j$ is \emph{far} from $S$ if $\cost(a_j, S) > \Delta_j \varepsilon^{-2}$, otherwise we say $C_j$ is \emph{close} to $S$. A point $p \in P$ that lies in a far (close) cluster is called a far (close) point with respect to $S$. Let $P_F(S)$ and $P_C(S)$ denote the set of far and close points with respect to $S$. 
% In the following, we separately bound errors of far and close points. 

% -- lemma: cost preservation of farclusters
\begin{lemma}[\textbf{Cost Preservation of Far Clusters}] \label{lemma:cost-preservation-far-points}
	Consider coreset $\Omega$ from Alg.~\ref{alg:sensitivity-sampling-oracle-coreset}, given the set $A$ of (predicted) centers such that $\AssignmentBound$ and $|A| = k' \leq \beta \cdot k$, with \ourassumptions. 
	If the coreset size is $\bigOmegatilde{k \varepsilon^{-2}}$, then
	$
	\Expsub{\Omega}{\sup_{S \in \mathcal{S}} \left| \frac{\cost(P_F(S), S) - \cost_{\Omega}(P_F(S), S)}{\cost(P, S)}\right|} \leq \varepsilon / 2.
	$
\end{lemma}
The proof of Lemma~\ref{lemma:cost-preservation-far-points} is presented in Section~\ref{sec:appendix-far-points}.

% -- lemma: cost preservation of close clusters
\begin{lemma}[\textbf{Cost Preservation of Close Clusters}] \label{lemma:cost-preservation-close-points}
	Consider coreset $\Omega$ from Alg.~\ref{alg:sensitivity-sampling-oracle-coreset}, given the set $A$ of (predicted) centers such that $\AssignmentBound$ and $|A| = k' \leq \beta \cdot k$, with \ourassumptions. 
	If the coreset size is:
	\begin{align*}
		m = \ourcoresetsize,
		%& = \bigOmegatilde{\frac{k}{\varepsilon^2} \cdot \min(\sqrt{k}, \varepsilon^{-2})},
	\end{align*}
	then
	$
	\Expsub{\Omega}{\sup_{S \in \mathcal{S}} \left| \frac{\cost(P_C(S), S) - \cost_{\Omega}(P_C(S), S)}{\cost(P, S)}\right|} \leq \varepsilon / 2.
	$
\end{lemma}
\noindent Finally, Eq.~\ref{eq:eps_coreset_equivalence} will follow from a union bound over Lemma~\ref{lemma:cost-preservation-far-points} and~\ref{lemma:cost-preservation-close-points}. 

\vspace{\baselineskip}
\noindent\textbf{Partitioning Close Clusters based on Cost.}
\label{sec:close-clusters-partitioning}
We now perform an additional step by classifying \emph{close clusters} into \emph{high-cost} and \emph{low-cost} clusters. The latter have a negligible effect on the total cost and can be safely ignored, while the former require a more careful analysis. 

Let $T \coloneq \varepsilon^3 \cdot \cost(P, A) / k'$. 
We say that a close cluster $C_j$ is of low-cost (resp. high-cost) if $\cost(C_j, A)$ is less (resp. greater or equal) than $T$. 
Recall that the notion of "close" cluster is intended with respect to set $S$ of $k$ centers (for which we want to preserve the cost).
We denote as $J_L(S)$ and $J_H(S)$ the set of low-cost and high-cost clusters, respectively. Since low-cost clusters $J_L(S)$ and high-cost clusters $J_H(S)$ partition the set of close clusters $P_C(S)$, preserving the costs for both $J_L(S)$ and $J_H(S)$ implies Lemma~\ref{lemma:cost-preservation-close-points} after rescaling $\varepsilon$.

% -- lemma: cost preservation of low-cost close clusters
\begin{lemma}[\textbf{Cost Preservation of Low-Cost Clusters}] \label{lemma:cost-preservation-low-cost-clusters}
	Consider coreset $\Omega$ from Alg.~\ref{alg:sensitivity-sampling-oracle-coreset}, given the set $A$ of (predicted) centers such that $\AssignmentBound$ and $|A| = k' \leq \beta \cdot k$, with \ourassumptions. 
	If the coreset size is $\bigOmegatilde{k \varepsilon^{-2}}$, then
	$
	\Expsub{\Omega}{\sup_{S \in \mathcal{S}} \left| \frac{\cost(J_L(S), S) - \cost_{\Omega}(J_L(S), S)}{\cost(P, S)}\right|} \leq \varepsilon,
	$
\end{lemma}
\noindent The proof of Lemma~\ref{lemma:cost-preservation-low-cost-clusters} is presented in Section~\ref{sec:appendix-low-cost-clusters}. %For high-cost cluster, we have the following lemma.

% -- lemma: cost preservation of high-cost close clusters
\begin{lemma}[\textbf{Cost Preservation of High-Cost Clusters}] \label{lemma:cost-preservation-high-cost-clusters}
	Consider coreset $\Omega$ from Alg.~\ref{alg:sensitivity-sampling-oracle-coreset}, given the set $A$ of (predicted) centers such that $\AssignmentBound$ and $|A| = k' \leq \beta \cdot k$, with \ourassumptions. 
	If the coreset size is
	$$ 
	m = \bigOmegatilde{k \varepsilon^{-2} \log(k \varepsilon^{-1}) \cdot \min(\sqrt{k}, \varepsilon^{-2})\cdot \beta \cdot \max(1, \alpha^2)},
	$$
	then
	$
	\Expsub{\Omega}{\sup_{S \in \mathcal{S}} \left| \frac{\cost(J_H(S), S) - \cost_{\Omega}(J_H(S), S)}{\cost(P, S)}\right|} \leq \varepsilon.
	$
\end{lemma}

\noindent Proof of Lemma~\ref{lemma:cost-preservation-high-cost-clusters} is handled in Section~\ref{sec:appendix-high-cost-clusters}.

\subsection{Analysis of Far Clusters}
\label{sec:appendix-far-points}
In the following, we prove Lemma~\ref{lemma:cost-preservation-far-points}, which says that for any set $S$ of $k$ centers, the cost of coreset points from \emph{far clusters} approximates, up to $\pm\ \varepsilon \cdot \cost(P, S)$, the true cost of far clusters. More formally, we show that for $m = \bigOmegatilde{k' \varepsilon^{-2}}$:
\begin{equation*}
	\Expsub{\Omega}{\sup_{S \in \mathcal{S}} \left| \frac{\cost(P_F(S), S) - \cost_{\Omega}(P_F(S), S)}{\cost(P, S)}\right|} \leq \varepsilon / 2,
\end{equation*}
provided that \ourassumptions. Above, $P_F(S)$ is the set of clusters $C_j$ such that $\cost(a_j, S) > \Delta_j \varepsilon^{-2}$.
First, we give a bound on the cost difference for far clusters when coreset $\Omega$ from Alg.~\ref{alg:sensitivity-sampling-oracle-coreset} satisfies event $\mathcal{E}$ (Def.~\ref{def:eventE}).
% -- lemma: cost difference event E far clusters
\begin{lemma} \label{lemma:appendix-far-clusters-event-E}
	If coreset $\Omega$ satisfies event $\mathcal{E}$, then for any set $S$ of centers, any cluster $C_j$ that is (deterministically) far from $S$ satisfies:
	$$
	\left| \cost(C_j, S) - \cost_{\Omega}(C_j, S) \right| = \bigO{\varepsilon \cdot \cost(C_j, S)}.
	$$
\end{lemma}
\begin{proof}
	% Remember that set $S$ contains $k$ centers, while clusters $C_j$ are $k' \in [k, \beta k]$ partitions induced by the (predicted) set $A$ of centers in Alg.~\ref{alg:sensitivity-sampling-oracle-coreset}.
	By adding and subtracting $|C_j| \cdot \cost(a_j, S)$ and by applying the triangle inequality, we have:
	\begin{align*}
		& \left| \cost(C_j, S) - \cost_{\Omega}(C_j, S) \right|  \nonumber \\
		% & = \left| \cost(C_j, S) - |C_j| \cdot \cost(a_j, S) + \cost_{\Omega}(C_j, S) + |C_j| \cdot \cost(a_j, S)  \right|  \\
		& \leq \underbrace{\left| \cost(C_j, S) - |C_j| \cdot \cost(a_j, S)  \right|}_\text{Term 1} + \underbrace{\left| \cost_{\Omega}(C_j, S) - |C_j| \cdot \cost(a_j, S)  \right|}_\text{Term 2},
	\end{align*}
	and the claim follows if both Term 1 and Term 2 are $ \bigO{\varepsilon \cdot \cost(C_j, S)}.$
	
	\vspace{\baselineskip}
	\noindent\textit{Bounding Term 1.} Given that rings partition all the points in each cluster $C_j$:
	\allowdisplaybreaks
	\begin{align} \label{eq:ring-partition-cost-bound-event-E}
		& \left| \cost(C_j, S) - |C_j| \cdot \cost(a_j, S)  \right| \leq \sum_{\ell = 0}^{\ell_{max} + 1} \left| \cost(R_j(\ell), S) - |R_j(\ell)| \cdot \cost(a_j, S)  \right| \nonumber \\
		&\leq \sum_{\ell = 0}^{\ell_{max} + 1} \sum_{p \in R_j({\ell})} \left| \cost(p, S) - \cost(a_j, S)  \right|
	\end{align}
	Remember that for each ring it holds $| R_j(\ell) | \leq | C_j | / 2^{\ell - 1}$. Consider $\ell = \ell_{max} + 1 = \lceil \log_2 (1 / \varepsilon) \rceil$. Then:
	\allowdisplaybreaks
	\begin{align} \label{eq:ring-partition-lmax}
		& \sum_{p \in R_j(\ell_{max} + 1)} \left| \cost(p, S) - \cost(a_j, S)\right| \nonumber \\
		& \overset{(i)}{\leq} \sum_{p \in R_j(\ell_{max} + 1)} \left| 2\cost(p, a_j) + 2\cost(a_j, S) - \cost(a_j, S)\right| \nonumber \\
		& = \sum_{p \in R_j(\ell_{max} + 1)} \left| 2\cost(p, a_j) + \cost(a_j, S)\right| \nonumber \\ 
		& \overset{(ii)}{\leq} 2  \sum_{p \in R_j(\ell_{max} + 1)} \cost(p, a_j) + \sum_{p \in R_j(\ell_{max} + 1)} \cost(a_j, S) \nonumber \\
		& \overset{(iii)}{\leq} 2  \sum_{p \in R_j(\ell_{max} + 1)} \cost(p, a_j) + \varepsilon |C_j| \cost(a_j, S) \nonumber \\
		& \overset{(iv)}{\leq} 2  \cost(C_j, a_j) + \varepsilon |C_j| \cost(a_j, S) \nonumber \\
		& \overset{(v)}{\leq} 2  \cost(C_j, a_j) + 2 \varepsilon \left( \cost(C_j, a_j) + \cost(C_j, S) \right) \nonumber \\
		& \overset{(vi)}{<} 2\varepsilon^2 |C_j| \cost(a_j, S) + 2 \varepsilon \left( \varepsilon^2 |C_j| \cost(a_j, S) + \cost(C_j, S)\right) \nonumber \\
		& = \bigO{\varepsilon \cdot \cost(C_j, S)},
	\end{align}
	where $(i), (ii)$ comes from triangle inequalities, $(iii)$ since we have $\left| R_j(\ell_{max} + 1) \right| \leq \varepsilon | C_j |$, $(iv)$ since $R_j(\ell_{max} + 1) \subseteq C_j$, $(v)$ from:
	\begin{align*}
		& |C_j| \cost(a_j, S) = \sum_{q \in C_j} \cost(a_j, S) \\
		& \leq 2 \sum_{q \in C_j} \left( \cost(q, a_j) + \cost(q, S) \right) = 2 \left( \cost(C_j, a_j) + \cost(C_j, S) \right),
	\end{align*}
	and $(vi)$ by the def. of \emph{far clusters}, $\cost(C_j, a_j) = \cost(C_j, A) < \varepsilon^2 |C_j| \cost(a_j, S) $.
	
	Now, consider $\ell \in [0, \ell_{max}]$. We can write:
	\begin{align}\label{eq:cost-p-S-cost-a_j-S-bound}
		& |\cost(p, S) - \cost(a_j, S)| \overset{(i)}{\leq} 2 \sqrt{\cost(p, a_j) \cost(a_j, S)} + \cost(p, a_j) \nonumber \\
		& \overset{(ii)}{<} 2 \sqrt{2^{\ell + 1} \varepsilon^2 \cost(a_j, S)^2} + 2^{\ell + 1} \varepsilon^2 \cost(a_j, S) \leq \bigO{2^{\ell / 2} \varepsilon \cdot \cost(a_j, S)},
	\end{align}
	where $(i)$ comes from Fact~\ref{fact:triangle-inequalties}-\ref{item:tineq-sqrt}, $(ii)$ since, by definition of rings for \emph{far clusters}, $\cost(p, a_j) \leq 2^{\ell + 1} \Delta_j < 2^{\ell + 1} \varepsilon^2 \cost(a_j, S).$ We have:
	\begin{align}  \label{eq:ring-partition-l}
		& \sum_{\ell = 0}^{\ell_{max}} \sum_{p \in R_j(\ell)} \left| \cost(p, S) - \cost(a_j, S) \right| \leq \sum_{\ell = 0}^{\ell_{max}} \sum_{p \in R_j(\ell)} \bigO{2^{\ell / 2} \varepsilon \cdot \cost(a_j, S)} \nonumber \\
		& \overset{(i)}{\leq} \sum_{\ell = 0}^{\ell_{max}} \frac{|C_j|}{2^{\ell - 1}} \ \bigO{2^{\ell / 2} \varepsilon \cdot \cost(a_j, S)} \overset{(ii)}{=} |C_j|  \ \bigO{ \varepsilon \cdot \cost(a_j, S)},
	\end{align}
	where $(i)$ comes from $\left| R_j(\ell) \right| \leq | C_j | / 2^{\ell - 1}$, and $(ii)$ follows since $\sum_{\ell = 0}^{\ell_{max}} 2^{-\ell/2} = \bigO{1}$.
	Gathering~\ref{eq:ring-partition-lmax} and~\ref{eq:ring-partition-l}, we get:
	\begin{align}
		& \text{Term 1} = \left| \cost(C_j, S) - |C_j| \cdot \cost(a_j, S)  \right| \nonumber \\
		& \leq \bigO{\varepsilon \cdot \cost(C_j, S) + |C_j| \varepsilon \cdot \cost(a_j, S) } \nonumber \\
		&\overset{(i)}{\leq} \bigO{\varepsilon \cdot \cost(C_j, S) +  \varepsilon \left( \cost(C_j, S) +  \underbrace{\left| \cost(C_j, S) - |C_j| \cdot \cost(a_j, S)  \right|}_\text{Term 1}  \right) } \nonumber \\
		& = \bigO{\varepsilon \cdot \cost(C_j, S)},
	\end{align}
	where $(i)$ holds since:
	\begin{align} \label{eq:cost-aj-S-bound}
		& |C_j| \cost(a_j, S) = \cost(C_j, S) - \cost(C_j, S) + |C_j| \cost(a_j, S) \nonumber \\
		& \leq \cost(C_j, S) + \left| \cost(C_j, S) - |C_j| \cost(a_j, S) \right|
	\end{align}
	
	\vspace{\baselineskip}
	\noindent\textit{Bounding Term 2.} 
	Let $W_\ell \coloneq \sum_{q \in R_j(\ell) \cap \Omega} w_q$, and $W \coloneq \sum_{\ell = 0}^{\ell_{max} + 1} W_\ell.$ 
	We have $\cost_{\Omega}(C_j, S) = \sum_{\ell = 0}^{\ell_{max} + 1} \cost_{\Omega}(R_j(\ell), S)$. 
	From Term 2:
	\begin{align}\label{eq:term-2-event-E-first-part}
		& \left| \cost_{\Omega}(C_j, S) - |C_j| \cdot \cost(a_j, S)  \right| \nonumber \\
		% &= \left| \cost_{\Omega}(C_j, S) - W \cdot \cost(a_j, S) - |C_j| \cdot \cost(a_j, S)  + W \cdot \cost(a_j, S) \right| \nonumber \\
		& \leq \left| \cost_{\Omega}(C_j, S) - W \cdot \cost(a_j, S) \right| + \left|  W \cdot \cost(a_j, S) - |C_j| \cdot \cost(a_j, S)  \right| \nonumber \\
		& = \left| \cost_{\Omega}(C_j, S) - W \cdot \cost(a_j, S) \right| + \left|  (W - |C_j|) \cdot \cost(a_j, S) \right| \nonumber \\
		& \overset{(i)}{\leq} \left| \cost_{\Omega}(C_j, S) - W \cdot \cost(a_j, S) \right| + \varepsilon |C_j| \cdot \cost(a_j, S)  \nonumber \\
		& \overset{(ii)}{\leq} \left| \cost_{\Omega}(C_j, S) - W \cdot \cost(a_j, S) \right| + \nonumber \\
		& + \varepsilon \left( \underbrace{\cost(C_j, S) + \left| \cost(C_j, S) - |C_j| \cost(a_j, S) \right| }_\text{Term 1} \right) \nonumber \\
		%& \leq \sum_{\ell = 0}^{\ell_{max} + 1} \left| \cost_{\Omega}(R_j(\ell), S) - W_\ell \cdot \cost(a_j, S) \right| + \bigO{\varepsilon^2 \cdot \cost(C_j, S)} \nonumber \\
		& {\le} \sum_{\ell = 0}^{\ell_{max} + 1} \sum_{p \in R_j(\ell) \cap \Omega} \left| w_p \cdot \left( \cost(p, S) - \cost(a_j, S) \right) \right| + \bigO{\varepsilon^2 \cdot \cost(C_j, S)},
	\end{align}
	where $(i)$ uses that if event $\mathcal{E}$ holds, then $\left| W - |C_j| \right| \leq \varepsilon \cdot |C_j|$, and $(ii)$ by Eq.~\ref{eq:cost-aj-S-bound}. We now bound the first term of~\ref{eq:term-2-event-E-first-part}. For $\ell = \ell_{max} + 1 \coloneq \lceil \log_2(1 / \varepsilon) \rceil$:
	\allowdisplaybreaks
	\begin{align}\label{eq:term-2-far-points-lmax}
		& \sum_{p \in R_j(\ell_{max} + 1) \cap \Omega} \left| w_p \left(\cost(p, S) - \cost(a_j, S) \right) \right| \nonumber \\
		& \leq \sum_{p \in R_j(\ell_{max} + 1) \cap \Omega} \left( \left| w_p \cost(p, S) \right| + \left| w_p \cost(a_j, S) \right| \right) \nonumber \\
		& \overset{(i)}{\leq} \sum_{p \in R_j(\ell_{max} + 1) \cap \Omega} \left( 2w_p \cost(p, a_j) + 3 w_p \cost(a_j, S) \right) \nonumber\\
		& \overset{(ii)}{\leq} 2 \cost_{\Omega}(C_j, A) + \sum_{p \in R_j(\ell_{max} + 1) \cap \Omega} 3 w_p \cost(a_j, S) \nonumber \\
		& \overset{(iii)}{\leq}  2 \cost_{\Omega}(C_j, A) + 3 \varepsilon |C_j| \cost(a_j, S) \overset{(iv)}{\leq} 2 \cost_{\Omega}(C_j, A) + \bigO{\varepsilon \cdot \cost(C_j, S)} \nonumber \\
		& \overset{(v)}{=} \bigO{\cost(C_j, A)} + \bigO{\varepsilon \cdot \cost(C_j, S)} \nonumber \\
		& \overset{(vi)}{<}  \bigO{\varepsilon^2 \cdot |C_j| \cost(a_j, S)} + \bigO{\varepsilon \cdot \cost(C_j, S)} = \bigO{\varepsilon \cdot \cost(C_j, S)},
	\end{align}
	where $(i)$ is from Fact~\ref{fact:triangle-inequalties}-\ref{item:tineq-basic}, $(ii)$ since $R_j(\ell_{max} + 1) \subseteq C_j$, $(iii)$ by property $P_2$ of event $\mathcal{E}$, $(iv)$ from Eq.~\ref{eq:cost-aj-S-bound}, $(v)$ from property $P_3$ of event $\mathcal{E}$, and $(vi)$ since for far clusters it holds $\cost(C_j, A) < \varepsilon^2 |C_j| \cost(a_j, S)$. 
	For $\ell \in [0, \ell_{max}]$, we have:
	\begin{align}\label{eq:term-2-far-points-l}
		& \sum_{\ell = 0}^{\ell_{max}} \sum_{p \in R_j(\ell) \cap \Omega} \left| w_p \left(\cost(p, S) - \cost(a_j, S) \right) \right| \nonumber \\
		& \overset{(i)}{=} \sum_{\ell = 0}^{\ell_{max}}  \sum_{p \in R_j(\ell) \cap \Omega} w_p\ \bigO{2^{\ell / 2} \varepsilon \cdot \cost(a_j, S)} \overset{(ii)}{\leq}  \sum_{\ell = 0}^{\ell_{max}} |C_j|/ 2^{\ell - 1} \ \bigO{2^{\ell / 2} \varepsilon \cdot \cost(a_j, S)} \nonumber \\ 
		& \overset{(iii)}{\leq} \sum_{\ell = 0}^{\ell_{max}} 2^{- \ell/2} \ \bigO{\varepsilon \cdot \cost(C_j, S)} = \bigO{\varepsilon \cdot \cost(C_j, S)},
	\end{align}
	where $(i)$ follows from~\ref{eq:cost-p-S-cost-a_j-S-bound}, $(ii)$ by property $P_2$ of event $\mathcal{E}$, $(iii)$ from~\ref{eq:cost-aj-S-bound}. Summing up~\ref{eq:term-2-far-points-lmax} and~\ref{eq:term-2-far-points-l}, we obtain that Term 2 is $\bigO{\varepsilon \cdot  \cost(C_j, S)}$ as well, which concludes the proof for Lemma~\ref{lemma:appendix-far-clusters-event-E}.
\end{proof}

We are now ready to prove Lemma~\ref{lemma:cost-preservation-far-points}, that we restate for completeness.
\allowdisplaybreaks
\begin{manualtheorem}{Lemma~\ref{lemma:cost-preservation-far-points}}[\textbf{Cost Preservation of Far Clusters}] 
	Consider coreset $\Omega$ from Alg.~\ref{alg:sensitivity-sampling-oracle-coreset}, given the set $A$ of (predicted) centers such that $\AssignmentBound$ and $|A| = k' \leq \beta \cdot k$, with \ourassumptions. 
	If the coreset size is $\bigOmegatilde{k \varepsilon^{-2}}$, then:
	\begin{equation*}
		\Expsub{\Omega}{\sup_{S \in \mathcal{S}} \left| \frac{\cost(P_F(S), S) - \cost_{\Omega}(P_F(S), S)}{\cost(P, S)}\right|} \leq \varepsilon / 2.
	\end{equation*}
\end{manualtheorem}
\begin{proof}
	Let $ \Psi(S) \coloneq \left| \frac{\cost(P_F(S), S) - \cost_{\Omega}(P_F(S), S)}{\cost(P, S)} \right| $. When coreset, of size $\bigOmegatilde{k' \varepsilon^{-2}}$ satisfies event $\mathcal{E}$ (with probability at least $1 -\NotEventEInline$), we have that, for each set $S$ of centers, it holds (deterministically) $|\cost(C_j, S) - \cost_{\Omega}(C_j, S)| \leq \varepsilon \cdot \cost(C_j, S)$ for all far clusters $C_j$'s (by Lemma~\ref{lemma:appendix-far-clusters-event-E}). 
	%Hence, by rescaling $\varepsilon$, we have that $\Expsub{\Omega}{\sup_{S \in \mathcal{S}} \Psi(S) \ | \ \mathcal{E}} \leq \varepsilon / 2$. 
	When $\mathcal{E}$ does not hold (with probability less than $\NotEventEInline$), we have, for any set $S$ of centers, $\cost_\Omega(P, S) = \bigO{\max(1, \alpha)} \cdot k' \cdot \cost(P, S)$ by Lemma~\ref{lemma:appendix-no-event-E-coreset-cost-lower-bound}. Then: $\Psi(S)\leq \bigO{\max(1, \alpha)} \cdot k' $. By law of total expectation, we can write $\Expsub{\Omega}{\sup_{S \in \mathcal{S}} \Psi(S)}$ as: 
	\begin{align}
		& = \Expsub{\Omega}{\sup_{S \in \mathcal{S}} \Psi(S) \ \bigg\rvert \ \mathcal{E} } \cdot \Prob{\mathcal{E} } + \Expsub{\Omega}{\sup_{S \in \mathcal{S}} \Psi(S) \ \bigg\rvert \ \bar{\mathcal{E}} } \cdot \Prob{ \bar{\mathcal{E}} } \nonumber \\
		& \leq \bigO{\varepsilon} \cdot (1 - \NotEventEInline) + \bigO{\max(1, \alpha)} \cdot k' \cdot \NotEventEInline,
	\end{align}
	that is $\leq \bigO{\varepsilon}$ for \ourassumptions, concluding the proof of Lemma~\ref{lemma:cost-preservation-far-points} by properly rescaling $\varepsilon$.
\end{proof}

% -- low-cost (close) clusters
\subsection{Analysis of Low-Cost Clusters}
\label{sec:appendix-low-cost-clusters}
Before proving cost preservation for low-cost (close) clusters (stated in Lemma~\ref{lemma:cost-preservation-low-cost-clusters}), we introduce the following results:
\begin{lemma} \label{lemma:claim_low-cost-clusters}
	Fix coreset $\Omega$ and set $S$ of $k$ centers. For any low-cost cluster $C_j \in J_L(S)$:
	\begin{enumerate} [label = (\roman*)]
		\item $\cost(C_j, S) = \varepsilon / k' \cdot \bigO{\cost(P, A)}$; \label{item:i-claim-low-clusters}
		\item If event $\mathcal{E}$ does not hold, $\cost_\Omega(C_j, S) = \bigO{\cost(P, A)}$; \label{item:ii-claim-low-clusters}
		\item If event $\mathcal{E}$ holds, $\cost_\Omega(C_j, S) = \varepsilon / k' \cdot \bigO{\cost(P, A)}$; \label{item:iii-claim-low-clusters}
	\end{enumerate}
\end{lemma} 
\begin{proof}
	Consider any low-cost cluster $C_j$, induced by any set $S$ of $k$ centers. For any $j \in [k]$:
	
	\noindent \textit{Item~\ref{item:i-claim-low-clusters}.} We have:
	\begin{align}
		& \cost(C_j, S) = \sum_{p \in C_j} \cost(p, S) \overset{(i)}{\leq} 2 \sum_{p \in C_j} \left( \cost(p, a_j) + \cost(a_j, S) \right) \nonumber \\
		& \overset{(ii)}{\leq} 2 \sum_{p \in C_j} \left( \cost(p, a_j) +\Delta_j \varepsilon^{-2} \right) \nonumber \\
		& = 2 \left( \cost(C_j, A) + \varepsilon^{-2} \cost(C_j, A) \right) \leq 4\varepsilon^{-2} \cost(C_j, A) \overset{(iii)}{\leq} \frac{4\varepsilon}{k'} \cost(P, A),
	\end{align}
	where $(i)$ follows from Fact~\ref{fact:triangle-inequalties}-\ref{item:tineq-basic}, $(ii)$ since for close clusters $\cost(a_j, S) \leq \Delta_j \varepsilon^{-2}$, and $(iii)$ from low-cost clusters $\cost(C_j, A) < \varepsilon^3 / k' \; \cost(P, A)$.
	
	\noindent \textit{Item~\ref{item:ii-claim-low-clusters}.} Similarly, 
	\allowdisplaybreaks
	\begin{align}
		& \cost_\Omega(C_j, S) = \sum_{p \in C_j \cap \Omega} w_p \cost(p, S) \nonumber \\
		&\leq 2 \sum_{p \in C_j  \cap \Omega} \left( w_p \cost(p, a_j) \right) + 2 \cost(a_j, S)  \sum_{p \in C_j  \cap \Omega} w_p \nonumber \\
		& \overset{(i)}{\leq} 2 \sum_{p \in C_j  \cap \Omega} \frac{4 \cost(P, A)}{m }  +  2 \cost(a_j, S)  \sum_{p \in C_j \cap \Omega} \frac{4k' |C_j|}{m} \nonumber \\ 
		& \overset{(ii)}{\leq}  8 \cost(P, A)+  8k' \cost(a_j, S) |C_j| \nonumber \overset{(iii)}{\leq} 8 \cost(P, A) + 8k' \varepsilon^{-2} \cost(C_j, A) \nonumber \\
		& \overset{(iv)}{\leq} 8 \cost(P, A) + 8 \varepsilon \cost(P, A) = \bigO{\cost(P, A)},
	\end{align}
	where $(i)$ follows from Fact~\ref{fact:weight-bound}, $(ii)$ since $|C_j  \cap \Omega| \leq |\Omega| = m$, $(iii)$ from definition of close clusters, $(iv)$ from definition of low-cost clusters. 
	
	\noindent \textit{Item~\ref{item:iii-claim-low-clusters}.} When $\mathcal{E}$ occurs, we can tighten bounds: 
	\allowdisplaybreaks
	\begin{align}
		& \cost_\Omega(C_j, S) = \sum_{p \in C_j \cap \Omega} w_p \cost(p, S) \nonumber \\
		& \leq 2 \sum_{p \in C_j  \cap \Omega} \left( w_p \cost(p, a_j) \right) + 2 \cost(a_j, S)  \sum_{p \in C_j  \cap \Omega} w_p \nonumber \\
		& \overset{(i)}{\leq} 2 (1 + \varepsilon) \cost(C_j, A) + 2 \cost(a_j, S) (1 + \varepsilon) |C_j|\nonumber \\
		& \overset{(ii)}{\leq} 2 (1 + \varepsilon) \cost(C_j, A) + 2 \Delta_j \varepsilon^{-2} (1 + \varepsilon) |C_j| \overset{(iii)}{\le} \varepsilon / k' \cdot \bigO{\cost(P, A)},
	\end{align}
	where $(i)$ follows from $P_3$, for which $\sum_{p \in C_j \cap \Omega} w_p \cost(p, A) = \cost_{\Omega}(C_j, A) \leq (1 + \varepsilon) \cost(C_j, A)$ and $P_1$, for which $\sum_{p \in C_j \cap \Omega} w_p \leq (1 + \varepsilon) |C_j|$, $(ii)$ by definition of close clusters, and $(iii)$ by definition of low-cost clusters. 
\end{proof}

Now we can prove Lemma~\ref{lemma:cost-preservation-low-cost-clusters}, which we restate for completeness.
% -- lemma: cost preservation of low-cost close clusters
\begin{manualtheorem}{Lemma~\ref{lemma:cost-preservation-low-cost-clusters}}[\textbf{Cost Preservation of Low-Cost Clusters}]
	Consider coreset $\Omega$ from Alg.~\ref{alg:sensitivity-sampling-oracle-coreset}, given the set $A$ of (predicted) centers such that $\AssignmentBound$ and $|A| = k' \leq \beta \cdot k$, with \ourassumptions. 
	If the coreset size is $\bigOmegatilde{k \varepsilon^{-2}}$, then:
	\begin{equation*}
		\Expsub{\Omega}{\sup_{S \in \mathcal{S}} \left| \frac{\cost(J_L(S), S) - \cost_{\Omega}(J_L(S), S)}{\cost(P, S)}\right|} \leq \varepsilon,
	\end{equation*}
\end{manualtheorem}
\begin{proof}
	Let:
	\begin{align*}
		& \Psi(\Omega, S) = \left| \cost(J_L(S), S) - \cost_\Omega(J_L(S), S)) \right| \nonumber \\
		& \leq \sum_{C_j \in J_L(S)} \left| \cost(C_j, S) - \cost_\Omega(C_j, S) \right| \leq \sum_{C_j \in J_L(S)} \left( \cost(C_j, S) + \cost_\Omega(C_j, S) \right) 
	\end{align*}
	If event $\mathcal{E}$ holds, we have $\Psi(\Omega, S) \leq k' \cdot \varepsilon / k' \cdot \bigO{ \cost(P, A)} \leq \varepsilon \cdot \bigO{ \alpha \cdot \cost(P, S)}$ by Lemma~\ref{lemma:claim_low-cost-clusters}. If event $\mathcal{E}$ does not hold, then $\Psi(\Omega, S) \leq k' \cdot \bigO{\alpha \cdot \cost(P, S)}$ (for \emph{any} set of centers). By LTE, we can write $\Expsub{\Omega}{\sup_{S \in \mathcal{S}}  \frac{\Psi(\Omega, S)}{\cost(P, S)} }$ as:
	\begin{align}
		& \Expsub{\Omega}{\sup_{S \in \mathcal{S}}  \frac{\Psi(\Omega, S)}{\cost(P, S)} \ \bigg\rvert \ \mathcal{E}} \Prob{\mathcal{E} } + \Expsub{\Omega}{\sup_{S \in \mathcal{S}}  \frac{\Psi(\Omega, S)}{\cost(P, S)} \ \bigg\rvert \ \bar{\mathcal{E}} } \cdot  \Prob{\bar{\mathcal{E}} } \nonumber \\
		& \leq \bigO{\alpha \cdot \varepsilon} \cdot \left( 1 - \NotEventE \right) + \bigO{\alpha \cdot k'} \cdot \NotEventE = \bigO{\varepsilon'};
	\end{align}
	Lemma~\ref{lemma:cost-preservation-low-cost-clusters} follows by rescaling $\varepsilon = \varepsilon' / \alpha$.
\end{proof}

% -- low-cost (close) clusters
\subsection{Analysis of High-Cost Clusters}
\label{sec:appendix-high-cost-clusters}
The following section is the most technical part of the entire proof, and is entirely dedicated to proving Lemma~\ref{lemma:cost-preservation-high-cost-clusters}. %We first partition again high-cost (close) clusters.

\subsection{Partitioning High-Cost Clusters: Bands and Types}
High-cost clusters are further partitioned such that any two different clusters $C_i, C_j$ in the same partition satisfy $\cost(C_i, A) \approx \cost(C_j, A)$, and $\cost(C_i, S) \approx \cost(C_j, S)$, given the set $A$ of $k'$ (predicted) centers, and a fixed set $S$ of $k$ centers. In this way, we can finely characterize contributions of estimated costs.
To this end, we first introduce \emph{bands} as partitions of clusters with similar cost induced by $A$, and \emph{types} as partitions of clusters with similar cost induced by $S$.

% -- definition: bands
\begin{definition} [\textbf{Bands}] \label{def:bands}
	For each integer $b \in [0, b_{max}]$, $b_{max} \coloneq \lfloor \log_2(k' \varepsilon^{-3}) \rfloor$, we denote as \emph{Band-$b$} the set of clusters $C_j$ with $\cost(C_j, A) \in \left[ 2^b T, 2^{b + 1} T \right)$, where $T \coloneq \varepsilon^3 / k' \ \cost(P, A)$.
\end{definition}
\noindent Since any high-cost cluster $C_j$ has $\cost(C_j, A) \in [T, \cost(P, A)]$, the $(b_{max} + 1)$ bands defined above form a partition of high-cost (close) clusters. 
While clusters in the same band have similar cost induced by $A$, their cost with respect to an arbitrary set $S$ of centers can be very different. %Next, we group clusters based on their cost induced by $S$.
% -- definition: types
\begin{definition} [\textbf{Types}] \label{def:types}
	Let $S$ be a set of $k$ centers. For each integer $t \in [1, t_{max}]$, where $t_{max} \coloneq \lceil \log_2( \varepsilon^{-2}) \rceil$, we denote as \emph{Type-$t$} the set of \emph{close} clusters $C_j$ with $\cost(a_j, S) \in \left[ 2^{t-1} \Delta_j, 2^{t} \Delta_j \right)$, and as \emph{Type-$0$} the set of \emph{close} clusters $C_j$ with $\cost(a_j, S) < \Delta_j$, where $\Delta_j = \cost(C_j, A) / |C_j|$.
\end{definition}
\noindent Since any close cluster $C_j$ satisfies $\cost(a_j, S) \leq \Delta_j \varepsilon^{-2}$, the $(t_{max} + 1)$ types defined above form a partition of the close clusters, inducing a partition of high-cost (close) clusters as well.

For a fixed set $S$ of $k$ centers, let $B_{b, t}(S)$ be the clusters from Band-$b$ and Type-$t$. In total, we can have $(b_{max} + 1) \times (t_{max} + 1) = \bigO{\log^2 (k' / \varepsilon) }$ possible pairs of sets $B_{b, t}(S)$. As such sets $B_{b, t}(S)$ partition points of clusters in $J_H(S)$, Lemma~\ref{lemma:cost-preservation-high-cost-clusters} follows if, for each pair $(b, t) \in [b_{max}] \times [t_{max}]$, the following holds:
\begin{equation} \label{eq:cost-preservation-bands-types}
	\Expsub{\Omega}{\sup_{S \in \mathcal{S}} \left| \frac{\cost(B_{b, t}(S), S) - \cost_{\Omega}(B_{b, t}(S), S)}{\cost(P, S)}\right|} \leq \bigO{ \frac{\varepsilon}{\log^2(k'  \varepsilon^{-1})}};
\end{equation}
then, main result follows from a union bound over all the possible pairs. 
Thus, from now on, we fix band $b$ and type $t$, and we aim to prove Eq.~\ref{eq:cost-preservation-bands-types}.  Also, when clear from the context, we use $B(S)$ to indicate $B_{b, t}(S)$. We denote as $k_{B(S)}$ the number of \emph{clusters} in $B(S)$ (we trivially have $k_{B(S)} \leq k' \leq \beta \cdot k$).

\subsubsection{Defining Cost Vectors and Nets.}
In order to (approximately) discretize the infinitely many centers in $\mathcal{S}$, we proceed to define nets. Instead of discretizing the whole space of centers, we focus on the set of \emph{cost vectors} induced by a given set $S \in \mathcal{S}$.
\begin{definition}[\textbf{Cost Vectors}] \label{def:cost-vectors}
	Let $\Omega = \{ q_1, ..., q_m \}$ be the coreset produced by Alg.~\ref{alg:sensitivity-sampling-oracle-coreset}, and $S$ be a set of centers in $\mathcal{S}$. The cost vector of $\Omega$ induced by $S$ is the vector $\CostVector(\Omega) = \{ \CostVectorEntry{1}(\Omega), ..., \CostVectorEntry{m}(\Omega)\}$, where for each $i \in [m]$:
	\begin{equation*}
		\CostVectorEntry{i}(\Omega)= \begin{cases}
			\cost(q_i, S) & \text{if } q_i \in B(S) \\
			0 & \text{otherwise}.
		\end{cases}
	\end{equation*}
\end{definition}
\noindent Note that $\CostVector(\Omega)$ is a random vector over the choice of $\Omega$, having mutually-independent coordinates. 

For a fixed $\Omega$ and $\mathcal{T} \subseteq \mathcal{S}$, let $\CostVectorNet{}{\Omega}{\mathcal{T}} \coloneq \{ \CostVector(\Omega) | S \in \mathcal{T} \}$ be the set of cost vectors induced by the sets of centers in $\mathcal{T}$. Clearly, if $\mathcal{T}$ has infinitely many centers, then the set $\CostVectorNet{}{\Omega}{\mathcal{T}}$ is unbounded.  
We now define \emph{cost vector nets}, which are discretizations of  $\CostVectorNet{}{\Omega}{\mathcal{T}}$. Formally, a cost vector net  $\CostVectorNet{\gamma}{\Omega}{\mathcal{T}}$ at scale $\gamma$ is a finite set of representative vectors such that any cost vector $\CostVector(\Omega)$ in  $\CostVectorNet{}{\Omega}{\mathcal{T}}$ is $\gamma$-approximated by some vector in the net. 
\begin{definition}[\textbf{Cost Vector Nets}] \label{def:cost-vector-net}
	Let $\gamma \in (0, 1/2]$. For a fixed coreset $\Omega$, and $\mathcal{T} \subseteq \mathcal{S}$, a $(\gamma, \mathcal{T})$-cost vector net, denoted as $\CostVectorNet{\gamma}{\Omega}{\mathcal{T}}$, is a finite subset of $\mathbb{R}^m$ such that, for any $S \in \mathcal{T}$ there exist some $v$ in the net with:
	\begin{equation}
		\left| v_i - \CostVectorEntry{i}(\Omega) \right| = \begin{cases}
			\left| v_i -\cost(q_i, S) \right| \leq \gamma \cdot \textup{err}(q_i, S) & \text{ if } q_i \in B(S) \\
			\left| v_i - 0 \right| = 0, \text{i.e., } v_i = 0 & \text{ otherwise}
		\end{cases}
	\end{equation}
	where $ \textup{err}(q, S) \coloneq \sqrt{\cost(q, S)\cost(q, A)} + \sqrt{\cost(q, S) \Delta_j} + \cost(q, A) + \Delta_j$ for any $q \in C_j$ and $S \in \mathcal{S}$.
\end{definition}

\subsubsection{Grouping Centers based on Interactions}
We proceed to group sets of ``similar'' centers, and construct nets for each group to bound the size of the cost vector nets (Def.~\ref{def:cost-vector-net}). Specifically, for each set $S \in \mathcal{S}$, we define a parameter called \emph{interaction number}, indicating how much the centers in $S$ ``interact'' with clusters in $B(S)$. Roughly, a center point interacts with a cluster if it is significantly far away from the center of the cluster, while still being approximately the nearest center to it. 
\begin{definition}[\textbf{Cluster-Center Interaction}]
	Let $S = \{ x_1, ..., x_k \}$ be a set of $k$ centers. For a center $x_i$, we say that a cluster $C_j$ in $B(S)$ having center $a_j$ \emph{interacts} with $x_i$ if:
	\begin{myitemize}
		\item Point $x_i$ is enough outside average cost ball of $C_j$, i.e., $\cost(a_j, x_i) \geq 32\Delta_j$, and
		\item Point $x_i$ is an approximate nearest center to $a_j$, i.e., $\cost(a_j, x_i) \leq 16 \cost(a_j, S)$
	\end{myitemize}
\end{definition}

%Furthermore, we define the \emph{signature} of $S$ to be:
%\begin{equation*}
%	\textup{Sign}(S) \coloneq \left( \left| \mathcal{I}(x_1) \right|, ..., \left| \mathcal{I}(x_k) \right| \right).
%\end{equation*}

\begin{definition}[\textbf{Interaction Number}]
	Let 
	$$\mathcal{I}(x_i) \coloneq \{ C_j \in B(S) \ | \ C_j \text{ interacts with } x_i\}
	$$ 
	for each $S = \{ x_1, ..., x_k\}$ in $\mathcal{S}$ and $x_i \in S$. We define the interaction number $N(S)$ of $S$ as $N(S) \coloneq \sum_{i=1}^k \left| \mathcal{I}(x_i) \right|$.
\end{definition}
\noindent Our goal is to group centers with \emph{similar interaction number}. Note that each center of $S$ can interact with at most $|B(S)| = k_{B(S)}$ clusters. So, for any set $S$, we have $N(S) \leq k \cdot k_{B(S)} \leq k \cdot k' \leq \beta \cdot k^2$.

\begin{definition}[\textbf{Center Class}] \label{def:center-class-N(S)}
	For an integer $r \in [0, r_{max}]$, $r_{max} \coloneq \lceil 2\log_2 k' \rceil$, the center class $\mathcal{S}(r)$ is the collection of all sets $S$ of $k$ centers with $N(S) \in [2^r, 2^{r+1})$. 
\end{definition}
%\noindent Note that the interaction number $N(S)$ is defined with respect to a specific set $B_{b, t}(S)$ of clusters. As $b$ and $t$ vary, the interaction number and thus the center class of any set $S$ may change. 
Note that center classes partition the collection of all possible sets $S \in \mathcal{S}$. Thus, Eq.~\ref{eq:cost-preservation-bands-types} follows if, for a fixed $r \in [r_{max}]$:
\begin{equation} \label{eq:cost-preservation-center-class}
	\Expsub{\Omega}{\sup_{S \in \mathcal{S}(r)} \left| \frac{\cost(B(S), S) - \cost_{\Omega}(B(S), S)}{\cost(P, S)}\right|} \leq \bigO{ \frac{\varepsilon}{\log^3(k' \varepsilon^{-1})}},
\end{equation} 
and applying a union bound over $\bigO{\log k'}$ many center classes.

% -- bound on cost vector nets size
\vspace{\baselineskip}
The following lemma, taken from~\cite{bansal2024sensitivity} and adapted to our setting with $k'$ clusters, bounds the size of the cost vector net (Def.~\ref{def:cost-vector-net}) for approximating cost vectors of $B(S) \cap \Omega$, given the fixed center class $\mathcal{S}(r)$.
\begin{lemma}[Size of Cost Vector Nets]
	\label{lemma:cost-vector-net-size}
	For any $\gamma \in (0, 1/2]$, there is a $(\gamma, \mathcal{S}(r))$-net $\CostVectorNet{\gamma}{B(S) \cap \Omega}{\mathcal{S}(r)}$ with cardinality:
	$$
	\exp \left( \bigO{\min(2^r + k \gamma^{-2}, 2^t k \gamma^{-2}) \cdot poly \log{ \left( k \gamma^{-1} \varepsilon^{-1} \right) }}\right).
	$$
\end{lemma}
\begin{proof}
	The proof is identical to the one of Lemma E.1 from~\cite{bansal2024sensitivity}. Defining signature of $S$ as $\left( |\mathcal{I}(x_1), ..., \mathcal{I}(x_k) |\right)$, we have that the total number of signatures is at most $(k' + 1)^k = \exp(O(k \log k'))$. Moreover, in Lemma E.3 from~\cite{bansal2024sensitivity}, we need to ``guess'' which clusters are in $B(S)$, whose size is $\le k'$. 
	Hence, we obtain $(\gamma, \mathcal{S}(r))$-net of size $\exp \left( \bigO{\min(2^r + k \gamma^{-2}, 2^t k \gamma^{-2}) \cdot \log{ \left( k' \gamma^{-1} \varepsilon^{-1} \right) }} + k'\right)$; 
	requiring $\beta = \bigO{poly \log(k \varepsilon^{-1})}$ ends the proof. 
\end{proof}

\subsubsection{Applying Symmetrization}
To control the expected supremum of Eq.~\ref{eq:cost-preservation-center-class}, we union bound over the size of cost vector nets (Lemma~\ref{lemma:cost-vector-net-size}) at different scales.
The next step is to bound the variance of the estimator (for different scales of the net). 
We first apply a standard symmetrization argument to fix the randomness of the coreset $\Omega$, and reduce Eq.~\ref{eq:cost-preservation-center-class} to a Gaussian process. Let $\{ g_i \}_{i \in [m]}$ be $m$ independent Gaussians from $\mathcal{N}(0, 1)$. For a fixed $S \in \mathcal{S}(r)$, we define:
\begin{equation} \label{eq:X-rv-definition}
	X^S(\Omega, g) \coloneq \sum_{i \in [m]} \frac{g_i w_{q_i} \CostVectorEntry{i}(\Omega)}{\cost(P, S)}.
\end{equation}
Applying the symmetrization technique of Lemma D.4 from~\cite{bansal2024sensitivity}, we have:
\begin{align*}
	& \Expsub{\Omega}{\sup_{S \in \mathcal{S}(r)} \left| \frac{\cost(B(S), S) - \cost_\Omega(B(S), S)}{\cost(P, S)} \right|} \leq \sqrt{2\pi} \ \Expsub{\Omega}{\Expsub{g}{\sup_{S \in \mathcal{S}(r)} \left| X^S(\Omega, g) \right|}},
\end{align*}
\noindent Thus, Eq.~\ref{eq:cost-preservation-center-class} follows by proving 
\begin{equation} \label{eq:bound-on-X}
	\Expsub{\Omega}{\Expsub{g}{\sup_{S \in \mathcal{S}(r)} \left| X^S(\Omega, g) \right|}} \leq \varepsilon
\end{equation}
and rescaling $\varepsilon$ by $\Theta \left( \log^3 (k' \varepsilon^{-1}) \right)$ factors. 

\subsubsection{Fixing the Randomness of $\Omega$}
The randomness of $X^S(\Omega, g)$ arises from both the choice of $\Omega$ and $g$, and also the expectation of Eq.~\ref{eq:bound-on-X} depends both on $\Omega$ and $g$. 
We now proceed by proving Eq.~\ref{eq:bound-on-X} by fixing $\Omega$ based on whether the coreset $\Omega$ satisfies event $\mathcal{E}$ (Def.~\ref{def:eventE}). 

\begin{lemma}[\textbf{Worst Case Bound}] \label{lemma:worst-case-bound-X}
	For any fixed coreset $\Omega$, we have that $\Expsub{g}{\sup_{S \in \mathcal{S}(r)} \left| X^S(\Omega, g) \right|} \leq \bigO{\alpha \cdot \varepsilon^{-2}}.$
\end{lemma}
\noindent However, if the coreset satisfies event $\mathcal{E}$ we can give a significantly tighter bound.
\begin{lemma}[\textbf{Bound conditioned on Event $\bf{\mathcal{E}}$}] \label{lemma:bound-event-E-X}
	For any fixed $\Omega$ satisfying event $\mathcal{E}$ with \ourassumptions, and having size:
	$$
	m = \ourcoresetsize
	$$
	we have
	$
	\Expsub{g}{\sup_{S \in \mathcal{S}(r)} \left| X^S(\Omega, g) \right|} \leq \bigO{\varepsilon}.
	$
\end{lemma}

\noindent The proofs of Lemmas~\ref{lemma:worst-case-bound-X},~\ref{lemma:bound-event-E-X} are given in the following sections.
By law of total expectation, we can write $ \Expsub{\Omega}{\Expsub{g}{\sup_{S \in \mathcal{S}(r)} \left| X^S(\Omega, g) \right|}}$ as:
\begin{align*}
	& \Expsub{g}{\sup_{S \in \mathcal{S}(r)} \left| X^S(\Omega, g) \right|  \ \bigg\rvert \ \mathcal{E}} \cdot \Prob{\mathcal{E}}+ \Expsub{g}{\sup_{S \in \mathcal{S}(r)} \left| X^S(\Omega, g) \right|  \ \bigg\rvert \ \bar{\mathcal{E}}} \cdot \Prob{\bar{\mathcal{E}}} \\
	&  \leq \bigO{\varepsilon} \cdot \left(1 - \NotEventEInline \right) + \bigO{\alpha \cdot \varepsilon^{-2}} \cdot \NotEventEInline,
\end{align*}
which is $\leq \bigO{\varepsilon}$, concluding the proof of Lemma~\ref{lemma:cost-preservation-high-cost-clusters}. 

\subsubsection{Bound on Gaussian Process}
We now proceed to prove Lemma~\ref{lemma:worst-case-bound-X} and~\ref{lemma:bound-event-E-X}.
\begin{proof}[Lemma~\ref{lemma:worst-case-bound-X}]
	Let $\bf{z}^S \in \mathbb{R}^m$ such that, for any $q_i \in \Omega = \{ q_1, ..., q_m\}$, $z_i^S = (w_{q_i} \CostVectorEntry{q_i}) / \cost(P, S)$. By Cauchy-Schwartz, we can write $X^S = \langle g, \bf{z}^s \rangle \leq \norm{g} \cdot \norm{\bf{z}^S} $. We bound the squared norm of $\bf{z}^S$: 
	\begin{align*}
		& \norm{\bf{z}^S}^2 = \sum_{q \in B(S) \cap \Omega} \frac{w_q^2 \cost(q, S)^2}{\cost(P, S)^2} \overset{(i)}{\leq} 4  \sum_{q \in B(S) \cap \Omega} \frac{w_q^2 \left( \cost(q, a_j) + \cost(a_j, S) \right)^2}{\cost(P, S)^2} \\ 
		& \overset{(ii)}{\leq} 4  \sum_{q \in B(S) \cap \Omega} \frac{w_q^2 \left( \cost(q, a_j) +\Delta_j \varepsilon^{-2} \right)^2}{\cost(P, S)^2} \nonumber \\
		& = 4  \sum_{q \in B(S) \cap \Omega} \frac{w_q^2 \left( \cost(q, a_j)^2 + \Delta_j^2 \varepsilon^{-4} + 2\cost(q, a_j) \Delta_j \varepsilon^{-2} \right)}{\cost(P, S)^2} \\
		& \overset{(iii)}{\leq} \bigO{\frac{\cost(P, A)^2}{m \varepsilon^4 \cost(P, S)^2}} \overset{(iv)}{\leq} \bigO{\frac{\alpha^2}{\varepsilon^4 m} },
	\end{align*}
	where $(i)$ follows from triangle inequality, $(ii)$ since for close clusters $\cost(a_j, S) \leq \Delta_j \varepsilon^{-2}$, $(iii)$ from Fact~\ref{fact:weight-bound}, and $(iv)$ since $\cost(P, A) \leq \alpha \cdot \cost(P, S)$. Combining the above with the fact that $\Expsub{g}{\norm{g}} \leq \sqrt{m}$, we get:
	\begin{align*}
		& \Expsub{g}{\sup_{S \in \mathcal{S}(r)} |X^S(\Omega, g)} \leq \Expsub{g}{\sup_{S \in \mathcal{S}(r)} \norm{g} \cdot \norm{\bf{z}^S}} \\
		& \leq \Expsub{g}{\sup_{S \in \mathcal{S}(r)} \norm{g} \cdot \bigO{\frac{\alpha}{\varepsilon^2 \sqrt{m}}}} \leq \bigO{\frac{\alpha}{\varepsilon^2 \sqrt{m}}} \Expsub{g}{\norm{g}},
	\end{align*}
	which is $\leq \bigO{\alpha \cdot \varepsilon^{-2}},$ concluding the proof.
\end{proof}

% -- bound when event E is satisfied 
In the following, we focus on proving Lemma~\ref{lemma:bound-event-E-X}. Henceforth, we fix a coreset $\Omega = \{ q_1, ..., q_m \}$ satisfying event $\mathcal{E}$; thus, $\CostVector(\Omega)$ is deterministic, and the randomness of $X^S(\Omega, g)$ (Eq.~\ref{eq:X-rv-definition}) is only due to $g$. For the sake of notation, we write $\CostVector$ to denote $\CostVector(\Omega)$.

\emph{The Chaining Argument.} We express each cost vector $\CostVector$ in the net as a telescoping sum. For an integer $h \geq 1$ and a set $S$ of centers in $\mathcal{S}(r)$, let $\mathbf{u}^{S, h}$ be the cost vector net from a $(2^{-h}, \mathcal{S}(r))$ net approximating $\CostVector$. We also let $\mathbf{u}^{S, 0}$ the vector with entries $u_i^{S, 0} = \mathbbm{1}[q_i \in B(S)] \cdot \cost(a_j, S)$, for all $i \in [m]$. We write:
\begin{equation} \label{eq:cost-vector-telescoping-sum}
	\CostVector \coloneq \mathbf{u}^{S, 0} + \sum_{h = 1}^{h_{max}} \left( \mathbf{u}^{S, h} - \mathbf{u}^{S, h - 1} \right) +  \left( \CostVector - \mathbf{u}^{S, h_{max}} \right),
\end{equation}
where $h_{max} \coloneq \lceil 2\log_2 (1 / \varepsilon) \rceil$. Eq.~\ref{eq:cost-vector-telescoping-sum} represents a telescopic sum of the cost vector $\CostVector$, where the first term is an extremely coarse approximation of $\CostVector$; the next $h_{max}$ summands are differences of net vectors at finer scales; the last summand takes the final step to reach $\CostVector$. We can decompose $X^S$ (Def.~\ref{eq:X-rv-definition}) as follows:
\begin{equation}\label{eq:X-rv-decomposition}
	X^S(\Omega, g) = X^{S, Init}(\Omega, g) + \sum_{h=1}^{h_{max}} X^{S, h}(\Omega, g) + X^{S, Fin}(\Omega, g),
\end{equation}
with 
\begin{equation}\label{eq:X-init}
	X^{S, Init}(\Omega, g) \coloneq \frac{\sum_{i \in [m]} g_i w_{q_i} u_i^{S, 0}}{\cost(P, S)},
\end{equation}
\begin{equation}\label{eq:X-fin}
	X^{S, Fin}(\Omega, g) \coloneq \frac{\sum_{i \in [m]} g_i w_{q_i} \left( u_i^S - u_i^{S, h_{max}} \right) }{\cost(P, S)} \text{, and}
\end{equation} 
\begin{equation}\label{eq:X-h}
	X^{S, h}(\Omega, g) \coloneq \frac{\sum_{i \in [m]} g_i w_{q_i} \left( u_i^{S, h} - u_i^{S, h - 1} \right) }{\cost(P, S)}, \text{ for } h \in [1, h_{max}].
\end{equation}
By triangle inequality and linearity of expectation:
\begin{align*}
	& \Expsub{g}{\sup_{S \in \mathcal{S}(r)} \left| X^S(\Omega, g) \right|} \leq \Expsub{g}{\sup_{S \in \mathcal{S}(r)} \left| X^{S, Init}(\Omega, g) \right|} + \\ 
	& + \sum_{h = 1}^{h_{max}} \Expsub{g}{\sup_{S \in \mathcal{S}(r)} \left| X^{S, h}(\Omega, g) \right|} + \Expsub{g}{\sup_{S \in \mathcal{S}(r)} \left| X^{S, Fin}(\Omega, g) \right|}.
\end{align*}
We bound separately each of the three terms on the right-hand side. Henceforth, since we fixed $\Omega$ satisfying $\mathcal{E}$, we indicate $X^{S, h}(\Omega, g)$ as $X^{S, h}$, and similarly for $X^{S}, X^{S, Init}, X^{S, Fin}$. Lemma~\ref{lemma:bound-event-E-X} follows then by rescaling $\varepsilon$. 

% -- bound on X^{S, Init}
\begin{lemma}\label{lemma:bound-on-X-init-appendix}
	For any fixed coreset $\Omega$ of size $\bigOmegatilde{k' \varepsilon^{-2} \cdot \max(1, \alpha^2)}$ satisfying event $\mathcal{E}$, we have:
	$
	\Expsub{g}{\sup_{S \in \mathcal{S}(r)} \left| X^{S, Init} \right| } \leq \varepsilon.
	$
\end{lemma}
\begin{proof}
	The numerator of $|X^{S, Init}|$ (Eq.~\ref{eq:X-init}) is upper bounded as:
	\begin{align*}
		& \left| \sum_{i \in [m]} g_i w_{q_i} u_i^{S, 0} \right| \overset{}{=} \left| \sum_{C_j \in B(S)} \sum_{q_i \in C_j \cap \Omega} g_i w_{q_i} \cost(a_j, S) \right|  \\ 
		& \le \sum_{C_j \in B(S)} \left|  \sum_{q_i \in C_j \cap \Omega} g_i w_{q_i} \cost(a_j, S) \right|
	\end{align*}
	%where $(i)$ follows from the definition of $\mathbf{u}^{S, 0}$, and $(ii)$ from the triangle inequality.
	
	Then, we lower bound the denominator:
	\begin{align*}
		\cost(P, S) & \overset{(i)}{\geq} \frac{1}{2} \left( \cost(P, S) + \cost(P, A) / \alpha \right) \\
		& \overset{(ii)}{\geq} \frac{1}{2} \left( \sum_{C_j \in B(S)} \left( \cost(C_j, S) + \cost(C_j, A) / \alpha \right) \right) \\
		%& \geq \frac{1}{2} \left( \sum_{C_j \in B(S)} \left( \cost(C_j, S) + \cost(C_j, A) \right) \cdot \min(1, 1/ \alpha) \right) \\ 
		& \overset{(iii)}{\geq} \frac{1}{4} \sum_{C_j \in B(S)} |C_j| \cost(a_j, S) \cdot \min(1, 1 / \alpha);
	\end{align*}
	$(i)$ follows from $\cost(P, A) \leq \alpha\ \cost(P, S)$, $(ii)$ since $B(S) \subseteq P$, and $(iii)$ from:
	\begin{align*}
		& |C_j| \cost(a_j, S) = \sum_{q \in C_j} \cost(a_j, S) \\
		& \leq 2 \sum_{q \in C_j} \left( \cost(q, S) + \cost(q, a_j) \right) \leq 2 \left( \cost(C_j, S) + \cost(C_j, A) \right).
	\end{align*}
	Combining these bounds, we obtain:
	\begin{align*}
		& \left| X^{S, Init} \right|  \leq 4 \ \frac{\sum_{C_j \in B(S)} \left|  \sum_{q_i \in C_j \cap \Omega} g_i w_{q_i} \cost(a_j, S) \right|}{ \sum_{C_j \in B(S)} |C_j| \cost(a_j, S) \cdot \min(1, 1 / \alpha)} \\ 
		& \leq 4 \ \max_{C_j \in B(S)} \frac{\left| \sum_{q_i \in C_j \cap \Omega} g_i w_{q_i} \right| \cdot \max(1, \alpha)}{|C_j|}.
	\end{align*}
	The above is independent of $S$, and is the maximum of $|B(S)| = k_{B(S)} \leq k' \leq k \cdot \beta$ Gaussians. We can write:
	\begin{align}\label{eq:variance-sum-of-gaussians-appendix}
		\Varsub{g}{\sum_{q_i \in C_j \cap \Omega} g_i w_{q_i}} & \overset{(i)}{=} \sum_{q \in C_j \cap \Omega} w_q^2 \overset{(ii)}{\leq} \frac{4k' |C_j|}{m} \sum_{q \in C_j \cap \Omega} w_q \overset{(iii)}{\leq} \frac{8k' |C_j|^2}{m},
	\end{align}
	where $(i)$ follows from Fact~\ref{fact:sum-of-independent-gaussians}, $(ii)$ from Fact~\ref{fact:weight-bound} and $(iii)$ from property $P_1$ of event $\mathcal{E}$.
	From Fact~\ref{fact:variance-max-gaussians}:
	\begin{align*}
		\Expsub{g}{\sup_{S \in \mathcal{S}(r)} \left| X^{S, Init} \right| } \leq \frac{4 \max(1, \alpha)}{|C_j|} \cdot \sqrt{\frac{8k' |C_j|^2}{m}} \cdot \sqrt{2 \log k'},
	\end{align*}
	which ends the proof for the suggested choice of $m$.
\end{proof}
%\noindent \cristian{NOTE: as long as $\alpha$ is $\bigO{polylog(k)}$, the above bound for coreset size (for $X^{S, Init}$) is still $\bigOmegatilde{k / \varepsilon^2}$}

% -- bound on X^{S, Fin}
\begin{lemma}\label{lemma:bound-on-X-fin-appendix}
	For any fixed coreset $\Omega$ that satisfies event $\mathcal{E}$ we have that \\
	$
	\Expsub{g}{\sup_{S \in \mathcal{S}(r)} \left| X^{S, Fin} \right| } \leq \varepsilon.
	$
\end{lemma}
\begin{proof}
	Let $\mathbf{z}^S \in \mathbb{R}^m$ be the vector with $z_i^S \coloneq \left( w_{q_i} \left( u_i^S - u_i^{S, h_{max}} \right) \right) / \cost(P, S)$. We can express $X^{S, Fin} = \langle g, \mathbf{z}^S \rangle$. Since $u_i^{S, h_{max}}$ is the finest approximation available for $\CostVector$ (in fact, it is a $2^{-h_{max}} = \varepsilon^2$-approximation), we can show that for any $S \in \mathcal{S}$, the vector $\mathbf{z}^S$ has small norm. 
	\begin{align*}
		& \norm{\mathbf{z}}^2 = \sum_{i \in [m]} \frac{w_{q_i}^2 \left( u_i^S - u_i^{S, h_{max}} \right)^2}{\cost(P, S)^2} \leq \sum_{C_j \in B(S)} \sum_{q \in C_j \cap \Omega} \frac{w_q^2 \cdot 2^{-2h_{max}} \cdot \textup{err}^2(q, S)}{\cost(P, S)^2} \\ 
		&= \varepsilon^4 \sum_{C_j \in B(S)} \sum_{q \in C_j \cap \Omega} \frac{w_q^2 \cdot \textup{err}^2(q, S)}{\cost(P, S)^2} \nonumber \\
		& \le \frac{\varepsilon^4}{\cost(P, S)^2} \sum_{C_j \in B(S)} \sum_{q \in C_j \cap \Omega} w_q^2 O \bigl( \underbrace{\cost(q, S)(\cost(q, A) + \Delta_j)}_\text{Summand 1} +  \nonumber \\ 
		& + \underbrace{\cost(q, A)^2}_\text{Summand 2} + \underbrace{\Delta_j^2}_\text{Summand 3} \bigr),
	\end{align*} 
	where $\textup{err}(q, S) = \sqrt{\cost(q, S)\cost(q, A)} + \sqrt{\cost(q, S) \Delta_j} + \cost(q ,A) + \Delta_j$ from Def.~\ref{def:cost-vector-net}, for $q \in C_j$.
	We bound each of the three summands.
	
	Summand 1:
	\begin{align*}
		& \frac{\varepsilon^4}{\cost(P, S)^2} \sum_{C_j \in B(S)} \sum_{q \in C_j \cap \Omega} w_q^2 \cdot \bigO{ \cost(q, S) \left( \cost(q, A)  + \Delta_j \right)} \\
		&  \overset{(i)}{\leq} \frac{\varepsilon^4 \ \cost(P, A)}{m \ \cost(P, S)^2} \sum_{C_j \in B(S)} \sum_{q \in C_j \cap \Omega} w_q \cdot \bigO{ \cost(q, S)} \\
		& = \frac{\varepsilon^4 \ \cost(P, A)}{m \ \cost(P, S)^2} \cdot \bigO{ \cost_\Omega(P, S)} \\
		& \overset{(ii)}{\leq} \frac{\varepsilon^4 \ \cost(P, A)}{m \ \cost(P, S)} \cdot \bigO{\max(1, \alpha)} \overset{(iii)}{\leq} \bigO{\frac{\varepsilon^4 \cdot \max(1, \alpha^2)}{m}},
	\end{align*}
	where $(i)$ follows from Fact~\ref{fact:weight-bound}, $(ii)$ from Lemma~\ref{lemma:appendix-event-E-coreset-cost-lower-bound}, and $(iii)$ from $\cost(P, A) \leq \alpha \cdot \cost(P, S)$.
	
	Summand 2:
	\begin{align*}
		& \frac{\varepsilon^4}{\cost^2(P, S)} \sum_{C_j \in B(S)} \sum_{q \in C_j \cap \Omega} w_q^2 \cdot \bigO{ \cost(q, A)^2} \\
		& \overset{(i)}{\leq} \frac{\varepsilon^4}{\cost^2(P, S)} \sum_{C_j \in B(S)} \sum_{q \in C_j \cap \Omega} \frac{\cost(P, A)^2}{m^2 \cost(q, A)^2} \cdot \bigO{ \cost(q, A)^2} \\
		& \overset{(ii)}{\leq} \bigO{\frac{\varepsilon^4 \cost(P, A)^2}{m \ \cost(P, S)^2}} \overset{(iii)}{\leq} \bigO{\frac{\varepsilon^4 \alpha^2}{m}} \leq \bigO{\frac{\varepsilon^4 \max(1, \alpha^2)}{m}},
	\end{align*}
	where $(i)$ follows from Fact~\ref{fact:weight-bound}, $(ii)$ since $\left| \sum_{C_j \in B(S)} C_j \cap \Omega \right| \leq m$ and $(iii)$ from $\cost(P, A) \leq \alpha \cdot \cost(P, S)$.
	
	Summand 3 follows from same steps of Summand 2, but using\\
	$w_q \leq \bigO{\cost(P, A) / \left( m \Delta_j \right)}$ from Fact~\ref{fact:weight-bound}.
	
	We showed that $\norm{\mathbf{z}^S}^2 \leq \bigO{\left( \varepsilon^4 \max(1, \alpha^2) \right) / m}$. From Cauchy-Schwartz:
	\begin{align*}
		& \Expsub{g}{X^{S, Fin}} = \Expsub{g}{\langle g \cdot \mathbf{z}^S \rangle} \leq \Expsub{g}{\norm{g} \cdot \norm{\mathbf{z}^S}} \leq \frac{\varepsilon^2 \max(1, \alpha)}{\sqrt{m}} \cdot \Expsub{g}{\norm{g}}
	\end{align*}
	which is $\leq \varepsilon^2 \max(1, \alpha)$, since $\Expsub{g}{||g||} \leq \sqrt{m}$ for a standard Gaussian vector $g \in \mathbb{R}^m$.
	The above is $\leq \varepsilon'$ by taking $\varepsilon = \sqrt{\varepsilon' / \max(1, \alpha)}$. (Assuming $\alpha = poly \log(k, \varepsilon^{-1})$.)
	
\end{proof}

% -- bound on X^{S, h}
We are left to prove the bound for a generic $h \in [1, h_{max}]$. We first introduce the following bound on the variance:
\begin{lemma} \label{lemma:variance-bound-in-term-of-costs}
	Let $S$ be any set of $k$ centers, and $h$ be an integer. The random state $X^{S, h}$ is a zero-mean Gaussian with $\Var{X^{S, h}} \leq 2^{-2h} / m \cdot \bigO{ \max(1, \alpha^2)}$.
\end{lemma}
\begin{proof}
	Since $X^{S, h}$ is the sum of (the absolute value) of independent zero-mean Gaussians, from Fact~\ref{fact:sum-of-independent-gaussians} it is also a zero-mean Gaussian with variance:
	\begin{align*}\label{eq:firs-bound-on-variance}
		\Var{X^{S, h}} = \frac{1}{\cost(P, S)^2} \cdot \sum_{i \in [m]} w_{q_i}^2 \left( u_i^{S, h} - u_i^{S, h - 1} \right)^2.
	\end{align*}
	Moreover, by definition of cost vector net (Def.~\ref{def:cost-vector-net}), we have that, if $q_i \in B(S)$:
	\begin{align*}
		\left| u_i^{S, h} - u_i^{S, h-1}\right| \leq \left| u_i^{S, h} - u_i^{S} \right| + \left| u_i^{S, h - 1} - u_i^{S} \right| \leq 3 \cdot 2^{-h} \textup{err}(q_i, S),
	\end{align*}
	yielding $\Var{X^{S ,h}} \leq 9 / (\cost(P, S)^2) \cdot 2^{-2h} \cdot \sum_{q \in B(S) \cap \Omega} w_{q_i}^2 \textup{err}(q, S)^2.$
	Similarly to proof of Lemma~\ref{lemma:bound-on-X-fin-appendix} (using scale $h$ instead of $h_{max}$ to bound the summands), we can derive $\Var{X^{S, h}} \leq 2^{-2h} \cdot \cost(P, A) / (m  \ \cost(P, S)) \cdot \bigO{\max(1, \alpha)}$.
\end{proof}

Now, we tighten bounds, based on specific conditions of the considered clusters (namely, depending on their type $t$).
We have the following result:
\begin{lemma} \label{lemma:variance-bound-second}
	Let $S$ be any set of centers. We have $\Var{X^{S, h}} \leq \bigO{2^{-2h - t}} \cdot k' / \left( m \ k_{B(S)} \right) \cdot  \max(1, \alpha^2)$.
\end{lemma}
\begin{proof}
	\emph{Case 1: $t < 6$}. The claim follows trivially from Lemma~\ref{lemma:variance-bound-in-term-of-costs} since $k' \geq k_{B(S)}$, and $2^t$ is a constant. 
	
	\emph{Case 2: $t \geq 6$}. 
	From Eq.~\ref{eq:firs-bound-on-variance} we can derive:
	\begin{align*}
		& \Var{X^{S, h}} \leq \bigO{ \frac{2^{-2h}} {\cost(P, S)^2}}\cdot \sum_{q \in B(S) \cap \Omega} w_q^2 \cdot \textup{err}(q, S)^2 \le \bigO{ \frac{2^{-2h}} {\cost(P, S)^2}} \cdot \\
		& \cdot \sum_{q \in B(S) \cap \Omega} w_q^2 \cdot  \bigl( \cost(q, S)\cost(q, A) + \cost(q, S)\Delta_j + \cost(q, A)^2 + \Delta_j^2 \bigr)
	\end{align*}
	Let 
	$
	\Gamma \coloneq \sum_{B(S) \cap \Omega} w_q^2 \cdot  \left( \cost(q, S)\cost(q, A) + \cost(q, S)\Delta_j + \cost(q, A)^2 + \Delta_j^2 \right).
	$
	\begin{align*}
		& \Gamma \overset{(i)}{\leq} \bigO{\frac{k' 2^b T}{m}} \sum_{B(S) \cap \Omega} w_q \left( 2\cost(q, S) + \cost(q, A) + \Delta_j \right) \\
		&  \overset{(ii)}{\leq} \bigO{\frac{k' 2^b T}{m}} \sum_{B(S) \cap \Omega} w_q ( 4 \left( \cost(q, a_j) + \cost(a_j, S) \right) + \cost(q, A) + \Delta_j ) \\
		& \overset{(iii)}{\leq}  \bigO{\frac{k' 2^{t + b} T}{m}} \sum_{B(S) \cap \Omega} w_q \left( \cost(q, A) + \Delta_j \right) \overset{(iv)}{\leq} \bigO{\frac{k' \cdot 2^{t + 2b} \cdot T^2 \cdot k_{B(S)}}{m}},
	\end{align*}
	where $(i)$ follows since from Fact~\ref{fact:weight-bound} we have $w_q \leq 4k' \frac{\cost(C_j, A)}{m \cost(q, A)}$, and $w_q \Delta_j \leq \bigO{k' \Delta_j |C_j| / m} \leq \bigO{k' \cost(C_j, A) / m}$ and by definition of band $b$: $\cost(C_j, A) \leq 2^{b+1} T$. Then, $(ii)$ follows from triangle inequality, $(iii)$ by definition of type $t$: $\cost(a_j, S) < 2^t \Delta_j$ and $t \geq 6$, and $(iv)$, from event $\mathcal{E}$:
	\begin{myitemize}
		\item Property $P_3$, and definition of bands (Def.~\ref{def:bands}): 
		\begin{align}
			& \sum_{q \in B(S) \cap \Omega} w_q \cost(q, A) = \sum_{C_j \in B(S) \cap \Omega} \sum_{q \in C_j} w_q \cost(q, A) \nonumber \\
			& \le \sum_{C_j \in B(S) \cap \Omega} (1 + \varepsilon) \cost(C_j, A) \le \sum_{C_j \in B(S) \cap \Omega} (1 + \varepsilon) 2^{b+1} \cdot T \leq \bigO{k_{B(S)} \cdot 2^{b} \cdot T}; \nonumber
		\end{align}
		\item Property $P_1$, and definition of bands (Def.~\ref{def:bands}): 
		\begin{align}
			& \sum_{q \in B(S) \cap \Omega} w_q \Delta_j = \sum_{C_j \in B(S) \cap \Omega} \sum_{q \in C_j} w_q \Delta_j \nonumber \\
			& \leq 2 \sum_{C_j \in B(S) \cap \Omega} \cost(C_j, A) \leq \sum_{C_j \in B(S) \cap \Omega} \bigO{\cost(C_j, A)} \leq \bigO{k_{B(S)} \cdot 2^b \cdot T}. 
		\end{align}
	\end{myitemize}
	We have $\cost(P, S) \geq \cost(B(S), S) = \sum_{C_j \in B(S)} \cost(C_j, S)$. Fix a cluster $C_j \in B(S)$; since $t \geq 6$, we have $\cost(a_j, S) \geq 32 \Delta_j$. From Claim 2.29 of~\cite{bansal2024sensitivity}: 
	$\cost(C_j, S) \geq \frac{1}{6} |C_j| 2^{t - 1} \Delta_j \geq \bigO{2^{b + t} T}$. Summing over clusters: $\cost(P, S) \geq \bigO{k_{B(S)} \cdot 2^{b + t} T}$.
	Thus:
	\begin{align}\label{eq:variance-bound-third}
		& \Var{X^{S, h}} \leq \bigO{2^{-2h}} / (\cost(P, S)^2) \cdot \Gamma \nonumber \\
		& \leq \bigO{\frac{2^{-2h} \cdot k' \cdot 2^{t + 2b} \cdot T^2 \cdot k_{B(S)}}{m \cdot \cost(P, S)^2}}  \leq \bigO{\frac{2^{-2h} \cdot k' \cdot 2^{t + 2b} \cdot T^2 \cdot k_{B(S)}}{m \cdot k^2_{B(S)} \cdot 2^{2t + 2b} \cdot T^2}} \nonumber \\
		& \leq \bigO{ \frac{2^{-2h -t} \cdot k' }{m \cdot k_{B(S)}}} \cdot \max(1, \alpha^2).
	\end{align}
	
\end{proof}

% - corollary
Since the interaction number $N(S) \leq k \cdot k_{B(S)} \leq k \cdot k' \leq \beta \cdot k^2$, and since $N(S) \geq 2^r$, it follows:
\begin{corollary}\label{corollary:variance-bound-second}
	For any set $S$ of centers, we have:
	\begin{align*}
		\Var{X^{S, h}} & \leq \frac{ 2^{-2h - t} \cdot k' \cdot k}{m \cdot N(S)} \cdot \bigO{\max(1, \alpha^2)} \leq \frac{2^{-2h - t - r} \cdot \beta \cdot k^2}{m} \cdot \bigO{ \max(1, \alpha^2)}.
	\end{align*}
\end{corollary}

We are now ready to prove the main lemma:
% -- bound on X^{S, Init}
\begin{lemma}\label{lemma:bound-on-X-h-appendix}
	For any fixed coreset $\Omega$ satisfying event $\mathcal{E}$ with \ourassumptions, and having size $m = \ourcoresetsize$ we have, for a fixed $h \in [h_{max}]$:
	$
	\Expsub{g}{\sup_{S \in \mathcal{S}(r)} \left| X^{S, h} \right| } \leq \varepsilon / h_{max}.
	$
	%where $h_{max} \coloneq \lceil 2 \log_2 \varepsilon^{-1} \rceil$.
\end{lemma}
\begin{proof}
	Conversely to proofs of Lemmas~\ref{lemma:bound-on-X-init-appendix},~\ref{lemma:bound-on-X-fin-appendix} in which we derived terms independent of the supremum over infinitely many centers (in the center class), here we discretize the space and provide bounds by carefully trading-off variance and net size at each scale $2^{-h}$. We have, for any $h \in [1, h_{max}]$:
	\begin{align*}
		& \sup_{S \in \mathcal{S}(r)} \left| X^{S, h} \right| = \sup_{S \in \mathcal{S}(r)} \frac{ \left| \sum_{i \in [m]} g_i w_{q_i} \left( u_i^{S, h} - u_i^{S, h - 1} \right) \right| }{\cost(P, S)} \\
		&  \leq \sup_{ \left( v^{h-1}, v^h \right) \in M_{2^{-(h-1)}} \times M_{2^{-h}}} \frac{ \left| \sum_{i \in [m]} g_i w_{q_i} \left( v_i^{h} - v_i^{h - 1} \right)\right| }{\cost(P, S)}.
	\end{align*} 
	Now, the supremum is finite, expressed over $|M_{2^{-(h-1)}} | \cdot |M_{2^{-h}} | \leq |M_{2^{-h}} |^2$ many pairs of vectors. By Fact~\ref{fact:variance-max-gaussians}:
	\begin{equation} \label{eq:sup-X-h-to-var-times-net-size}
		\Expsub{g}{\sup_{S \in \mathcal{S}(r)} \left| X^{S, h} \right|}  \leq \sigma_h \cdot \sqrt{2 \log \left| M_{2^{-h}} \right|},
	\end{equation}
	where $\Var{X^{S, h}} \leq \sigma_h^2$ for $S \in \mathcal{S}(r)$. We now focus on finding values for $\sigma_h^2$, and combining them with the bound on the net size (Lemma~\ref{lemma:cost-vector-net-size}).
	We consider two cases, based on the interaction number $r$:
	
	\emph{Case 1: $2^r \geq 2k$}. From Lemma~\ref{lemma:variance-bound-in-term-of-costs} and Corollary~\ref{corollary:variance-bound-second}, we have, for any set $S$ of $k$ centers in the center class $\mathcal{S}(r)$:
	\begin{align} \label{eq:variance-bound-X-h-case-2}
		\Var{X^{S, h}} \leq \min \left(1, k^2 \cdot 2^{-(t + r)} \right) \cdot 2^{-2h} / m \cdot \beta \cdot \bigO{\max(1, \alpha^2)}.
	\end{align}
	From Lemma~\ref{lemma:cost-vector-net-size}:
	\begin{align} \label{eq:net-size-bound-X-h-case-1}
		& \log \left| M_{2^{-h}} \right|  \leq \bigO{\min \left( 2^r + k 2^{2h}, 2^t \ k  2^{2h} \right) \cdot poly \log \left( k 2^{h} \varepsilon^{-1}\right)} \nonumber \\
		& \overset{(i)}{\leq} \bigO{2^{2h} \cdot \min \left( 2^r, 2^t k \right) \cdot poly \log \left( k  2^{h} \varepsilon^{-1}\right)} \nonumber \\
		& \overset{(ii)}{\leq}  \bigO{2^{2h} \cdot \min \left( 2^r, 2^t \ k \right) \cdot poly \log \left( k \varepsilon^{-1}\right)},
	\end{align}
	where $(i)$ follows since $2^r + k / 2^{-2h} < 2^{r + 2h}$  (for $2^r \geq 2k$), $(ii)$ since $\log \left( k 2^{h} \varepsilon^{-1} \right) \leq \bigO{ \log \left( k \varepsilon^{-1} \right) }$. Combining Eq.~\ref{eq:sup-X-h-to-var-times-net-size},~\ref{eq:variance-bound-X-h-case-2} and~\ref{eq:net-size-bound-X-h-case-1}:
	\begin{align*}
		& \Expsub{g}{\sup_{S \in \mathcal{S}(r)} \left| X^{S, h} \right| } \leq \sqrt{ \min(1, k^2 \cdot 2^{-(t + r)}) \cdot 2^{-2h} / m \cdot \beta \cdot \bigO{ \max(1, \alpha^2) } } \cdot \\
		& \cdot \sqrt{ \bigO{2^{2h} \cdot \min \left( 2^r, 2^t \ k \right) \cdot poly \log( k \varepsilon^{-1})}}  \\
		% & \leq \sqrt{ \min(1, k^2 \cdot 2^{-(t + r)}) \cdot \min \left( 2^r, 2^t \ k \right)  \cdot \log(k \varepsilon^{-1}) / m \cdot \beta \cdot \bigO{ \max(1, \alpha^2) }} \\
		& \overset{}{\leq} \sqrt{ \min \left( k \cdot 2^t , k^2 \cdot 2^{-t} \right)  \cdot poly \log(k \varepsilon^{-1}) / m \cdot \beta \cdot \bigO{ \max(1, \alpha^2) }},
	\end{align*}
	%where $(i)$ follows since, for any non-negatives $A, B, C, D$ we have $\min(A, B) \cdot \min(C, D) \leq \min(A \cdot C, B \cdot D)$. 
	Consider $t^* = \log_2 \sqrt{k}$. If $t \geq t^*$, then $ \min \left( k \cdot 2^t , k^2 \cdot 2^{-t} \right)  = k^2 \cdot 2^{-t} \leq k^2 \cdot 2^{-t^*} = k^{1.5}$. Otherwise, for $t < t^*$, we have $ \min \left( k \cdot 2^t , k^2 \cdot 2^{-t} \right)  = k \cdot 2^t \leq k \cdot \varepsilon^{-2}$, since $t \leq t_{max} \coloneq \lceil \log_2( \varepsilon^{-2}) \rceil$. Combining both cases we can write: 
	\begin{align*}
		& \Expsub{g}{\sup_{S \in \mathcal{S}(r)} \left| X^{S, h} \right| } \leq \sqrt{\frac{k}{m} poly \log(k \varepsilon^{-1}) \cdot \min(\sqrt{k}, \varepsilon^{-2}) \cdot \beta \cdot \bigO{ \max(1, \alpha^2) }},
	\end{align*}
	which is $\leq \varepsilon$ for the suggested coreset size $m$.
	
	\emph{Case 2: $2^r < 2k$}. Similarly, applying Lemmas~\ref{lemma:variance-bound-in-term-of-costs} and~\ref{lemma:cost-vector-net-size} to Eq.~\ref{eq:sup-X-h-to-var-times-net-size}:
	\begin{align*}
		& \Expsub{g}{\sup_{S \in \mathcal{S}(r)} \left| X^{S, h} \right|} \leq \sqrt{ \bigO{2^{-2h} / m} \cdot \max(1, \alpha^2) } \cdot \sqrt{ \bigO{2^{2h}} \cdot k \cdot poly \log(k \varepsilon^{-1})} \\
		&  \leq \sqrt{\bigO{k \cdot poly \log(k \varepsilon^{-1}) / m} \cdot \max(1, \alpha^2)},
	\end{align*}
	as $2^r + k\cdot2^t < 2^tk\cdot 2^{2h}$ for $2^r < 2k$. The above is $\leq \bigO{\varepsilon}$ for coreset size $m = \bigOmega{k \varepsilon^{-2} \cdot \log(k \varepsilon^{-1}) \cdot \max(1, \alpha^2)}$. 
	
	Combining both cases, we obtain coreset size 
	$$
	m = \ourcoresetsize,
	$$
	which concludes the proof, rescaling $\varepsilon$ by $h_{max} = \bigO{\log \varepsilon^{-1}}$. 
	
\end{proof}

\section{Analysis: Predictions in Sequence of Snapshots}
\label{sec:proof_of_good_predictions}

In this section, we focus on proving Thm.~\ref{thm:good_predictions}.
We first report some results from~\cite{ben2007framework}, using the notation of our paper. 
Let $\OPT_k(\Distro)$ be the optimal $k$-means cost on the distribution $\Distro$, defined as $\OPT_k(\mathcal{D}) = \min_{S \colon |S| = k} \Expsub{p \sim \mathcal{\Distro}}{\cost(p, S)}$.
\begin{corollary}[Application of Corollary 3 from~\cite{ben2007framework}]
	\label{corollary:uc_multiplicative}
	For every $\varepsilon > 0$, and every probability distribution $\Distro$ over $\mathbb{R}^d$, and every set $S$ of $k$ centers, if a sample $P \subseteq \mathbb{R}^d$ of size $n = \bigOmega{\frac{1}{\varepsilon^2 \OPT_k(\Distro)}}$ is picked i.i.d. via $\Distro$, then with constant probability (over the choice of P), 
	$$
	\left| \frac{\cost(P, S)}{n} - \Expsub{p \sim \Distro}{\cost(p, S)} \right| \leq \varepsilon \cdot \Expsub{p \sim \Distro}{\cost(p, S)}.
	$$
\end{corollary}
\begin{proof}
	Corollary 3 from~\cite{ben2007framework} is formulated for clustering problems that admit a local and complete description scheme, an assumption satisfied by $k$-means clustering~\cite{ben2007framework}. Our stated result follows directly from an application of Lemma 1 and Corollary 3 from~\cite{ben2007framework}, by the definition of empirical risk of a clustering on a (finite) pointset, and by definition of $k$-means problem over the distribution $\Distro$.
	Note that we used a Chernoff Bound to get the multiplicative guarantees, and we applied optimality of $\OPT_k(\Distro)$.
\end{proof}

\begin{theorem}[Application of Theorem 4 from~\cite{ben2007framework}]
	\label{thm:uc_for_fixed_centers_mul_apx}
	For every $\varepsilon > 0$, and every probability distribution $\Distro$ over $\mathbb{R}^d$, if a sample $P \subseteq \mathbb{R}^d$ of size $n = \bigOtilde{\frac{k}{\varepsilon^2 \OPT_k(\Distro)}}$ is picked i.i.d. via $\Distro$, then, for every set $S \subseteq P$ of $k$ centers:
	$$
	\left| \frac{\cost(P, S)}{n} - \Expsub{p \sim \Distro}{\cost(p, S)} \right| \leq \varepsilon \cdot \Expsub{p \sim \Distro}{\cost(p, S)}
	$$
	with constant probability (over the choice of $P$).
\end{theorem}
\begin{proof}
	Our stated result follows immediately from an application of Theorem 4 from~\cite{ben2007framework}, where we have that $k$-means clustering satisfies the condition of allowing a local description scheme, holding $l$, $|I|$ and $m$ (notation of~\cite{ben2007framework}) to be $k$, 1, and $n$ (our notation), respectively. Note that we used Corollary~\ref{corollary:uc_multiplicative} before applying union bound, in order to get the multiplicative approximation.
\end{proof}

Let $S^*$ be a clustering that minimizes the cost on the distribution $\Distro$, i.e., $S^* = \min_S \Expsub{p \sim \Distro}{\cost(p, S)}$.
%\begin{theorem}[Application of Theorem 5 from~\cite{ben2007framework}]
%	\label{thm:center_sample_to_distro}
%	Let $\mathcal{A}$ be an algorithm returning an $\alpha$-approximation to the optimal $k$-means cost on any finite pointset $P \subseteq \mathbb{R}^d$; that is, $\cost(P, \mathcal{A}(P)) \leq \alpha \cdot \OPT_k(P)$, with $\mathcal{A}(P) \subseteq P$. 
%	For every $\varepsilon > 0$, and every probability distribution $\Distro$ over $\mathbb{R}^d$, let $P \subseteq \mathbb{R}^d$ be a sample of size $n = \OurSizeBound$ picked i.i.d. via $\Distro$. Then:
%	$$
%	\Expsub{p \sim \Distro}{\cost(p, \mathcal{A}(P))} \leq \bigO{\alpha} \cdot \Expsub{p \sim \Distro}{\cost(p, S^*)} + \varepsilon,
%	$$
%	with constant probability (over the choice of P).
%\end{theorem}
%\begin{proof}
%	Theorem 5 from~\cite{ben2007framework} is stated for clustering problems that admit a local and complete description scheme with $c$-multiplicative coverage, which is the case for $k$-means problem (for $c = 4$, Corollary 1 of~\cite{ben2007framework}).
%	The proof of our stated result follows the one of Theorem 5 from~\cite{ben2007framework}, where $\alpha$-approximation algorithm $\mathcal{A}$ is used in place of empirical risk minimization, introducing the $\bigO{\alpha}$ term.
%\end{proof}
\begin{theorem}[Application of Theorem 5 from~\cite{ben2007framework}]
	\label{thm:center_sample_to_distro_mul_apx}
	Let $\mathcal{A}$ be an algorithm returning an $\alpha$-approximation to the optimal $k$-means cost on any finite pointset $P \subseteq \mathbb{R}^d$; that is, $\cost(P, \mathcal{A}(P)) \leq \alpha \cdot \OPT_k(P)$, with $\mathcal{A}(P) \subseteq P$. 
	For every $\varepsilon \in (0, 1/2]$, and every probability distribution $\Distro$ over $\mathbb{R}^d$, let $P \subseteq \mathbb{R}^d$ be a sample of size $n = \bigOtilde{\frac{k}{\varepsilon^2 \OPT_k(\Distro)}}$ picked i.i.d. via $\Distro$. Then:
	$$
	\Expsub{p \sim \Distro}{\cost(p, \mathcal{A}(P))} \leq \bigO{\alpha} \cdot \OPT_k(\Distro),
	$$
	with constant probability (over the choice of P).
\end{theorem}
\begin{proof}
	Theorem 5 from~\cite{ben2007framework} is stated for clustering problems that admit a local and complete description scheme with $c$-multiplicative coverage, which is the case for $k$-means problem (for $c = 4$, Corollary 1 of~\cite{ben2007framework}).
	The proof of our stated result follows the one of Theorem 5 from~\cite{ben2007framework}, where $\alpha$-approximation algorithm $\mathcal{A}$ is used in place of empirical risk minimization, introducing the $\bigO{\alpha}$ term, and by using multiplicative approximation results from Thm.~\ref{thm:uc_for_fixed_centers_mul_apx}, Corollary~\ref{corollary:uc_multiplicative}. We used the fact $\varepsilon \le 1/2$ to bound $(1 + \varepsilon) / (1 - \varepsilon) \le 3$. 
\end{proof}

Now, we shall proceed to prove our main result (Thm.~\ref{thm:good_predictions}), which we restate for completeness. 
\begin{manualtheorem}{Theorem~\ref{thm:good_predictions}.}
	For any $\varepsilon \in (0, 1/2]$, and any distribution $\Distro$ over $\mathbb{R}^d$, let $P_i, P_j \subseteq \mathbb{R}^d$ be two pointsets of size $n_i, n_j$, respectively, drawn i.i.d. from $\Distro$. 
	Let $\mathcal{A}$ be an algorithm returning a set $\mathcal{A}(P_i) \subseteq P_i$ of $k'$ centers that yields an $\Bicriteria{\alpha}{\beta}$ bi-criteria approximation for dataset $P_i$. 
	If $n_i, n_j = \widetilde{O} \left( {\frac{k'}{ \varepsilon^2  \OPT_{k'}(\Distro)}} \right) $
	then, with constant probability (over the choices of $P_i$ and $P_j$), $
	\cost(P_j, \mathcal{A}(P_i)) \leq \bigO{\alpha} \cdot \OPT_k(P_j)$.
\end{manualtheorem}
\begin{proof}
	We have, for $j \neq i$, with constant probability (over the choice of $P_j$):
	\begin{align*}
		& \frac{\cost(P_j, \mathcal{A}(P_i))}{n_j} \overset{(i)}{\leq} (1 + \varepsilon) \Expsub{p \sim \Distro}{\cost(p, \mathcal{A}(P_i))} \overset{(ii)}{\leq} \bigO{\alpha} \cdot \OPT_{k'}(\Distro)  \\
		& \overset{(iii)}{\leq} \bigO{\alpha} \cdot \Expsub{p \sim \Distro}{\cost(p, S')} \overset{(iv)}{\leq} \bigO{\alpha} \cdot \frac{1}{1 - \varepsilon} \frac{\cost(P_j, S')}{n_j} \overset{(v)}{\leq} \bigO{\alpha} \cdot \frac{\OPT_{k'}(P_j)}{n_j},
	\end{align*}
	where $(i)$ follows from Corollary~\ref{corollary:uc_multiplicative}, $(ii)$ by applying Theorem~\ref{thm:center_sample_to_distro_mul_apx}, holding with constant probability over the choice of $P_i$, and since $\varepsilon \le 1/2$.
	Then, $(iii)$ follows from optimality of $\OPT_k(\Distro)$, by fixing a set $S' \subseteq P_j$ of $k'$ centers that guarantees multiplicative coverage for $P_j$, that is $\cost(P_j, S') \leq 4 \cdot \OPT_{k'}(P_j)$. 
	(From Corollary 1 of~\cite{ben2007framework}, the existence of such $S'$ is guaranteed in the context of Euclidean $k$-means clustering.)
	Finally, $(iv)$ follows by Theorem~\ref{thm:uc_for_fixed_centers_mul_apx} for the fixed choice of $S'$, and $(v)$ by coverage property of $S'$.
	Result of the thorem follows since $\OPT_{k'}(P_j) \le \OPT_{k}(P_j)$ for $k' \ge k$, and then by a union bound and by rescaling the parameters. 
\end{proof}
From the above theorem, it immediately follows the following. 
\begin{manualtheorem}{Corollary~\ref{corollary:bicriteria_from_snapshots}.}
	For every $\varepsilon \in (0, 1/2]$, and every distribution $\Distro$ over $\mathbb{R}^d$ such that $\OPT_{k'}(\Distro) \ge \nu$, for some integer $k$', let $P_i, P_j \subseteq \mathbb{R}^d$ be two pointsets of size 
	$
	n_i, n_j = \bigOtilde{ \frac{k'}{ \varepsilon^2 \; \nu }}
	$ 
	drawn i.i.d. via $\Distro$.
	If a set $A \subseteq P_i$ of $k'$ centers induces an $\Bicriteria{\alpha}{\beta}$ bi-criteria approximation on $P_i$, then, with constant probability (over the choices of $P_i$ and $P_j$), the same set $A$ induces an $\Bicriteria{\bigO{\alpha}}{\beta}$ bi-criteria approximation on $P_j$. 
\end{manualtheorem}
\begin{proof}
	The proof follows applying Theorem 3. Dataset-size bounds are stated in terms of $\nu$ as it holds $\nu \le OPT_{k'}(\Distro)$. The statement follows by the definition of bi-criteria approximation.
\end{proof}

\newpage
\section{Experimental Evaluation}
\label{sec:experiments-appendix}
In this section, we provide further details and visualization regarding datasets, and we report additional results from our experimental evaluation. 

%\DataVizIntelLab
\subsection{Datasets}
\label{sec:datasets-appendix}

In our experimental evaluation, we considered the following snapshot sequences.
\begin{myitemize}
	\item \emph{\textbf{Twitter}}~\cite{twitter_geospatial_data_1050} contains a full week of data sampled from Twitter in the United States, from Saturday Jan 12, 2013 to Friday Jan 18, 2013, analyzed in~\cite{helwig2015analyzing}. Features represent geospatial information, namely GPS location and timestamps; the latter are used for subdividing the datasets into snapshots, for each day of the week. A visualization of geospatial information and number of snapshots can be found in Fig.~\ref{fig:twitter-dataset}.
	\DataVizTwitter
	\item \emph{\textbf{IntelLab}}\footnote{Link: \url{https://db.csail.mit.edu/labdata/labdata.html}, accessed on Jan 24, 2026.} contains information about data collected from 54 sensors (i.e., temperature, humidity, light, voltage) deployed in the Intel Berkeley Research lab between February 28th and April 5th, 2004. Original dataset consists in 38 different days; we removed the last four days, since they contain less than 345 points, resulting in 34 total snapshots. 
	\item \emph{\textbf{Taxi}}~\cite{moreira2013taxi}, from ECML PKDD 2015 prediction challenge, consists of datasets describing trajectories performed by all the 442 taxis running in the city of Porto, in Portugal, from July $1^{st}$, 2013 to June $30^{th}$, 2014. Following~\cite{draganov2024settling} we kept only the starting (pickup) latitude and longitude as features; we divided the data into 12 snapshots (one per each month). The datasets contain outliers, as shown in Fig.~\ref{fig:taxi-dataset}.
	\DataVizTaxi
	\item \emph{\textbf{NYC Taxi and Limousine Commission (TLC)}} is the agency responsible for licensing and regulating New York City's Medallion (Yellow) taxi cabs, for-hire vehicles, commuter vans, and paratransit vehicles. From Kaggle\footnote{Link: \url{https://www.kaggle.com/datasets/elemento/nyc-yellow-taxi-trip-data/data}, accessed on Feb  23, 2025.}, we downloaded data that consider only the Yellow Taxis data for the \emph{months} of Jan 2015, and Jan - Mar 2016 (4 snapshots), having 16 numeric features.  
	%The dataset has 16 numeric features: VendorID, number of passenger, trip distance, pickup latitude, pickup longitude, RateCodeID, dropoff latitude, dropoff longitude, payment type, fare amount, extra, MTA tax, improvement surcharge, tip amount, tolls amount, total amount. The datasets are noisy, as shown in Fig.~\ref{fig:NYC-Taxi-dataset}. 
	For scaling to bigger datasets, we downloaded\footnote{NYC TLC website: \url{https://www.nyc.gov/site/tlc/about/tlc-trip-record-data.page}} datasets discretized by \emph{year}, from 2015 to 2024. 10-year datasets result in snapshots with 14 features. 
\end{myitemize}

\DataVizNYCTLC

\subsection{Additional Experiments}
\label{sec:additional-results-appendix}
In the following, we report results omitted in the main text for lack of space (i.e., results for experiments with parameters number $k$ of centers $k \in \{10, 20, 50\}$, and coreset size $m \in \{200k, 500k \}$ ).

Results across different coreset size $m$ are very similar to the ones reported in the main text. Hence, we just report results and defer to the main text for detailed comments on the different approaches.   
First, we report runtimes for coreset size $m = 200k, m = 500k$ in Figures~\ref{fig:runtime_m200} and~\ref{fig:runtime_m500}.
Then, Fig.~\ref{fig:log_ratio_compression_optimization_m200} and~\ref{fig:log_ratio_compression_optimization_m500} show cost-ratio related to the dataset (on a log-scale) after applying clustering algorithm \texttt{kmeans++} plus Lloyd's\footnote{We set $\text{max iterations}=300$, $\text{tolerance}=0.001$, and $\text{number of initializations}=10$, as in common libraries defaults}, for coreset size $m = 200k, m = 500k$ respectively. 
Finally, we depict plots of estimated distortions, showing the cost ratio related to our algorithm \algname\ (on a log scale), when allowing algorithms to produce coresets of fixed size $m = 200k, m = 500k$. Results are in Fig.~\ref{fig:estimated_log_ratio_distortion_m200} and~\ref{fig:estimated_log_ratio_distortion_m500}.

All in all, our algorithm \algname\ is the fastest approach that provides high-quality coresets across all datasets.
As observed in our main text, our experimental results confirm the practical impact of our approach: across the considered sequences, \algname\ consistently achieves a superior trade-off between clustering cost and runtime compared to uniform sampling and state-of-the-art sensitivity-based techniques.

\CompactRuntime{200}
\vspace{-0.5cm}
\CompactRuntime{500}

\CompactLogRatioCompressionOptimization{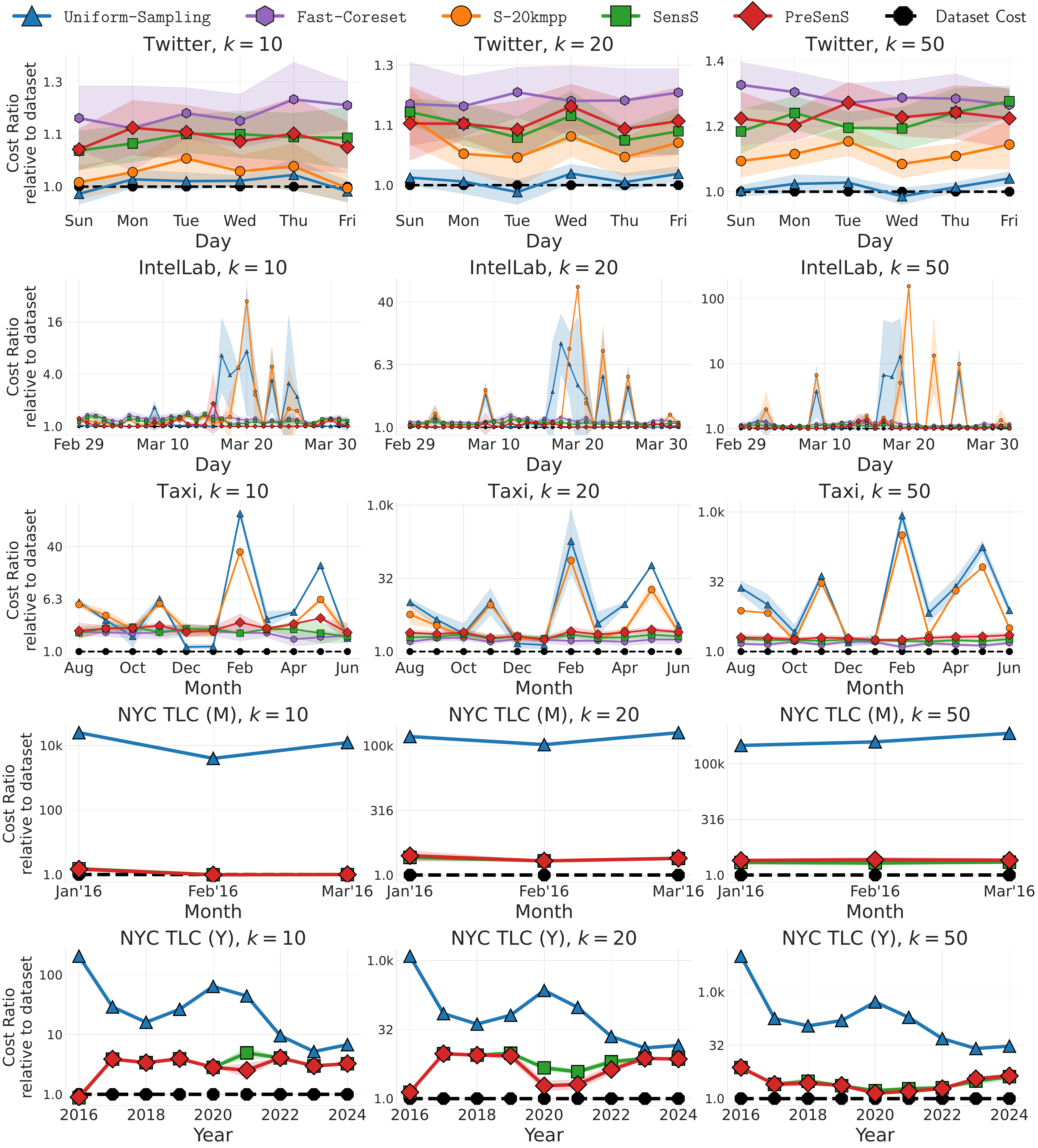}
\CompactLogRatioCompressionOptimization{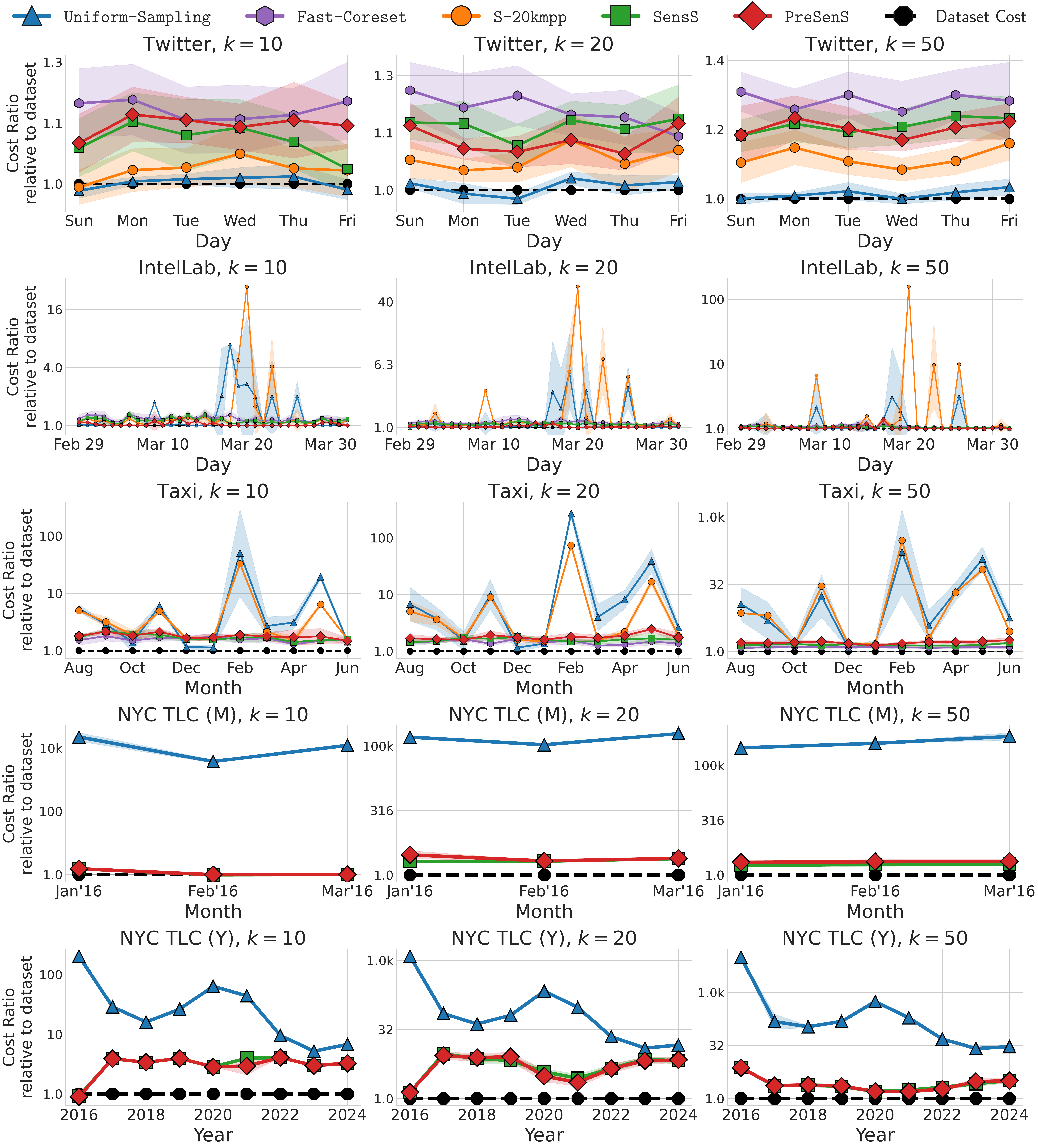}

\EstimatedLogRatioDistortion{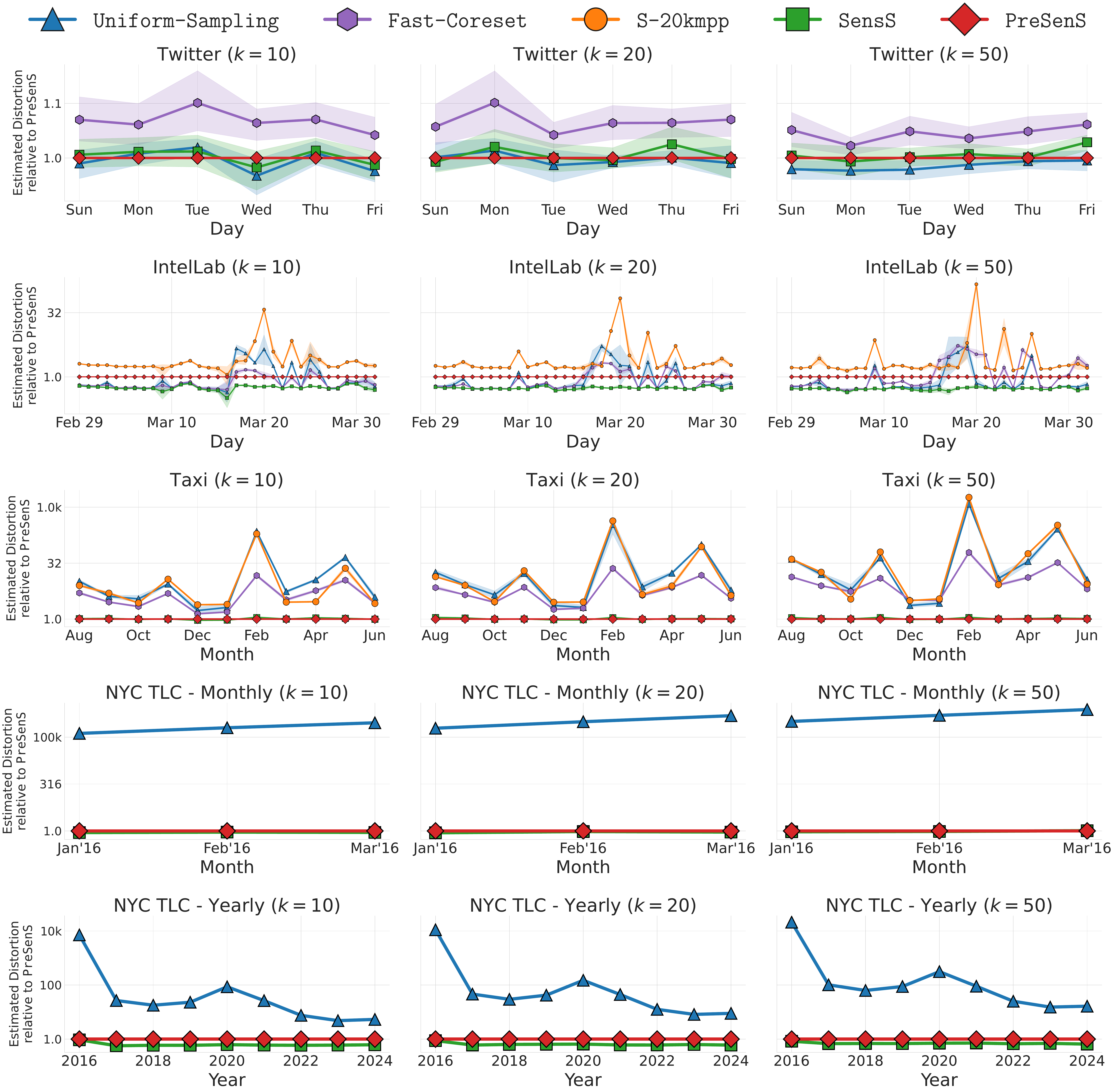}
\EstimatedLogRatioDistortion{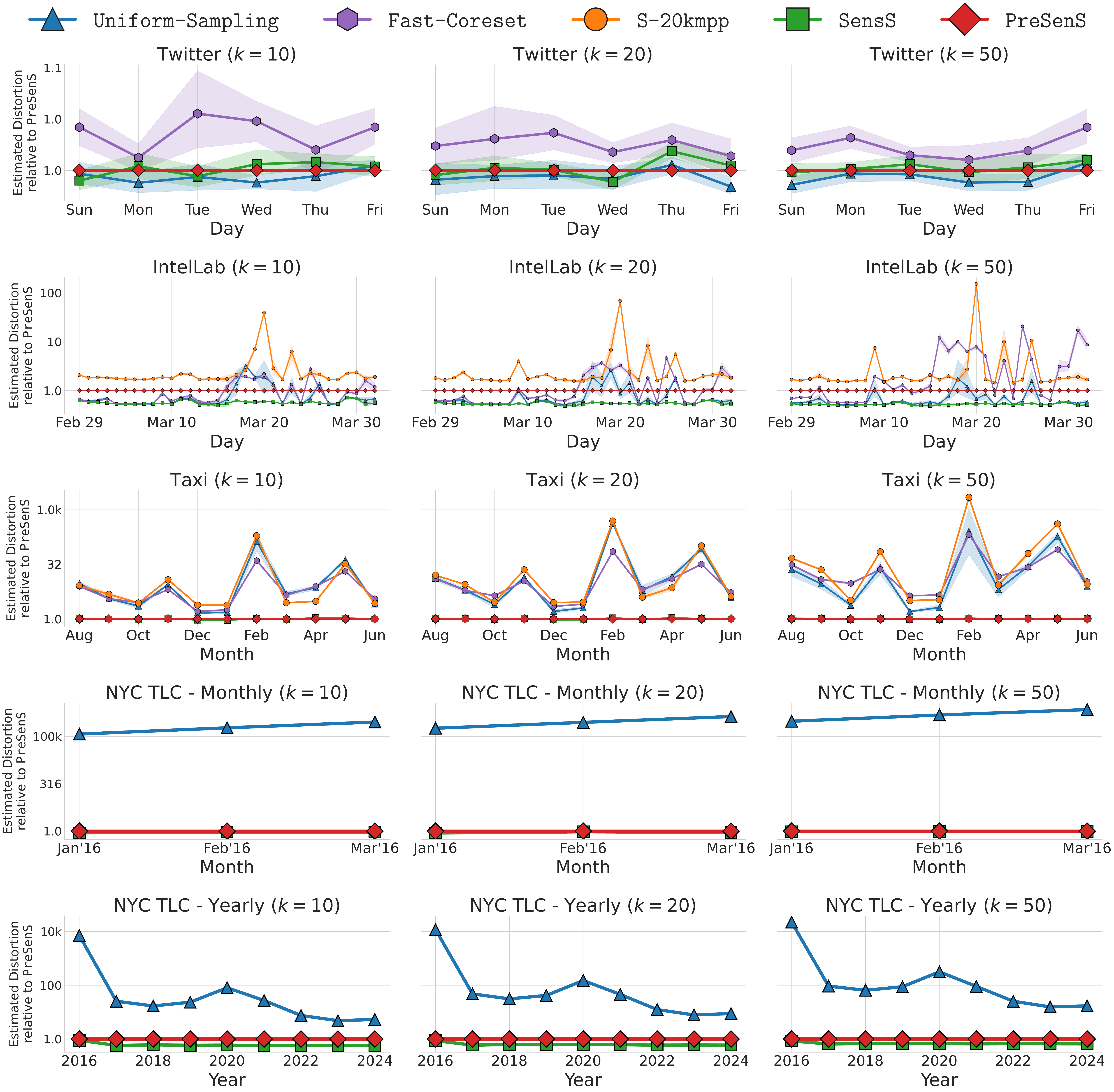}

\end{document}